%% file: neurips_2026.tex
\documentclass{article}

\usepackage[preprint]{neurips_2026}

\usepackage[utf8]{inputenc}
\usepackage[T1]{fontenc}
\usepackage{hyperref}
\usepackage{url}
\usepackage{booktabs}
\usepackage{nicefrac}
\usepackage{microtype}
\usepackage{xcolor}
\usepackage{colortbl}
\usepackage{wrapfig}
\usepackage[ruled,longend]{algorithm2e}
\usepackage{nccmath}
\usepackage{algpseudocode}
\usepackage{graphicx}
\usepackage{amsthm}
\theoremstyle{definition}
\newtheorem{example}{Example}
\usepackage{booktabs}
\usepackage{amsmath}
\usepackage{multirow}
\usepackage{amssymb}
\usepackage{mathtools}
\usepackage{amsthm}
\usepackage{amsfonts}
\usepackage{enumitem}
\usepackage{microtype}
\usepackage{graphicx}
\usepackage{subcaption}
\usepackage{booktabs}

\usepackage[capitalize,noabbrev]{cleveref}

\theoremstyle{plain}
\newtheorem{theorem}{Theorem}[section]

\newtheorem{lemma}[theorem]{Lemma}

\theoremstyle{definition}
\newtheorem{definition}[theorem]{Definition}

\theoremstyle{remark}
\newtheorem{remark}[theorem]{Remark}

\newcommand{\Lop}{\mathtt{L}}
\newcommand{\m}{\mathtt{m}}

\usepackage[disable,textsize=tiny]{todonotes}

\title{Local--Global Geometric Insights for\\Graph Neural Networks via Entropic Curvature}

\author{
  \begin{tabular}{c}
    Rachid Caich \textbf{\thanks{Equal contribution, order determined by coin flip.}} \\
    \small Centre de recherches mathématiques (CRM) \\
    \small Université de Montréal \\
    \small Montréal, Canada \\
    \small \texttt{rachid.caich@umontreal.ca}
  \end{tabular}
  \qquad
  \begin{tabular}{c}
    Yassine Abbahaddou \textbf{\footnotemark[1]} \\
    \small LIX, École Polytechnique\\ IP Paris \\
    \small Palaiseau, France \\
    \small \texttt{yassine.abbahaddou@polytechnique.edu}
  \end{tabular}
}
\begin{document}

\maketitle

\begin{abstract}
\input{Content/abstract}
\end{abstract}

\section{Introduction}\label{sec:introduction}
\input{Content/introduction}

\section{Background and Related Work}\label{sec:background}

\input{Content/related_work}

\section{Entropic Curvature on Graphs}\label{sec:math_fram}
\input{Content/mathematical_framework}

\section{Expansion Paradox and Rewiring}\label{sec:Expansion}
\input{Content/Entropic_expansion}

\section{Entropic Curvature and Generalization}\label{sec:generalization}

\input{Content/Robustness}

\section{Entropic Curvature-Based Architectures}\label{sec:exper}
\input{Content/experiments}


\section{Conclusion.}\label{sec:conclusion}\input{Content/conclusion}

\bibliographystyle{plain}
\bibliography{references}

\newpage
\appendix

\section{Interpreting Weak Entropic Curvature via Local Transport}\label{app:interpretation}
\input{Appendix/appendix_interpretation}

\section{Explicit Computation of Weak Entropic Curvature}\label{app:theoretical_calculation}
\input{Appendix/Entropic_curvature_some_calculations}

\section{Spectral Gap Enhancement via Curvature-Based Rewiring}\label{app:curvature_based_rewiring}

\input{Appendix/curvature_based_rewiring}

\section{Geometric Characterization and Proof of Theorem \ref{thm:Motzkin-Straus}}\label{app:proof_Strauss}
\input{Appendix/proof_Strauss}

\section{Proof of Theorem \ref{thm:lower_bound}}\label{app:Proof_entropic_curvature_bound}
\input{Appendix/Proof_entropic_curvature_bound}

\section{Proof of Theorem \ref{thm:robustness}}\label{app:proof_thm_robustness}
\input{Appendix/proof_thm_robustness}

\section{Proof of Theorem \ref{thm:gen_bound_negative}}\label{app:proof_thm_neg_bound}
\input{Appendix/proof_gen_negative}

\section{Datasets Statistics \& Implementation Details}\label{app:dataset_statis}
\input{Appendix/dataset_stat}

\section{Proof of Theorem \ref{thm:expansion}}\label{app:proof_expansions}
\input{Appendix/proof_expansions}

\section{Proof of Theorem \ref{thm:poincare-gnn}}\label{app:proof:thm:poincare-gnn}
\input{Appendix/proof_poincare_gnn}

\section{Proof of Theorem \ref{thm:poincare_negative} }\label{app:proof:thm:poincare_negative}
\input{Appendix/proof_poincare_negative}

\section{Proof of Theorem \ref{thm:rewiring}}\label{app:proof_rewiring}
\input{Appendix/proof_rewiring_thm}

\section{Algorithmic Computation of Weak Entropic Curvature and Convergence Analysis}\label{app:impl}
\input{Appendix/implementation_details}

\section{Implementation of Entropic Structural Encodings}\label{app:lcp_ent}

\input{Appendix/lcp_ent_exp}

\section{Alternative Generator Constructions}\label{alternative_generator}
\input{Appendix/alternative_generators}

\section{Interpretation of Graph Generators and Reference Measures}\label{app:generator_viz}
\input{Appendix/generator_vizualization}

\section{Complexity Analysis}\label{app:complexity}
\input{Appendix/complexity_analysis}

\section{Additional Experiments}\label{app:additional_experiments}
\input{Appendix/additional_experiments}

\section{Entropic Curvature and Oversmoothing}\label{app:oversm}

\input{Appendix/oversmoothing}

\section{Limitations and Open Directions}\label{app:limitations}
\input{Appendix/limitations}

\end{document}

%% file: Content/abstract.tex
Curvature notions on graphs, particularly Ollivier--Ricci and Forman, have emerged as powerful tools for addressing fundamental issues in Graph Neural Networks (GNNs) such as oversmoothing and oversquashing, but rely almost exclusively on \emph{local} edge-level comparisons and therefore fail to certify how information actually propagates over long distances. We introduce \emph{Entropic Curvature}, a global, transport-based curvature obtained by extending the Lott--Sturm--Villani framework to graphs through the displacement convexity of entropy along $W_1$-Wasserstein geodesics. We define a tractable Weak Entropic Curvature proxy $\kappa_w$ that lower-bounds the global entropic curvature, and from it derive (i) a Poincaré-type inequality controlling oversmoothing, (ii) a transport--entropy generalization bound, and (iii) an \emph{expansion paradox} proving that sparsity, strong spectral expansion, and positive entropic curvature cannot coexist in large graphs, unifying oversmoothing and oversquashing as opposite ends of a single curvature spectrum. We translate the theory into three practical mechanisms, the E-Gate aggregator, the ENT structural encoding, and Midpoint-Completion Rewiring (MCR), and benchmark them against SDRF, FoSR, BORF, LCP, and Graph Ricci Flow on six node-classification benchmarks, and graph-classification. 

%% file: Content/introduction.tex
Graphs are increasingly analyzed through the lens of geometry, with curvature emerging as a key concept in understanding graph structure and dynamics~\citep{li_yau_ineq,Davies_lemma,sharp_Davies_lemma,Isoperimetric,Harnack_inequalities,Strong_Harnack_inequalities,Li_Yau_inequality_laplace,Ricci_graph,Ricci_graph2}. Recently, attention has shifted to incorporating geometric insights into GNNs: graph curvature offers an efficient way to capture phenomena such as bottlenecks~\citep{topping2022understanding,elhag2022graph}, characterise community structure~\citep{ni2019community,jost2014ollivier,lubberts2024curvature}, and inspire new architectures~\citep{chen2025graph,li2022curvature,bachmann2020constant}.

Discrete graph curvatures, e.g.\ Ollivier--Ricci~\citep{ollivier2009ricci} and Forman~\citep{forman2003bochner}, have been used to analyze fundamental GNN problems~\citep{torieffectiveness} such as oversmoothing~\citep{nguyen2023revisiting}, where deep GNN layers produce indistinguishable node features, and oversquashing~\citep{topping2022understanding,karhadkar2023fosr}, where long-range information is compressed through bottlenecks. \emph{However, existing graph curvatures are inherently local: they evaluate edges or one-hop neighborhoods, and therefore cannot certify how information actually propagates over long distances or quantify the trade-off between oversmoothing and oversquashing in a single framework.} As a result, oversmoothing and oversquashing are diagnosed by separate, often-incompatible local quantities.

This work introduces \emph{entropic curvature} on graphs, a global, transport-theoretic curvature grounded in the Lott--Sturm--Villani framework~\citep{Villani,sturm}. Unlike edge-based curvatures, entropic curvature evaluates a graph as a metric-measure space by examining how probability distributions evolve along $W_1$-Wasserstein geodesics. From a single non-negativity assumption on entropic curvature we derive (a) a Poincaré inequality controlling Dirichlet-energy decay (oversmoothing), (b) a transport--entropy generalization bound, and (c) the impossibility of simultaneously achieving sparsity, strong spectral expansion, and positive entropic curvature (the \emph{expansion paradox}, governing oversquashing). The same theory directly produces curvature-aware GNN building blocks, an aggregator (E-Gate), a structural encoding (ENT), and a rewiring rule (MCR), which we evaluate against the most relevant curvature-aware baselines including SDRF~\citep{topping2022understanding}, FoSR~\citep{karhadkar2023fosr}, BORF~\citep{nguyen2023revisiting}, and LCP~\citep{fessereffective}. 

Our main contributions are:

\begin{itemize}
    \item \textbf{Rigorous Discrete Reformulation (Section~\ref{sec:math_fram}):} We adapt Lott--Sturm--Villani to graphs via convexity of entropy along $W_1$-Wasserstein geodesics, and introduce a tractable \emph{Weak Entropic Curvature} proxy $\kappa_w$ computable in $O(k_{\text{iter}}\,d_{\max}^2)$ per node that provably lower-bounds the global entropic curvature (Theorem~\ref{thm:lower_bound}).

    \item \textbf{Expansion Paradox (Section~\ref{sec:Expansion}):} Sparsity, strong spectral expansion, and positive entropic curvature cannot coexist in large graphs (Theorem~\ref{thm:expansion}). This unifies oversmoothing and oversquashing as opposite ends of a single curvature spectrum.
    \item \textbf{Curvature-Controlled Generalization (Section~\ref{sec:generalization}):} The train--test risk gap is bounded by $\kappa_w$ via a transport-entropy inequality, in both positive- (Theorem~\ref{thm:robustness}) and negative-curvature (Theorem~\ref{thm:gen_bound_negative}) regimes.
    \item \textbf{Curvature-Monotonic Rewiring (MCR, Section~\ref{sec:Expansion}):} We give a rewiring rule that provably non-decreases $\kappa_w$ at the rewired node (Theorem~\ref{thm:rewiring}), outperforms ORC and FRC rewiring in spectral-gap improvement on six random-graph families, and translates to downstream accuracy gains over SDRF/FoSR/BORF.
    \item \textbf{Architectures and Empirical Validation (Sections~\ref{sec:Expansion},~\ref{sec:exper}):} The \emph{E-Gate} aggregator improves its backbone in 16 of 20 (backbone, dataset) pairs at $p<0.05$ (paired Wilcoxon, Holm--Bonferroni). \emph{ENT} encodings outperform LCP, LAPE, RWPE, and EGO on 4 of 5 benchmarks. We additionally benchmark against other baslines such as SDRF and FoSR, and on graph classification.
\end{itemize}

%% file: Content/related_work.tex

\textbf{Ricci Curvature.} Ricci curvature is a fundamental concept in differential geometry that describes how a manifold deviates from being flat. One way to detect curvature is through \emph{holonomy} \citep{kobayashi1996foundations}, which measures how parallel transport around closed loops fails to preserve the geometrical data being carried, thereby revealing twists intrinsic to the manifold. Another approach involves \emph{geodesic dispersion}; if nearby geodesics with identical initial velocities diverge, the space is positively curved; if they remain parallel, it is flat; and if they converge, it is negatively curved \citep{do1992riemannian}. A major advance in extending Ricci curvature to more general metric measure spaces beyond smooth manifolds was provided by Sturm \citep{sturm} and Lott and Villani \citep{Villani} in two independent works. Building on McCann’s pioneering work~\citep{McCann}, they established a synthetic definition of Ricci curvature via optimal transport and entropy. Specifically, let $(\Omega , d)$ be a complete and separable metric space, and let $\mathcal{P}(\Omega)$ be the space of Borel probability measures on $\Omega$. For $1 \leq p < \infty$ and $\mu, \nu \in \mathcal{P}(\Omega)$ we consider the Monge-Kantorovich (Wasserstein) distance $W_p$ is defined by,
$$\mathrm{W}_p(\mu, \nu) := \inf_{\Gamma \in \Pi(\mu, \nu)} \left(\int_{\Omega \times \Omega} d(x, y)^p \mathrm{d}\Gamma(x, y)\right)^{1/p},$$
where $\Pi(\mu, \nu)$ denotes the set of all couplings of $\mu$ and $\nu$.  $\mathrm{W}_p(\mu, \nu)^p$ can be interpreted as the minimal cost required to transfer a distribution of mass from an initial configuration $\mu$ to another configuration $\nu$, where the transportation cost per unit mass is given by $d(x, y)^p$. In this synthetic framework, relative entropy plays a key role in capturing curvature lower bounds. With a measure $\nu \in \mathcal{P}(\Omega)$ and a probability measure $\mu$ that is absolutely continuous with respect to $\nu$, the relative entropy of $\mu$ with respect to $\nu$ is defined by $\operatorname{Ent}_\nu(\mu) := \int_\Omega \rho(x) \log \rho(x) \mathrm{d}\nu(x),$ where $\rho = \mathrm{d}\mu/\mathrm{d}\nu$. 

The following result, proved in~\citep{CorderoErausquin,Otto_villani,Von_sturm}, characterizes Ricci curvature lower bounds on Riemannian manifolds in terms of convexity properties of the relative entropy (with respect to the volume measure $\operatorname{vol}$) and optimal transport. 
\begin{theorem}[Characterisation of Ricci lower bounds]\label{theorem_and_definition_of_entr_in_theory} Let $\kappa \in \mathbb{R}$. For a complete Riemannian manifold $\mathcal{M}$ with a Ricci Curvature $\operatorname{Ric} $, the following assertions are equivalent,
\begin{enumerate}[noitemsep,leftmargin=*,topsep=0pt,partopsep=0pt]
    \item $\operatorname{Ric} \geq \kappa$ on $\mathcal{M}$.
    \item Each pair of probability measures $\mu_0, \mu_1 \in \mathcal{P}_2(\mathcal{M})$ can be connected by a constant speed $\mathrm{W}_2$-geodesic $(\mu_t)_{t \in [0,1]}$, i.e., $ W_2(\mu_s,\mu_t)=|t-s|\cdot W_2(\mu_0,\mu_1)$, along which the entropy satisfies the $\kappa$-convexity inequality,
    \begin{equation}\label{Entropic_curvature_ineuqlity}
            \operatorname{Ent}_{\operatorname{vol}}(\mu_t) \leq (1- t) \operatorname{Ent}_{\operatorname{vol}}(\mu_0) + t \operatorname{Ent}_{\operatorname{vol}}(\mu_1) - \frac{\kappa}{2} t(1- t)\mathrm{W}_2(\mu_0, \mu_1)^2.
    \end{equation}
\end{enumerate}
\end{theorem}

This inequality establishes a connection between the Ricci curvature $\operatorname{Ric} $ of a Riemannian manifold and the convexity of entropy along optimal transport paths. This means positive curvature  enhances the convexity of entropy, while negative curvature weakens it, reflecting how the manifold's geometry influences the behavior of entropy under optimal transport.

\textbf{Graph Curvatures.} Since graphs are discrete, and non-differentiable structures, the classical notion of Ricci curvature from differential geometry cannot be directly applied. Instead,  discrete analogues of curvature have been introduced to mimic aspects of Ricci curvature. One such measure is Forman curvature \citep{forman2003bochner} which assigns scalar weights to edges or nodes, providing a purely combinatorial curvature notion and offering insights primarily into local topology or geometry. By contrast, Ollivier-Ricci curvature \citep{ollivier2009ricci} uses a probabilistic viewpoint by comparing single-step neighborhood mass transport between adjacent vertices, capturing local bottlenecks and concentration but ignoring how measures evolve over longer paths. Another Ricci-type notion is Bakry-Émery curvature \citep{bakry2006diffusions}, which is evaluated at vertices rather than edges. It measures the local cohesion of a network around each vertex, so that higher cohesion corresponds to more positive curvature. Finally, the Balanced Forman curvature refines the original Forman curvature by incorporating symmetrical weighting factors \citep{topping2022understanding}, aiming to balance the contributions of vertices and edges in a way that can provide a more nuanced assessment of a graph’s local and global curvature-like properties.

\textbf{Graph Neural Networks (GNNs).}
A graph is denoted by $\mathcal{G}=(\mathcal{V},\mathcal{E},\mathbf{X})$, where $\mathcal{V}$ is the node set, $\mathcal{E}$ is the edge set, and $\mathbf{X}$ contains node features. Each node $u\in\mathcal{V}$ is initialized with $h_u^{(0)}=x_u$. A broad class of message-passing GNNs updates node representations through two operations: an aggregation step $m_u^{(\ell)}
=
\mathrm{AGG}^{(\ell)}
(
\{h_v^{(\ell)}:v\in\mathcal{N}(u)\cup\{u\}\}
),$ and an update step $h_u^{(\ell+1)}
=
\mathrm{UPDATE}^{(\ell)}
(
h_u^{(\ell)},m_u^{(\ell)}
).$ Here, $\mathrm{AGG}^{(\ell)}$ collects information from the local neighborhood, while $\mathrm{UPDATE}^{(\ell)}$ combines the aggregated message with the current node representation. 

\textbf{Generalization of GNNs.} A fundamental GNN challenge in GNNs is characterizing their generalization capabilities, particularly when the training and testing data arise from distinct topological distributions on the graph \citep{tang2023towards}. In the standard node classification setting, we treat the training and evaluation sets not merely as indices, but as discrete probability measures supported on the vertex set $\mathcal{V}$. Let $\mathcal{V}_{\text{train}}, \mathcal{V}_{\text{test}} \subset \mathcal{V}$ denote the labeled training and testing node sets, respectively. We define the empirical training distribution $\mu_{\text{train}}$ and the evaluation distribution $\mu_{\text{test}}$ as, $\mu_{\text{train}} := \frac{1}{|\mathcal{V}_{\text{train}}|} \sum_{u \in \mathcal{V}_{\text{train}}} \delta_u, \quad \text{and} \quad \mu_{\text{test}} := \frac{1}{|\mathcal{V}_{\text{test}}|} \sum_{u \in \mathcal{V}_{\text{test}}} \delta_u,$ where $\delta_u$ represents the Dirac measure centered at node $u$. Consider a GNN $f_\theta: (\mathcal{V},\mathcal{E},X) \to \mathcal{Y}$ parameterized by $\theta$, and a pointwise loss function $\mathcal{J}(\cdot, \cdot)$. The expected risk with respect to an arbitrary node distribution $\mu$ is given by, $$\mathcal{L}_\mu(f_\theta) = \mathbb{E}_{u \sim \mu}[\mathcal{J}(f_\theta(u), y_u)].$$

While the optimization objective is to minimize the empirical risk $\mathcal{L}_{\mu_{\text{train}}}(f_\theta)$, the ultimate goal is to minimize the risk on the test distribution. The generalization gap quantifies the discrepancy between these risks for the learned parameters $\hat{\theta} = \arg\min_\theta \mathcal{L}_{\mu_{\text{train}}}(f_\theta)$. Formally, this gap is defined as the absolute deviation between the test and training risks, $$\Delta \mathcal{L}(f_{\hat{\theta}}) := \left| \mathcal{L}_{\mu_\text{test}}(f_{\hat{\theta}}) - \mathcal{L}_{\mu_\text{train}}(f_{\hat{\theta}}) \right|.$$  
As shown in Section \ref{sec:generalization}, the entropic curvature of the underlying graph imposes strict bounds on this sensitivity via transport-entropy inequalities.

%% file: Content/mathematical_framework.tex

\paragraph{Notations and Assumptions.}\label{Proporties_lm}
We consider a connected, undirected graph $\mathcal{G} = (\mathcal{V}, \mathcal{E}, \mathrm{X})$. We denote the combinatorial distance between vertices $u, v \in \mathcal{V}$ by $d(u,v)$, defined as the length of the shortest path connecting them. We assume $\mathcal{G}$ is \emph{locally finite}, meaning every vertex has a finite degree. For any node $u \in \mathcal{V}$, we define the local 2-ball as the subgraph induced by nodes within a distance of at most two, denoted $\mathrm{B}_2(u) := \{y \in \mathcal{V} \mid d(u, y) \le 2\}$. The 2-sphere is the set of nodes at a distance of exactly two from $u$, denoted $\mathrm{S}_2(u) := \{w \in \mathcal{V} \mid d(u, w) = 2\}$. Let $G(x, y)$ be the set of all geodesic paths joining $x$ to $y$ (i.e. all shortest paths from $x$ to $y$).
For any pair of vertices $u, v \in \mathcal{V}$, we denote by $[u, v]:= \{ z\in \mathcal{V}| z\in \gamma, \gamma \in G(u,v)\}$ and $]u, v[:= [u,v] \backslash \{u,v\} $ the set of \emph{strict midpoints} lying on any shortest path between them. A \emph{graph space} is defined as the tuple $(\mathcal{V}, d, \m, \Lop)$, where $\m \colon \mathcal{V} \to \mathbb{R}_+$ is a positive reference measure and $\Lop \colon \mathcal{V} \times \mathcal{V} \to \mathbb{R}$ is a generator encoding the jump rates of a reversible Markov process on $\mathcal{V}$. The operator $\Lop$ is constructed to satisfy three key properties,

\begin{enumerate}
    \item[(i)] \textbf{Conservation:} The diagonal terms ensure probability conservation, i.e., $\Lop (u, u) = -\sum_{v \neq u} \Lop (u, v).$
    \item[(ii)] \textbf{Sparsity:} We require $\Lop (u, v) > 0$ if and only if $u \sim v$. This ensures that $\Lop$ acts as a Graph Shift Operator (GSO)~\cite{dasoulaslearning,gavili2017shift,abbahaddou2026graph}, preserving the sparsity of $\mathcal{G}$. 
    \item[(iii)] \textbf{Reversibility:} The measure $\m$ satisfies the detailed balance condition, i.e., $\m (u)\Lop (u, v) = \m (v)\Lop (v, u) \quad \forall u,v \in \mathcal{V}.$ This constraint implies that $\Lop$ is self-adjoint with respect to the inner product induced by $m$, guaranteeing a real spectrum.
\end{enumerate}
The selection of the reference measure $\mathrm{m}$ inherently dictates the structure of the generator $\Lop$ through the reversibility condition. This relationship is illustrated in Example \ref{ex:uniform}. While derivations for generators under general positive measures and concentrated approximations are provided in Appendices~\ref{alternative_generator} and~\ref{app:generator_viz}, the present work focuses on the canonical pair $(\mathrm{m}_0, \Lop_0)$.
\begin{example}[Uniform Measure]\label{ex:uniform}
    Let $\mathcal{C} \subseteq \mathcal{V}$. In the standard case where $\mathcal{C} = \emptyset$ or $\mathcal{C} = \mathcal{V}$, we adopt the uniform measure $\m_0 (u) = 1$. To satisfy detailed balance, the corresponding generator is simply the unnormalized Combinatorial Laplacian adjacency weights, i.e., $\Lop_0 (u,v) := \mathbf{1}_{\{u,v\} \in \mathcal{E}} \quad \text{for } u \neq v.$
\end{example}

\textbf{Theoretical Definition of Entropic Curvature.} The concept of entropic curvature provides a powerful way to extend Ricci-type geometric ideas to discrete spaces, particularly graphs. Recall from the Sturm, Lott and Villani framework, c.f. Section \ref{sec:background}, that lower bounds on Ricci curvature can be captured by the convexity of the entropy functional along Wasserstein geodesics. This carries over to the discrete setting via the following definition,

\begin{definition}[Entropic Curvature]\label{def:entropic_curvature}
The \emph{entropic curvature} $ \kappa \in \mathbb{R} $ is the largest constant such that for every pair of probability measures $ \mu_0, \mu_1 $ on $ \mathcal{V} $, there exists a constant-speed $ \mathrm{W}_1 $-geodesic $ (\mu_t)_{t \in [0,1]} $, i.e., $\mathrm{W}_1(\mu_s,\mu_t)=|t-s|\cdot\mathrm{W}_1(\mu_0,\mu_1)$, for which the entropy functional satisfies the $\kappa$-convexity inequality,
\begin{equation}
        \operatorname{Ent}_\m (\mu_t) \leq (1 - t)\operatorname{Ent}_\m (\mu_0) + t\operatorname{Ent}_\m (\mu_1) \; \; - \frac{\kappa}{2}t(1 - t) \mathrm{W}_1(\mu_0, \mu_1)^2.
\end{equation}
\end{definition}
This definition mirrors the role of Ricci curvature in smooth geometry; if $ \kappa > 0 $, then the entropy is uniformly displacement convex, which in turn imposes geometric rigidity on the space. To make the computation of curvature feasible in large graphs, we build on the work of \cite{samson2022entropic}, introducing a local proxy for $ \kappa $, known as the \emph{weak entropic curvature}. This approximation relies on evaluating the behavior of entropy in neighborhoods of radius two.

The classical Lott--Sturm--Villani framework characterises Ricci curvature lower bounds on smooth Riemannian manifolds via displacement convexity along constant-speed $\mathrm{W}_2$-geodesics (Theorem~\ref{theorem_and_definition_of_entr_in_theory}). In the discrete graph setting, $\mathrm{W}_2$-geodesics require interpolating measures that generally lack combinatorial interpretations and are computationally intractable. Following \cite{rapaport2024samson}, we instead use $\mathrm{W}_1$-geodesics, which are well-defined on finite graphs via linear programming and are computable in polynomial time. This substitution is theoretically grounded: on smooth manifolds, a Ricci lower bound is equivalent to displacement convexity along \emph{both} $\mathrm{W}_1$- and $\mathrm{W}_2$-geodesics \citep{Villani}. In the discrete regime, the $\mathrm{W}_1$ metric directly enables the transport-entropy inequality (Lemma~\ref{lemm:westrass_controled_entropy}) which underpins all downstream theoretical results.


\textbf{Formal definition of Weak Entropic Curvature.} Let $z \in \mathcal{V}$ be a vertex and $\mathcal{W} \subseteq \mathcal{V}$ a target set, typically the 2-sphere $\mathrm{S}_2(z)$. For any $z'' \in \mathcal{W}$, we define the two-step transition weight $\Lop^2(z, z'') := \sum_{v} \Lop (z,v)\Lop (v,z'')$ and the relative path weight $\ell (z, z', z'') := \frac{\Lop (z, z')\Lop (z', z'')}{\Lop^2(z, z'')}.$ The local functional $\mathrm{H}_\Lop (z, \mathrm{W})$ is defined as,
\begin{equation}
    \mathrm{H}_\Lop (z, \mathrm{W}) := \sup_{\alpha \in \mathcal{P}}\left\{ \sum_{z'' \in W} \Lop^2(z, z'') 
    \quad \cdot \prod_{z' \in ]z, z''[} \left( \frac{\alpha(z')}{\Lop (z, z')} \right)^{2\ell (z, z', z'')} \right\},
\end{equation}
Here, $]z, W[ := \bigcup_{z'' \in W} ]z, z''[$ denotes the union of all midpoints. The supremum is taken over all functions $\alpha: ]z, W[ \rightarrow \mathbb{R}_+$ satisfying the normalization condition $\sum_{v \in ]z, W[} \alpha(v) = 1$. Beyond this formal definition, weak entropic curvature admits a useful geometric interpretation in terms of local transport redundancy. In Appendix~\ref{app:interpretation}, we show that the local functional $\mathrm{H}_\Lop$ can be viewed through an entropy/KL-divergence lens: it measures how well the graph's two-hop geodesic structure supports the optimal redistribution of mass across intermediate nodes. This perspective reveals that high-curvature regions correspond to locally redundant, diffusion-friendly neighborhoods, whereas low- or negative-curvature regions indicate bottleneck-like structures where transport is forced through fewer fragile paths. We also relate this interpretation to a Motzkin--Straus-type characterization \citep{motzkin1965maxima}, showing how local clique density influences the curvature proxy.

Inspired by the functional form of \textit{Boltzmann Entropy}, we define the \textit{local weak entropic curvature}, a tractable local proxy for the entropic curvature defined, c.f., Definition \ref{weakentropic}.

\begin{definition}[Weak Entropic Curvature]\label{weakentropic}
Let $u \in \mathcal{V}$. We define the \textit{local weak entropic curvature} at node $z$ as, $\kappa_w(u) := -2 \log \mathrm{H}_\Lop(u, \mathrm{S}_2(u)) \text{,}$ and the \textit{weak entropic curvature} of the graph as the infimum over all vertices, i.e., 
$\kappa_w(\mathcal{G}) := \inf_{u \in \mathcal{V}} \kappa_w(u).$
\end{definition}
While the global curvature examines the convexity of entropy along the entire Wasserstein space, the local weak entropic curvature focuses on the displacement of mass within a controlled two-hop neighborhood. This approximation allows for efficient computation while preserving the core geometric insights of the Lott-Sturm-Villani framework. The formal link between this local proxy and the global entropic curvature is established in Theorem \ref{thm:lower_bound}.

\begin{theorem}\label{thm:lower_bound}
The weak entropic curvature $\kappa_w$ provides a rigorous lower bound for the global entropic curvature $\kappa$, i.e., $ \kappa \ge \kappa_w. $
\end{theorem}
\begin{proof}
    See  Appendix \ref{app:Proof_entropic_curvature_bound}.
\end{proof}
Theorem \ref{thm:lower_bound} shows that the computable  $\kappa_w$ is not merely a heuristic local statistic: it is a certified lower bound on the global entropic curvature $\kappa$. Thus, although $\kappa$ is defined through entropy convexity over the full Wasserstein space, its geometric effect can already be detected through two-hop transport structure. Intuitively, large positive values of $\kappa_w$ arise in locally redundant regions where mass can be redistributed through many alternative midpoints, while low or negative values indicate bottleneck-like neighborhoods in which transport is forced through a small number of fragile paths. To illustrate this, Appendix \ref{app:theoretical_calculation} reports $\kappa_w$ on canonical graph motifs, ranging from dense clustered structures to sparse trees. These examples show that weak entropic curvature separates diffusion-friendly local traps from negatively curved expansion and bottleneck regimes, providing the geometric intuition behind the oversmoothing and oversquashing results developed in the following sections.

\textbf{Algorithmic Implementation.} We evaluate the practical utility of entropic curvature through a node-wise implementation of the weak entropic curvature proxy. Algorithm \ref{alg:weak-curv-exact} in Appendix \ref{subsec:app:algo} operationalizes this by transforming the theoretical displacement convexity requirement into a constrained optimization problem. Specifically, for each node, we precompute the relative path weights $\ell(z, u, w)$, which represent the local transport efficiency across two-hop geodesics. The core of the algorithm involves solving a smooth concave maximization for the local functional $\mathrm{H}_\Lop$ over the probability simplex $\Delta_{|\mathrm{S}_1(z)|}$. To achieve high-fidelity convergence at scale, we use the Sequential Least Squares Programming (SLSQP) solver \citep{nocedal2006numerical}. This gradient-based approach leverages the smoothness of the objective function $\mathrm{J}(\alpha)$, typically reaching a stable curvature estimate within a small number of iterations. A detailed analysis of the algorithm is provided in Appendices \ref{app:impl} and \ref{app:complexity}.





%% file: Content/Entropic_expansion.tex
A central question in GNN design is whether a graph can simultaneously achieve three desirable properties: \emph{sparsity} (bounded degree, ensuring scalable message-passing), \emph{spectral expansion} (fast mixing, preventing oversquashing), and \emph{positive entropic curvature} (geometric regularity, preventing oversmoothing). Our main result of this section shows these three properties are fundamentally incompatible in the large-graph limit, a structural impossibility theorem with direct consequences for GNN design. To state it precisely, let $\mathbb{P}_\mathcal{G}$ denote the transition matrix of the lazy simple random walk on $\mathcal{G}$, $\mathbb{P}_\mathcal{G}(u,v) = \tfrac{1}{2}\delta_{uv} + \tfrac{\mathbf{1}_{\{u,v\}\in\mathcal{E}}}{2\deg(u)},$ and let $\lambda_2(\mathcal{G})$ be its second-largest eigenvalue. A small $\lambda_2$ indicates strong mixing; a large $\lambda_2$ (close to $1$) signals slow convergence and bottlenecks.

\begin{definition}[Perfect Graph]\label{def:perfect}
Given constants $\Delta\ge 1$ and $\rho<1$, a graph $\mathcal{G}$ is called \emph{perfect} with respect to $(\Delta,\rho)$ if it simultaneously satisfies: \textit{(A) Sparsity}, $\Delta(\mathcal{G})\le\Delta$; \textit{(B) Spectral expansion}, $\lambda_2(\mathcal{G})\le\rho$; and \textit{(C) Non-negative curvature}, $\kappa_w(\mathcal{G})\ge 0$.
\end{definition}

A perfect graph, as defined in \ref{def:perfect}, would be ideal for a GNN: cheap to run, free of bottlenecks, and geometrically stable. The following result, which extends \citep{salez2022} to our entropic curvature framework, proves that no sufficiently large graph can be perfect.

\begin{theorem}[Expansion Paradox]\label{thm:expansion}
Let $(\mathcal{G}_n)_{n\ge 1}$ be a sequence of finite graphs satisfying the sparsity condition $\sup_{n\ge 1}\left\{\frac{1}{|\mathcal{V}_n|}\sum_{x\in\mathcal{V}_n}\deg_{\mathcal{G}_n}(x)\log\deg_{\mathcal{G}_n}(x)\right\}<\infty,$ and suppose that $\kappa_{\mathcal{G}_n}(x)\ge a>0$ for all but $o(|\mathcal{V}_n|)$ vertices. Then, for all $\rho<1$, $\liminf_{n\to\infty}\frac{1}{|\mathcal{V}_n|}\#\bigl\{i:\lambda_i(\mathcal{G}_n)\ge\rho\bigr\}>0.$

\end{theorem}
\begin{proof}
    See Appendix~\ref{app:proof_expansions}.
\end{proof}

Theorem~\ref{thm:expansion} resolves a long-standing tension in GNN design. Graphs that are sparse and spectrally well-connected, exactly those that appear in standard benchmarks such as citation, social, and web networks, \emph{must} exhibit negative entropic curvature. Since negative $\kappa_w$ permits rapid diffusion of probability mass, it explains why long-range message-passing on these graphs is prone to oversquashing. Conversely, positive curvature controls this diffusion via the Poincaré inequality of Theorem~\ref{thm:poincare-gnn}, preventing oversmoothing. The paradox thus unifies oversmoothing and oversquashing as two manifestations of the same geometric trade-off: they are not independent failure modes, but opposite ends of a single curvature spectrum.

\begin{figure}[h]
  \centering
  \includegraphics[width=\linewidth]{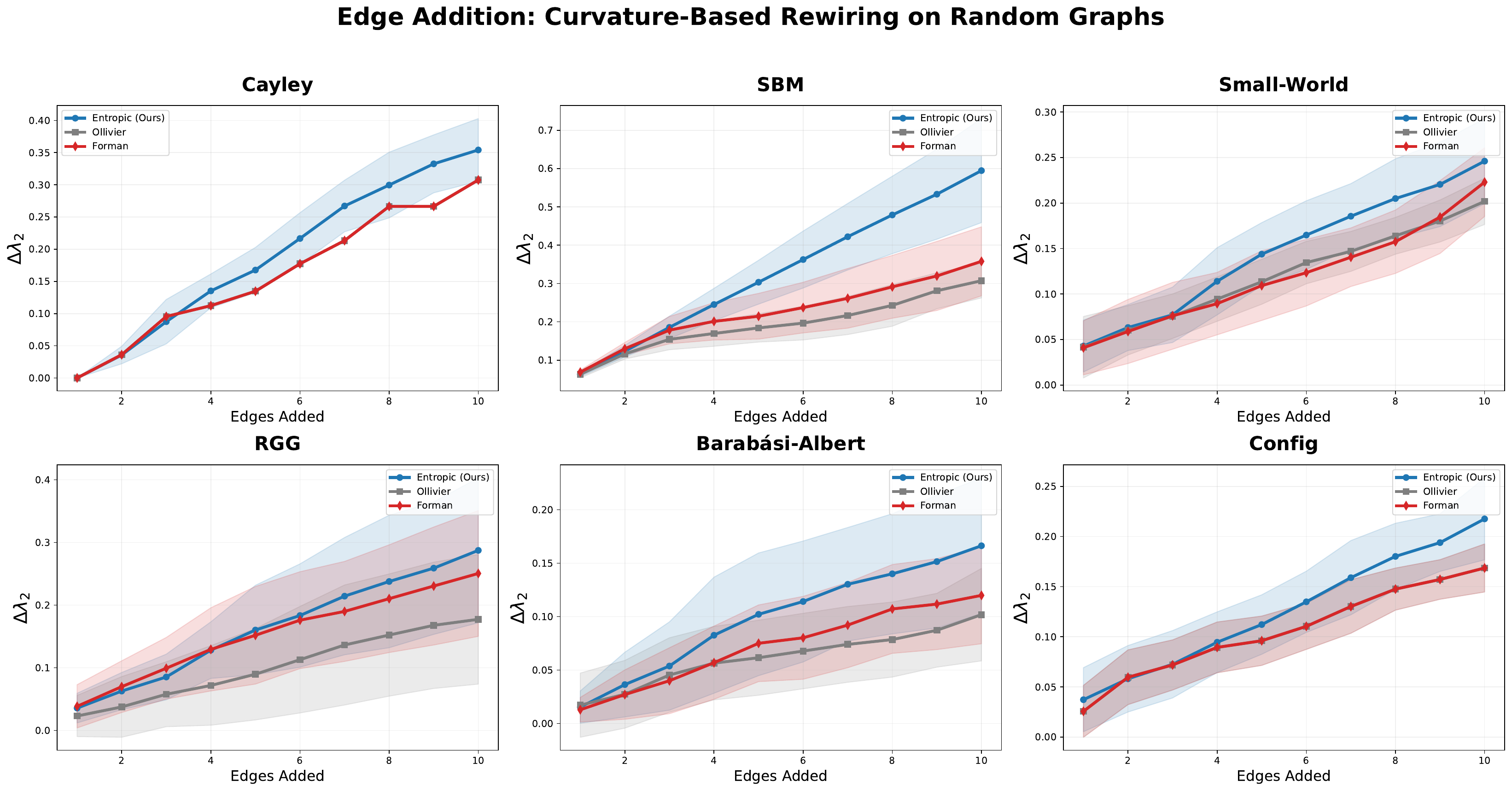}
  \caption{Spectral gap improvement $\Delta\lambda_2$ under curvature-guided edge addition (10 edges, 10 seeds, shaded $\pm$1 std). MCR (ours) consistently outperforms ORC and FRC across all six graph families.}
  \label{fig:rewiring}
\end{figure}

\textbf{Curvature-Based Rewiring.}
The expansion paradox shows that sparse graphs cannot simultaneously be strongly expanding and positively curved in the large-graph regime. This does not mean, however, that negative curvature is fixed by the input topology. Since $\kappa_w$ is determined by the two-hop transport geometry around each node, one can locally modify the graph by adding edges that create additional geodesic midpoints. The purpose of curvature-based rewiring is therefore not to make the graph uniformly positively curved, but to reduce the severity of negative-curvature bottlenecks while preserving sparsity. The following result gives the local principle behind our rewiring strategy: adding alternative two-hop routes cannot decrease the weak entropic curvature at the rewired node.

\begin{theorem}[Curvature Monotonicity Under Rewiring]
\label{thm:rewiring}
Let $\mathcal{G}^+$ be obtained from $\mathcal{G}$ by adding edges that create new two-hop geodesics from $z\in\mathcal{V}$ to nodes in $\mathrm{S}_2(z)$, without removing any existing two-hop geodesic. Define the midpoint set
$\mathrm{M}_\mathcal{G}(z,w):=\{m\in\mathcal{V}: z\sim m,\; m\sim w\}.$ If $\mathrm{M}_\mathcal{G}(z,w)\subseteq \mathrm{M}_{\mathcal{G}^+}(z,w)$ for every $w\in\mathrm{S}_2(z)$, then $\kappa_w^{\mathcal{G}^+}(z)\ge \kappa_w^\mathcal{G}(z),$
with strict improvement whenever $|\mathrm{M}_{\mathcal{G}^+}(z,w)|>|\mathrm{M}_\mathcal{G}(z,w)|$ for some $w$.
\end{theorem}
\begin{proof}
   See Appendix~\ref{app:proof_rewiring}.
\end{proof}
The theorem formalizes a simple geometric mechanism: increasing the number of admissible midpoints makes two-hop transport less concentrated, thereby improving the local entropic-curvature certificate. Based on this principle, our \emph{Midpoint-Completion Rewiring} (MCR) procedure repeatedly targets the most negatively curved region and adds an edge that creates an additional midpoint for a poorly supported two-hop connection. We give the implementation details in Appendix~\ref{app:curvature_based_rewiring}.

This monotonicity result gives a local curvature certificate for each accepted rewiring step. Its downstream consequences follow from the earlier theory: when rewiring increases the graph-level lower bound $\kappa_w$, Theorem~\ref{thm:lower_bound} yields a tighter certificate for the global entropic curvature, and in the positive-curvature regime Theorem~\ref{thm:robustness} gives a sharper transport--entropy generalization bound. The effect on spectral expansion is evaluated empirically, since the expansion paradox identifies a structural trade-off rather than a stepwise monotonicity guarantee. Figure~\ref{fig:rewiring} shows that MCR consistently improves the spectral gap more than Ollivier--Ricci (ORC) and Forman--Ricci (FRC) rewiring across all six random graph families. The gains are especially pronounced on community-structured graphs, where adding midpoint-completion edges directly targets transport bottlenecks between regions rather than relying only on one-hop or degree-based curvature information.

\begin{wrapfigure}{r}{0.6\linewidth}

  \centering
  \includegraphics[width=\linewidth]{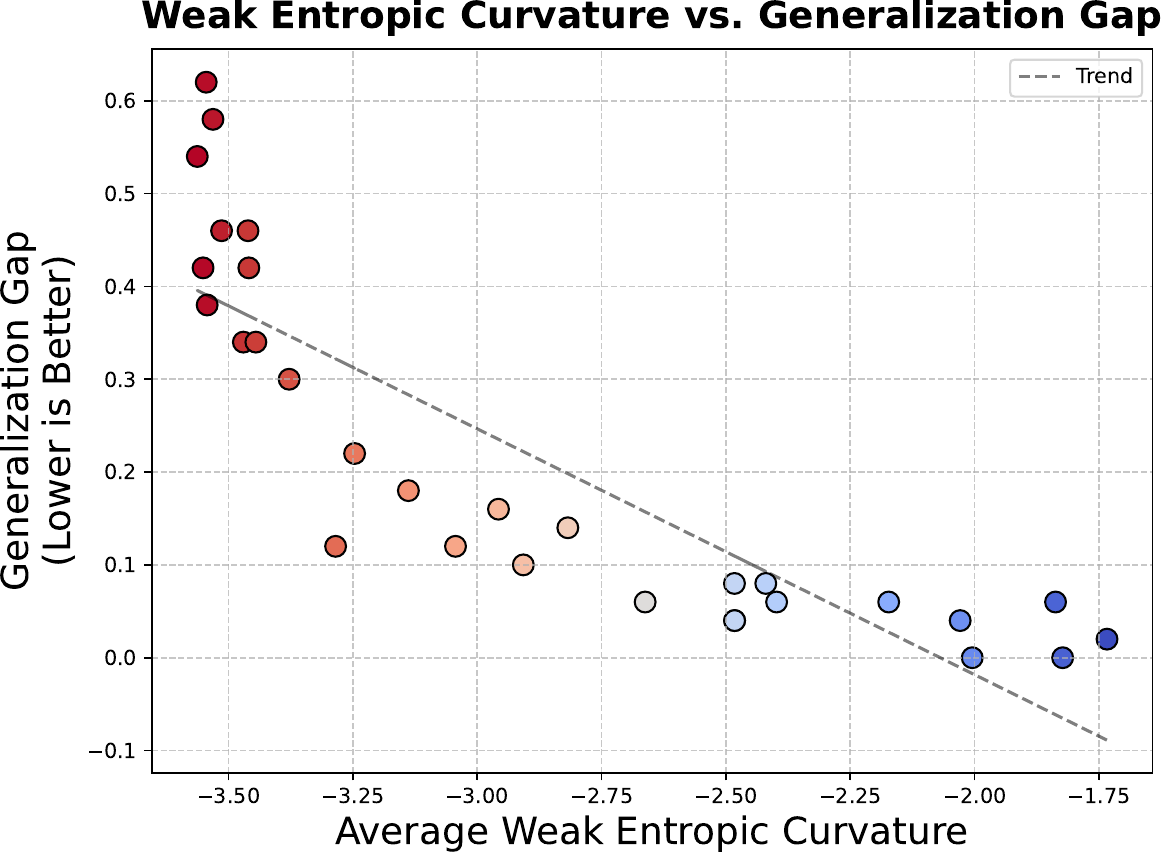}
  \caption{Empirical correlation between $\bar\kappa_w$ and the GNN generalization gap on SBM graphs.}
  \label{fig:sbm_generalization}
\end{wrapfigure}

%% file: Content/Robustness.tex
A central question in GNN theory is how well a model trained on one part of a graph transfers to another, especially when the training and test nodes occupy different regions of the topology~\citep{wangmanifold}. We formalize this discrepancy by viewing the training and test sets as probability measures $\mu_{\mathrm{train}}$ and $\mu_{\mathrm{test}}$ on $\mathcal{V}$. The generalization gap $\Delta\mathcal{L}(f_\theta)
:=
\left|
\mathcal{L}_{\mu_{\mathrm{test}}}(f_\theta)
-
\mathcal{L}_{\mu_{\mathrm{train}}}(f_\theta)
\right|$ then measures the sensitivity of the learned predictor to a distributional shift over the graph.

\begin{remark}[Entropy finiteness and applicability]
\label{rem:entropy_finite}
Theorems~\ref{thm:robustness} and~\ref{thm:gen_bound_negative} require that $\mathrm{Ent}_{\mu_{\mathrm{train}}}(\mu_{\mathrm{test}}) < \infty$. This holds whenever $\mu_{\mathrm{test}}$ is absolutely continuous with respect to $\mu_{\mathrm{train}}$, i.e., $\mathrm{supp}(\mu_{\mathrm{test}}) \subseteq \mathrm{supp}(\mu_{\mathrm{train}})$. In the standard node classification setting with disjoint train/test splits, this condition is \emph{not} automatically satisfied (since the splits are disjoint), and the relative entropy is infinite. The theorems therefore apply in the \emph{overlapping} or \emph{soft-split} regime, or when $\mu_{\mathrm{train}}$ and $\mu_{\mathrm{test}}$ are defined as smoothed distributions over the node set (e.g., with a small positive mass on every node). In our empirical analysis (Figure~\ref{fig:sbm_generalization}), we use a random split with overlap to ensure finiteness. We further note that the transport-entropy inequality (Lemma~\ref{lemm:westrass_controled_entropy}) is derived using the graph's invariant reference measure $\mathfrak{m}$ as guaranteed by the positive curvature condition; the generalization bound then applies with $\mu_0 = \mu_{\mathrm{train}}$ playing the role of a reference measure for the relative entropy term, which is a standard application of the Kantorovich-Rubinstein duality and does not require $\mu_{\mathrm{train}} = \mathfrak{m}$.
\end{remark}

Entropic curvature provides a natural way to control this sensitivity: through transport--entropy inequalities, curvature bounds convert the cost of moving from $\mu_{\mathrm{train}}$ to $\mu_{\mathrm{test}}$ into an explicit bound on the change in risk.

\begin{theorem}[Generalization Bound under Positive Entropic Curvature; Proof in Appendix~\ref{app:proof_thm_robustness}]
\label{thm:robustness}
Assume that $\mathcal{G}$ has weak entropic curvature bounded below by $\kappa_w>0$. Suppose the loss $\mathcal{J}$ is $\sigma$-Lipschitz and the GNN $f_\theta$ is globally $\Lambda$-Lipschitz with respect to the graph distance $d$, i.e., $\|f_\theta(x)-f_\theta(y)\|\le \Lambda d(x,y)
\qquad \forall x,y\in\mathcal{V}.$
Then, for any test distribution $\mu_{\mathrm{test}}$, 
\[
\left|
\mathcal{L}_{\mu_{\mathrm{test}}}(f_\theta)
-
\mathcal{L}_{\mu_{\mathrm{train}}}(f_\theta)
\right|
\le
\sigma\Lambda
\sqrt{
\frac{2}{\kappa_w}
\,
\mathrm{Ent}_{\mu_{\mathrm{train}}}(\mu_{\mathrm{test}})
}.
\]
\end{theorem}

\begin{proof}
    See Appendix~\ref{app:proof_thm_robustness}.
\end{proof}

Theorem~\ref{thm:robustness} separates the generalization gap into three interpretable factors. The product $\sigma\Lambda$ captures the sensitivity of the loss and the model, while $\mathrm{Ent}_{\mu_{\mathrm{train}}}(\mu_{\mathrm{test}})$ measures the informational cost of shifting from the training distribution to the test distribution. The curvature term $1/\sqrt{\kappa_w}$ acts as a geometric amplifier: larger positive curvature yields a tighter certificate, making the model less sensitive to the same train--test distributional shift. Positive curvature gives the cleanest transport--entropy bound, but many sparse real-world graphs contain negatively curved regions. The next result shows that a meaningful generalization certificate still holds when curvature is negative, at the cost of a diameter-dependent factor.

\begin{theorem}[Generalization Bound under Negative Entropic Curvature; ]
\label{thm:gen_bound_negative}
Suppose $\mathcal{G}$ is finite and connected with diameter $D<\infty$, and assume its weak entropic curvature satisfies $\kappa_w\ge -K$ for some $K>0$. Let the loss be $L$-Lipschitz with respect to the predictions $f_\theta(x)$, and let $f_\theta$ be globally $\Lambda$-Lipschitz with respect to the graph distance $d$. Then,
\[
\left|
\mathcal{L}_{\mu_{\mathrm{test}}}(f_\theta)
-
\mathcal{L}_{\mu_{\mathrm{train}}}(f_\theta)
\right|
\le
L\Lambda
\sqrt{
\frac{4}{K}
\left(
e^{KD^2/8}-1
\right)
\mathrm{Ent}_{\mu_{\mathrm{train}}}(\mu_{\mathrm{test}})
}.
\]
\end{theorem}

\begin{proof}
    See Appendix~\ref{app:proof_thm_neg_bound}.
\end{proof}
Together, Theorems~\ref{thm:robustness} and~\ref{thm:gen_bound_negative} show that entropic curvature controls generalization in both curvature regimes. When $\kappa_w>0$, the dependence scales as $1/\sqrt{\kappa_w}$. When $\kappa_w\ge -K$, the bound remains finite but deteriorates with both the negative-curvature scale $K$ and the graph diameter $D$, reflecting the greater sensitivity of negatively curved, expansive regions to distributional shift.

\textbf{Empirical Analysis of Generalization and Curvature.} To empirically evaluate the relationship between graph geometry and learning performance, we analyze the correlation between the Average Weak Entropic Curvature ($\kappa_w$) and the Generalization Gap. Following the framework in Theorem \ref{thm:robustness}, we define the generalization gap as the absolute difference between training and testing accuracies $\Delta\mathcal{A} = \text{Acc}_{\text{train}} - \text{Acc}_{\text{test}}$. This experiment tests the hypothesis that positive curvature lower bounds effectively tighten the model's sensitivity to distributional shifts between the training and testing measures. Our pipeline utilizes a Stochastic Block Model (SBM) to generate $30$ distinct graphs, each containing 100 nodes divided into two equal communities. We systematically modulate the graph's curvature by varying the intra-community edge probability $p_{in}$ on a sliding scale from $5\%$ to $10\%$, while maintaining a fixed inter-community probability $p_{out} = 5\%$. This transitions the topology from sparse, tree-like regimes to dense, clustered configurations. We note that varying $p_{in}$ simultaneously increases both graph curvature and community signal-to-noise ratio; as a result, the observed correlation between curvature and generalization gap may be partially driven by the task becoming easier on denser graphs. This is an acknowledged confound: the experiment provides correlational evidence consistent with the theory, but should not be interpreted as a controlled causal test.  For each graph, we compute the node-wise weak entropic curvature by solving the exact optimization for the local functional $\mathrm{H}_\Lop$. The global signature is defined as the mean curvature across all vertices: $\bar{\kappa}_w = \frac{1}{|\mathcal{V}|} \sum_{z \in \mathcal{V}} \kappa_w(z)$. We use a standard 2-layer GCN \citep{kipf2017semisupervised}, and the node features are initialized as an identity matrix. The model is trained for 200 epochs using Adam optimizer to predict community membership \citep{kingma2014adam}.

As illustrated in Figure \ref{fig:sbm_generalization}, the empirical results reveal a significant negative correlation: higher average weak entropic curvature consistently corresponds to a lower generalization gap. Since the curvature values in this experiment are strictly negative (ranging from approximately $-2.5$ to $-1.75$), the evaluated graphs fall in the negative-curvature regime. The observed trend is therefore consistent with Theorem~\ref{thm:gen_bound_negative}: as $\kappa_w$ increases toward zero (i.e., $K$ decreases), the bound tightens, limiting how much the testing distribution can stray from the training configuration. The trend also aligns directionally with Theorem~\ref{thm:robustness} as the curvature approaches the positive regime, confirming a monotone geometric effect across both regimes.

%% file: Content/experiments.tex
\textbf{Adapting GNN Dynamics to Entropic Geometry.}
We evaluate $\kappa_w$ as a lightweight geometric prior for standard GNNs on six node-classification benchmarks spanning homophilic regimes (Cora, CiteSeer, PubMed~\citep{yang2016revisiting}) and heterophilic regimes (Cornell,  Wisconsin~\citep{pei2020geom}). Motivated by Theorem~\ref{thm:poincare-gnn}, we introduce \emph{E-Gate}, a node-wise modulation $h_u^{(\ell+1)}=\mathrm{UPDATE}^{(\ell)}(h_u^{(\ell)},\mathcal{R}(u)\,\mathrm{AGG}^{(\ell)}_u)$ with $\mathcal{R}(u)=\sigma(-\tau\kappa_w(u))$, where $\tau$ is a learnable temperature, and $\sigma(\cdot)$ is the Sigmoid function. The current framework is motivated by Lemma~\ref{lemm:westrass_controled_entropy}: $\sigma(-\tau\kappa_w)$ attenuates aggregation in proportion to the node's curvature-induced transport resistance, cooling positively curved  regions and amplifying negatively curved  ones. 

\input{Tables/egate_gnns}

Table~\ref{tab:gcn_results} shows that E-Gate generally improves performance across backbones and graph regimes, with the largest gains on heterophilic datasets (Cornell, Wisconsin) where standard aggregation struggles with bottlenecks; on homophilic citation graphs the improvements are smaller but mostly positive. Appendix~\ref{app:additional_experiments} reports additional  experiments for the graph classification task.

\textbf{Entropic Curvature as a Structural Encoding.}
GNN expressivity is bounded by the 1-Weisfeiler-Lehman test~\citep{feng2022powerful}, motivating structural encodings derived from spectral properties or discrete geometry. Local Curvature Profiles (LCP) inject Ollivier--Ricci curvature statistics~\citep{fessereffective}, but discrete curvatures primarily capture local clustering and may miss broader transport bottlenecks. We hypothesize that $\kappa_w$, derived from the convexity of entropy along Wasserstein geodesics, provides a more globally-aware structural prior. We evaluate ENT against Laplacian Eigenmaps (LAPE)~\citep{belkin2003laplacian}, Random Walk PE (RWPE)~\citep{eliasof2023graph}, Ego-Network encodings (EGO)~\citep{alvarez2023beyond}, and LCP~\citep{fessereffective}. Following the LCP protocol, we augment $\mathrm{X}$ with a 3-dimensional curvature signature; details in Appendix~\ref{app:lcp_ent}.

\input{Tables/lcp_ent}
Table~\ref{tab:node_classification_reduced} reports node-classification accuracy across five benchmarks. ENT achieves the best performance on 4 of 5 datasets, outperforming spectral and discrete-curvature baselines. ENT underperforms LCP on Wisconsin ($44.28$ vs.\ $47.14$, within 1 std). On heterophilic graphs Texas (+3.9\% vs.\ LCP) and Cornell (+0.6\%), where bottlenecks are frequent, the entropic prior is significantly more informative than local metrics. On homophilic graphs ENT maintains a consistent lead.

%% file: Tables/egate_gnns.tex
\begin{table*}[t]
\centering
\caption{Test accuracy (\%, mean $\pm$ std) of baseline backbones and their E-Gate variants across five node classification benchmarks. Best results between baseline and E-Gate are shown in bold.}
\label{tab:gcn_results}
\resizebox{\textwidth}{!}{%
\begin{tabular}{lcccccccccc}
\toprule
& \multicolumn{2}{c}{Cora} 
& \multicolumn{2}{c}{CiteSeer} 
& \multicolumn{2}{c}{PubMed} 
& \multicolumn{2}{c}{Cornell} 
& \multicolumn{2}{c}{Wisconsin} \\
\cmidrule(lr){2-3} \cmidrule(lr){4-5} \cmidrule(lr){6-7}
\cmidrule(lr){8-9} \cmidrule(lr){10-11}
Backbone 
& Base & E-Gate
& Base & E-Gate
& Base & E-Gate
& Base & E-Gate
& Base & E-Gate \\
\midrule
GCN
& 86.68 $\pm$ 1.07 & \textbf{88.08 $\pm$ 1.04}
& \textbf{76.23 $\pm$ 0.76} & 75.21 $\pm$ 1.41
& \textbf{87.29 $\pm$ 0.47} & 87.16 $\pm$ 0.49
& 41.35 $\pm$ 6.05 & \textbf{44.86 $\pm$ 6.53}
& 43.73 $\pm$ 4.81 & \textbf{44.51 $\pm$ 6.20} \\

SAGE
& 86.77 $\pm$ 1.78 & \textbf{87.56 $\pm$ 1.58}
& \textbf{75.11 $\pm$ 1.95} & 74.65 $\pm$ 1.36
& 88.45 $\pm$ 0.27 & \textbf{89.03 $\pm$ 0.52}
& 61.89 $\pm$ 3.51 & \textbf{67.84 $\pm$ 8.33}
& 75.10 $\pm$ 5.95 & \textbf{79.22 $\pm$ 5.35} \\

GIN
& 82.58 $\pm$ 2.08 & \textbf{83.03 $\pm$ 2.53}
& \textbf{69.55 $\pm$ 2.25} & 69.47 $\pm$ 1.67
& 84.45 $\pm$ 4.85 & \textbf{85.81 $\pm$ 0.60}
& 42.16 $\pm$ 8.39 & \textbf{49.46 $\pm$ 7.26}
& \textbf{53.33 $\pm$ 5.67} & 53.33 $\pm$ 7.58 \\

GAT
& 86.29 $\pm$ 1.73 & \textbf{86.77 $\pm$ 1.27}
& \textbf{74.79 $\pm$ 1.28} & 74.71 $\pm$ 1.58
& 86.26 $\pm$ 0.37 & \textbf{86.44 $\pm$ 0.40}
& 40.00 $\pm$ 8.09 & \textbf{46.22 $\pm$ 10.36}
& 47.84 $\pm$ 6.27 & \textbf{49.22 $\pm$ 7.15} \\
\bottomrule
\end{tabular}%
}
\end{table*}

%% file: Tables/lcp_ent.tex
\begin{table}

\centering
\caption{Structural encoding accuracy (\%, mean$\pm$std, GCN backbone). Best results are shown in bold}
\label{tab:node_classification_reduced}

\resizebox{0.8\linewidth}{!}{%
\begin{tabular}{lccccc}
\toprule
\textbf{Dataset} & \textbf{ENT (ours)} & \textbf{LCP} & \textbf{LAPE} & \textbf{RWPE} & \textbf{EGO} \\
\midrule
Cornell   & $\mathbf{42.60\pm3.54}$ & $41.96\pm4.44$ & $38.91\pm2.42$ & $38.47\pm3.53$ & $41.09\pm3.81$ \\
Wisconsin & $44.28\pm3.22$ & $\mathbf{47.14\pm3.42}$ & $42.22\pm2.95$ & $43.65\pm5.24$ & $42.70\pm2.67$ \\
PubMed    & $\mathbf{86.18\pm0.23}$ & $85.76\pm0.20$ & $85.37\pm0.18$ & $85.77\pm0.18$ & $85.71\pm0.26$ \\
Texas     & $\mathbf{46.95\pm4.48}$ & $43.04\pm1.92$ & $43.48\pm4.52$ & $43.04\pm3.52$ & $45.43\pm3.86$ \\
CiteSeer  & $\mathbf{71.86\pm0.93}$ & $71.13\pm0.78$ & $71.34\pm0.92$ & $71.63\pm1.08$ & $70.97\pm0.98$ \\
\bottomrule
\end{tabular}

}
\end{table}

%% file: Content/conclusion.tex
This work introduces entropic curvature as a global and theoretical measure to analyze GNNs. In particular, a positive entropic‐curvature lower bound yields log-Sobolev and transport–entropy inequalities that directly control information propagation in message-passing architectures. These inequalities translate into quantitative guarantees against oversmoothing, limits on distributional shift, and provable generalization. Entropic curvature offers a unifying geometric principle that enhance our theoretical understanding of GNN dynamics and yields concrete algorithmic benefits. We hope this framework will inspire further exploration of global geometric insights, not only for generalization and oversmoothing, but also for curriculum learning on evolving graphs, principled regularization of large-scale graph models, and the broader integration of optimal-transport ideas into machine learning.  As for limitations, Entropic curvature depends on the choice of reference measure $\mathfrak{m}$; learning
the measure that maximises $\kappa_w$ is a promising open direction that could clarify
how global geometry interacts with mass distribution and guide curvature-aware learning
algorithms. An interesting research direction would be to approximate the measure which maximises the Entropic Curvature. Answering this would clarify how global geometry interacts with mass distribution and could guide curvature-aware learning algorithms.

\newpage

%% file: Appendix/appendix_interpretation.tex
\textbf{Interpretation of the Local Functional.} To provide a geometric interpretation of $\mathrm{H}_{\Lop}$, it is instructive to specialize the functional to the case where the target set is a singleton, $\mathrm{W} = \{z''\}$, for some $z'' \in \mathrm{S}_2(z)$. The support of the optimization restricts to the set of strict midpoints $]z, z''[$. In this setting, the functional $\mathrm{H}_\Lop(z, \{z''\})$ simplifies to,
$$\mathrm{H}_{\Lop}(z, \{z''\}) = \Lop^2(z,z'') \sup_{\alpha} \; \prod_{z' \in ]z,z''[} \left( \frac{\alpha(z')}{\Lop(z, z')} \right)^{2 \ell(z, z', z'')}.$$

By applying the transformation $-2 \log(\cdot)$. We convert the product maximization into a summation minimization. This yields the following energy functional, 
\begin{equation}
\label{eq:sup_log}
    -2 \log \mathrm{H}_{\Lop}(z, {z''}) = \inf\limits_{\alpha} \Bigg[ -4 \sum\limits_{z' \in ]z,z''[} \ell(z, z', z'') 
    \quad \times \log \left( \frac{\alpha(z')}{\Lop(z, z')} \right) \Bigg] + \mathrm{C}_{\Lop},
\end{equation}

where $\mathrm{C}_{\Lop} = -2 \log \Lop^{(2)}(z,z'')$ is a constant independent of $\alpha$ (here $\Lop^{(2)}(z,z'')$ denotes the 2-step transition weight). The factor of $-4$ arises because $-2\log(\cdot)$ applied to $(\alpha(z')/\Lop(z,z'))^{2\ell}$ yields $-2 \cdot 2\ell \cdot \log(\alpha(z')/\Lop(z,z')) = -4\ell\log(\alpha(z')/\Lop(z,z'))$. Recall that $\sum_{z'}\ell (z,z',z'')=1$ and $\sum_{z'}\alpha(z')=1$. By introducing the entropy of the geometric weights $\ell(z, \cdot, z'')$ into the sum, the equation \eqref{eq:sup_log} can be rewritten directly in terms of the Kullback–Leibler divergence $\mathcal{KL}(\cdot \parallel \cdot)$,
\begin{equation*}
\label{eq:KL_div}
    \inf_{\alpha} \left[ 4 \sum_{z' \in ]z,z''[} \ell(z, z', z'') \log \left( \frac{\ell(z, z', z'')}{\alpha(z')} \right) \right] + \mathrm{C}'_L 
\quad = 4 \inf_{\alpha} \mathcal{KL}\Big( \ell(z, \cdot, z'') \parallel \alpha(\cdot) \Big) + \mathrm{C}'_L,
\end{equation*}

where $\mathrm{C}'_L$ is a constant independent from $\alpha(\cdot)$. Here, we use the property that both the path weights $\ell$ and the transport plan $\alpha$ sum to unity over the midpoints. This derivation reveals that the weak entropic curvature locally penalizes the divergence between the geometry of the graph (encoded by the relative path weights $\ell$) and the optimal transport plan ($\alpha$).
When we return to the full definition where $\mathrm{W}=\mathrm{S}_2(z)$, the functional $\mathrm{H}_{\Lop}(z, \mathrm{S}_2(z))$ acts as a weighted regularization of these pairwise divergences over all possible targets $z''$.

\paragraph{Geometric Characterization via the Motzkin–Straus Theorem.} 

To provide more geometric intuition for the local functional $\mathrm{H}_\Lop$, we establish a link between local graph density and its optimization landscape. The key tool for this characterization is the \text{Motzkin--Straus Theorem} \citep{motzkin1965maxima}, which relates the clique number of a graph to a quadratic optimization problem over the probability simplex. In their work, \citep{motzkin1965maxima}, established this elegant link between the size of the maximum clique, defined as the cardinality of the maximum complete subgraph $g(\mathcal{G})$, and the global maxima of a quadratic optimization problem on the standard simplex. By relating these global maxima to the functional $\mathrm{H}_{\Lop_0}(z, \mathrm{S}_2(z))$, we can interpret the graph's local geometry in a useful way that captures the branching of the space and the local expansion or contraction of the metric structure.

Consider a vertex $u \in \mathcal{V}$ where the local geometry satisfies a specific regularity condition: every node in the 2-sphere $\mathrm{S}_2(u)$ shares exactly two midpoints with the central node $u$, as visualized in Figure \ref{fig:positiveEntropy}. Under the $\Lop_0$ generator and uniform measure, the weak entropic curvature directly reflects the density of the local neighborhood through the following structural characterization,

\begin{theorem}[Motzkin–Straus Type Identity; Proof in Appendix \ref{app:proof_Strauss}] \label{thm:Motzkin-Straus}
Let $u$ be a node in $\mathcal{G}$ and assume that $\mathrm{B}_2(u)$ satisfies the condition that any $v \in \mathrm{S}_2(u)$ shares exactly two midpoints with $u$. Then,
\begin{equation}
    \mathrm{H}_{\Lop_0}(z, \mathrm{S}_2(u)) = 1 - \frac{1}{g(\mathrm{B}_2(u))}. 
\end{equation}
\end{theorem}
This identity reveals that the magnitude of the local functional $\mathrm{H}_{\Lop_0}$ serves as a signature of local redundancy and clustering. As the maximum clique size within the 2-hop neighborhood increases, the optimization value reflects a geometry that supports highly efficient, isotropic diffusion. Geometrically, these high-density regions act as \emph{local traps} or dense communities where probability mass spreads uniformly. Conversely, as the clique number decreases, approaching tree-like regime, the functional value shifts, signaling the presence of structural bottlenecks where information flow is constrained to fewer, more fragile paths.

%% file: Appendix/Entropic_curvature_some_calculations.tex
\subsection{Bridging Local Computation and Global Topology}
\label{app:local_global_bridge}

A key advantage of our framework is its ability to control global graph properties through local optimization. While Entropic Curvature is defined by the convexity of entropy along global Wasserstein geodesics, we derive a tractable local proxy, $\kappa_w$, computed strictly within 2-hop neighborhoods. This proxy acts as a lower bound for the global curvature. Consequently, by locally rewiring edges to maximize $\kappa_w$, we rigorously improve the global spectral gap and alleviate the bottleneck phenomena without incurring the prohibitive cost of global spectral analysis.

\subsection{Canonical Motifs and Entropic Curvature Regimes}

Figure~\ref{fig:motifs} visualizes the local weak entropic curvature $\kappa_w$ across several canonical graph motifs. Dense or highly redundant neighborhoods exhibit larger curvature because two-hop mass transport can be spread across multiple intermediate nodes. In contrast, stars, trees, and bottleneck-like structures have lower or negative curvature, reflecting the limited number of available geodesic midpoints and the resulting concentration of transport mass. These examples provide concrete intuition for the role of $\kappa_w$ as a proxy for local transport redundancy.

\begin{figure}[h]
  \centering
  \includegraphics[width=\linewidth]{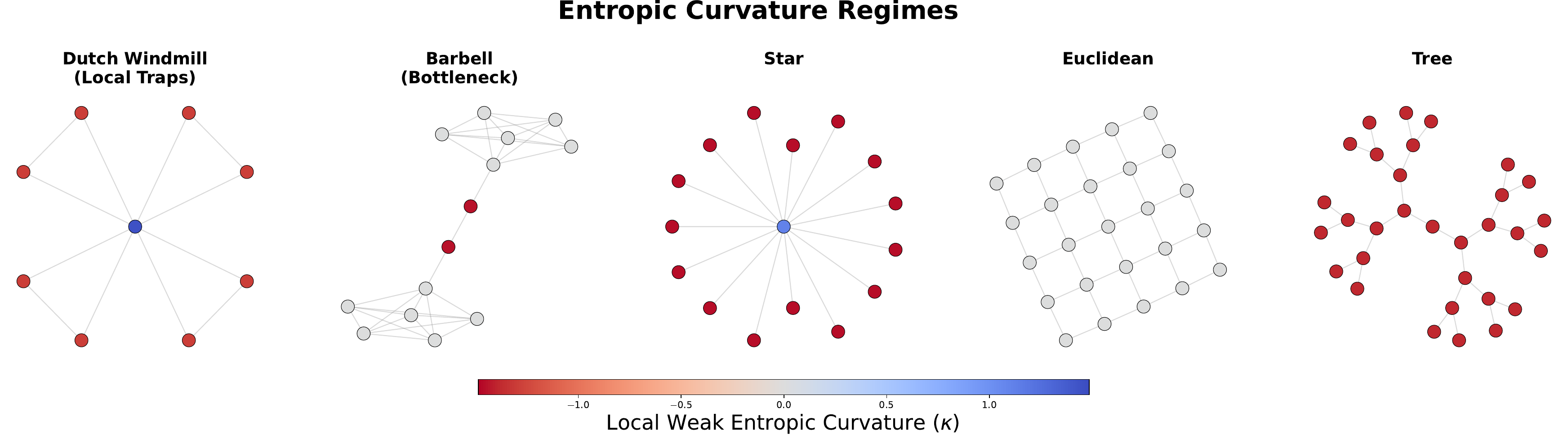}
  \vspace{-1.5em}
  \caption{Local weak entropic curvature across various graph motifs.}
  \label{fig:motifs}
\end{figure}

\subsection{Theoretical Calculation of Some Examples of Weak Entropic Curvature}
\begin{figure*}[h]
    \centering
    \includegraphics[width=0.9\textwidth]{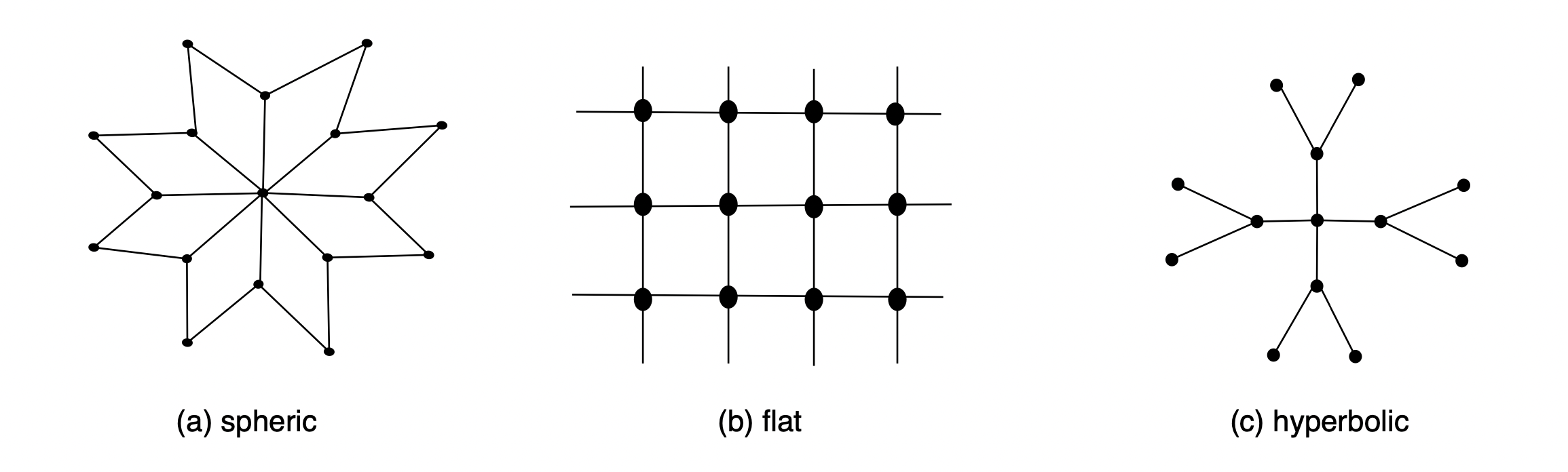}
    \caption{\textbf{Characterization of discrete geometric regimes via Weak Entropic Curvature ($\kappa_w$) under the uniform generator $L_0$.}  (a) \textit{Spherical ($\kappa_w > 0$):} High local clustering results in positive curvature, satisfying the lower bound $\kappa_w(z) \ge -2 \log(1 - \frac{1}{\deg(z)})$.  (b) \textit{Flat ($\kappa_w = 0$):} The Euclidean grid exhibits zero curvature, corresponding to a local functional value $H_{L_0}(z, S_2(z)) = 1$.  (c) \textit{Hyperbolic ($\kappa_w < 0$):} Tree-like structures characteristic of exponential expansion yield strictly negative curvature, specifically $\kappa_w(z) = -2 \log 2 \approx -1.386$.}
\label{fig:curvature_regimes}
    \label{fig:toy_graphs}
\end{figure*}

In this appendix, we provide a theoretical computation  of the weak entropic curvature in some specific configurations. Throughout, we work with the $\Lop_0$ generator (see \eqref{ex:uniform}).

 \textbf{Spherical configuration}: Case (a) in figure \ref{fig:toy_graphs}.

Let $z$ denote the center. We consider
\begin{equation}\label{eq_HL0}
    \begin{aligned}
 \mathrm{H}_{\Lop_0}(z, \mathrm{S}_2(z))  := \sup_{\alpha} \left\{ 
\sum_{z'' \in \mathrm{S}_2(z)} \#]z,z''[ \prod_{z' \in ]z, z''[} 
\left( \alpha(z') \right)^{\frac{2}{\#]z,z''[}}
\right\}
\end{aligned}
\end{equation}

such that $\sum_{z' \in \mathrm{S}_1(z)} \alpha(z')=1$.
In this geometry we have $\#]z,z''[ = 2$ for every $z'' \in \mathrm{S}_2(z)$. 
Hence,
$\sum_{z'' \in \mathrm{S}_2(z)} \#]z,z''[ \prod_{z' \in ]z, z''[} 
\left( \alpha(z') \right)^{\frac{2}{\#]z,z''[}} \leq (\sum_{z'\in \mathrm{S}_1(z)}\alpha(z'))^2-\sum_{z'\in \mathrm{S}_1(z)} \alpha(z')^2\leq \big( 1-\frac{1}{deg(z)} \big)$.

We therefore obtain in this case the following lower bound for the weak entropic curvature,
$\kappa_{w}(z) \ge -2 \log\Big(1 - \tfrac{1}{\deg(z)}\Big).$

Observe that this quantity becomes positive once $\deg(z)$ is sufficiently large ($deg(z) \geq 2$).

\textbf{Flat configuration:} Case (b) in figure \ref{fig:toy_graphs}.

Let $z$ be an arbitrary node. In this configuration, the quantity $\#]z,z''[$ takes values either $1$ or $2$.
Enumerate the vertices of $\mathrm{S}_1(z)$ as $z_1,z_2,\dots ,z_N$, and set $z_{N+1}=z_1$ for convenience.

Repeating the same computation as before, the expression inside the braces in \eqref{eq_HL0} becomes
$$
\sum_{i}\alpha(z_i)^2 + 2\sum_{i}\alpha(z_i)\alpha(z_{i+1})
\le \left(\sum_{i}\alpha(z_i)\right)^2 = 1.
$$
Moreover, choosing $\alpha(z_i)=1$ for some fixed $i$ and $\alpha(z_j)=0$ for every $j\neq i$, gives
$$
\sum_{i}\alpha(z_i)^2 + 2\sum_{i}\alpha(z_i)\alpha(z_{i+1}) = 1.
$$

Hence, in this case we obtain $\mathrm{H}_{\Lop_0}(z,\mathrm{S}_2(z)) = 1$, which implies that $\kappa_w(z) = 0$.

\textbf{Hyperbolic Configuration }: Case (c) in figure \ref{fig:toy_graphs}.

Let $z$ be the center of the figure. Repeating the same calculation we get the expression inside the braces in \eqref{eq_HL0} becomes $2\sum_{z' \in \mathrm{S}_{1}(z)}\alpha(z')^2$. The supremum of this expression under the condition $\sum_{z' \in \mathrm{S}_1(z)} \alpha(z')=1$ is $2$ which implies $\kappa_w(z)= -2\log(2)<0$.

%% file: Appendix/curvature_based_rewiring.tex
\subsection{Motivation}

The spectral gap of the graph Laplacian, denoted as $\lambda_2$ (also known as the algebraic connectivity, is a fundamental metric in spectral graph theory \citep{chung1997spectral,friedland1999spectra}. It governs the rate of convergence of diffusion processes, the robustness of the network to bisection, and the expansion properties of the graph. In many real-world applications, such as synchronization in power grids or information flow in communication networks, maximizing $\lambda_2$ is desirable \citep{pecora1998master}.

Traditional edge addition strategies often rely on global metrics or random heuristics. However, recent advances in discrete geometry suggest that graph curvature, specifically \emph{bottleneck} regions characterized by negative curvature, dictates structural vulnerabilities \citep{topping2022understanding}. This experiment aims to demonstrate that a curvature-based rewiring strategy, i.e., targeting nodes with the lowest (most negative) curvature to be connected to the most distant nodes in the network, provides a more efficient path to spectral gap enhancement than traditional discrete curvatures across various synthetic graph topologies.

\subsection{Experimental Setup}

We evaluate the performance of our proposed Weak Entropic Curvature against two established discrete Ricci curvatures: Ollivier-Ricci Curvature (ORC) and Forman-Ricci Curvature (FRC). The experimental setup is as follows,

\begin{enumerate}
    \item \textbf{Graph Generation:} We initialize six distinct random graph models to ensure topological diversity,

        \begin{itemize}
            \item \textbf{Cayley:} Dihedral group-based graphs representing highly symmetric structures \citep{babai1979spectra}.

            \item \textbf{SBM:} Stochastic Block Models representing community-structured networks \citep{holland1983stochastic}.

            \item \textbf{Small-World:} Watts-Strogatz graphs characterized by high clustering and short path lengths \citep{wolfe2024effect}.

            \item \textbf{RGG:} Random Geometric Graphs where connectivity is distance-dependent \citep{penrose2003random}.

            \item \textbf{Barabási-Albert:} Scale-free networks with preferential attachment \citep{albert2002statistical}.

            \item \textbf{Config:} Configuration model graphs with fixed degree sequences \citep{fosdick2018configuring}.
        \end{itemize}

    \item \textbf{Iterative Rewiring:} For each model, we perform $10$ iterative edge additions. In each step,
    
        \begin{itemize}
            \item The curvature of all nodes is computed using the three respective metrics. For our Weak Entropic Curvature, $\kappa_w(z)$ is computed directly at the node level via Algorithm~\ref{alg:weak-curv-exact}. For Ollivier-Ricci Curvature (ORC) and Forman-Ricci Curvature (FRC), which are fundamentally defined as \emph{edge-level} quantities, node-level curvature is obtained by aggregating over incident edges using the \emph{minimum} function: $\kappa_{\mathrm{ORC}}(u) = \min_{v \sim u} \kappa_{\mathrm{ORC}}(u,v)$ and $\kappa_{\mathrm{FRC}}(u) = \min_{v \sim u} \kappa_{\mathrm{FRC}}(u,v)$. Using the minimum identifies the most bottlenecked incident edge, consistent with the rewiring objective of targeting structural vulnerabilities.
            \item The node $u$ with the minimum curvature value  is selected.
            \item An edge is added between $u$ and the node $v$ that is furthest from it in terms of shortest-path distance.
        \end{itemize}
    \item \textbf{Statistical Rigor:} Each experiment is repeated over 10 different random seeds. We report the mean change in spectral gap, $\Delta \lambda_2 = \lambda_2^{(t)} - \lambda_2^{(0)}$, where $t$ is the number of added edges. Shaded regions in the figures represent the standard deviation across seeds.
\end{enumerate}

\subsection{Results and Discussion}

The results of the edge addition experiment are illustrated in Figure~\ref{fig:rewiring} (main paper). Across all six graph topologies, the Entropic Curvature (Ours) strategy consistently outperforms both Ollivier and Forman Ricci curvatures in maximizing the spectral gap.

\begin{itemize}
    \item \textbf{Symmetry and Structure (Cayley \& SBM):} In the Cayley and SBM models, Entropic curvature shows a significantly steeper trajectory. In the SBM case specifically, Entropic curvature bridges communities more effectively, reaching a $\Delta \lambda_2 \approx 0.6$ after 10 edges, while Forman and Ollivier remain below $0.4$.

    \item \textbf{Expansion in Scale-Free and Small-World Graphs:} For Barabási-Albert and Small-World graphs, our metric demonstrates superior sensitivity to the bottlenecks that limit algebraic connectivity. While Forman and Ollivier often track closely together, Entropic curvature identifies superior edges that lead to faster global integration.

    \item \textbf{Robustness (RGG \& Config):} Even in Random Geometric Graphs, where spatial constraints often limit rewiring efficacy, our method maintains a lead. The Entropic curvature's ability to optimize a local distribution of weights allows it to identify critical bridge nodes that coarser metrics like Forman (based primarily on degrees) or Ollivier (based on 1-hop neighborhoods) may overlook.
\end{itemize}


\subsection{Midpoint-Completion Rewiring (MCR) Algorithm}

The Midpoint-Completion Rewiring strategy is motivated by the geometric interpretation of weak entropic curvature: $\kappa_w(z)$ measures how efficiently probability mass can be displaced from $z$ to its 2-sphere $\mathrm{S}_2(z)$ via two-hop paths. A bottleneck node, one with highly negative $\kappa_w(z)$, has few midpoints connecting it to its 2-hop neighbours, forcing probability mass through a small number of fragile paths. Adding edges that create new two-hop paths enriches the local transport geometry, increases $|\mathrm{M}_\mathcal{G}(z,w)|$, and thereby raises curvature, as certified by Theorem~\ref{thm:rewiring}.

\paragraph{Algorithm.}
The MCR procedure is summarised in Algorithm~\ref{alg:mcr}. At each step, we select the node with the most negative $\kappa_w$ value, identify the two-hop target $w^\star$ with the fewest existing midpoints (the most bottlenecked two-hop connection), and then add the edge $(z^\star, w^\star)$ directly. This choice is certified by Theorem~\ref{thm:rewiring}: adding an edge that creates a new midpoint for $w^\star \in \mathrm{S}_2(z^\star)$ strictly increases $|\mathrm{M}_\mathcal{G}(z^\star, w^\star)|$ and therefore non-decreasingly updates $\kappa_w(z^\star)$.

In the rewiring experiments of Appendix~\ref{app:curvature_based_rewiring}, we additionally explore a \emph{spectral variant} of MCR in which the new edge is drawn to the most geographically distant vertex $v^\star$ (identified by BFS) rather than directly to $w^\star$. This spectral variant is a heuristic motivated by the observation that bridging topologically distant nodes tends to most rapidly increase $\lambda_2$~\citep{spielman2010algorithms}. It does not in general satisfy the premise of Theorem~\ref{thm:rewiring} (since $v^\star$ may be far from $\mathrm{S}_2(z^\star)$), but is empirically effective for spectral gap maximization and is clearly labelled as a spectral heuristic in the experiments.

\begin{algorithm}[t]
\caption{Midpoint-Completion Rewiring (MCR)}\label{alg:mcr}
\KwIn{Graph $\mathcal{G}=(\mathcal{V},\mathcal{E})$, generator $\mathtt{L}$, edge budget $T \in \mathbb{N}$}
\KwOut{Rewired graph $\mathcal{G}_T$}
\For{$t = 1, \ldots, T$}{
    Compute $\kappa_w(z)$ for all $z \in \mathcal{V}$ \tcp*{Algorithm~\ref{alg:weak-curv-exact}}
    $z^\star \leftarrow \arg\min_{z \in \mathcal{V}}\; \kappa_w(z)$ \tcp*{Most negative bottleneck node}
    $w^\star \leftarrow \arg\min_{w \in \mathrm{S}_2(z^\star)}\; |\mathrm{M}_{\mathcal{G}}(z^\star, w)|$ \tcp*{Two-hop target with fewest midpoints}
    $m^\star \leftarrow \arg\min_{m \in \mathrm{S}_1(z^\star) \setminus \mathrm{S}_1(w^\star)} d(m, w^\star)$ \tcp*{Best new midpoint candidate}
    $\mathcal{E} \leftarrow \mathcal{E} \cup \{(m^\star, w^\star)\}$;\quad update $\mathcal{G}$, recompute $\mathrm{S}_2(z^\star)$\;
}
\Return $\mathcal{G}_T$\;
\end{algorithm}

\paragraph{Connection to Theorem~\ref{thm:rewiring}.}
Each edge addition in Algorithm~\ref{alg:mcr} is guaranteed by Theorem~\ref{thm:rewiring} to be non-decreasing in $\kappa_w(z^\star)$: by adding the edge $(m^\star, w^\star)$, we introduce $m^\star$ as a new shared midpoint for the pair $(z^\star, w^\star)$, strictly increasing $|\mathrm{M}_\mathcal{G}(z^\star, w^\star)|$ and thereby raising the pairwise curvature contribution. By Theorem~\ref{thm:rewiring}, this can only improve or maintain $\kappa_w(z^\star)$. Moreover, each such step tightens the generalization bound (Theorem~\ref{thm:robustness}) when the graph-level curvature lower bound increases.

\paragraph{Why target the least-supported two-hop connection?}
The choice $w^\star = \arg\min_{w \in \mathrm{S}_2(z^\star)} |\mathrm{M}_\mathcal{G}(z^\star, w)|$ selects the two-hop neighbor of the bottleneck node whose transport is most concentrated on a single fragile midpoint. Adding a new midpoint for this pair provides the largest marginal improvement in $H_{L_0}(z^\star, \mathrm{S}_2(z^\star))$ and thus the largest curvature gain, since the functional is most sensitive to changes in terms with the smallest $k_w$.

\paragraph{Complexity.}
Each iteration of Algorithm~\ref{alg:mcr} requires: (i) computing $\kappa_w$ for all nodes, $O(n \cdot k_{\mathrm{iter}} \cdot d_{\max}^2)$ (see Appendix~\ref{app:complexity}); (ii) finding $z^\star$, $O(n)$; (iii) finding $w^\star$ by scanning $\mathrm{S}_2(z^\star)$, $O(d_{\max}^2)$; (iv) finding the new midpoint candidate $m^\star$, $O(d_{\max})$. The dominant cost per iteration is therefore $O(n \cdot k_{\mathrm{iter}} \cdot d_{\max}^2)$.

%% file: Appendix/proof_Strauss.tex
\subsection{Geometric Interpretation via the Motzkin-Straus Theorem}

We leverage the Motzkin-Straus Theorem  to link local graph density to the optimization landscape of $\mathrm{H}_\Lop$. Consider a node $u$ where every 2-hop neighbor shares exactly two midpoints with $u$, as illustrated in Figure \ref{fig:positiveEntropy}.

\begin{figure}[h!]
    \centering
    \includegraphics[width=0.24\textwidth]{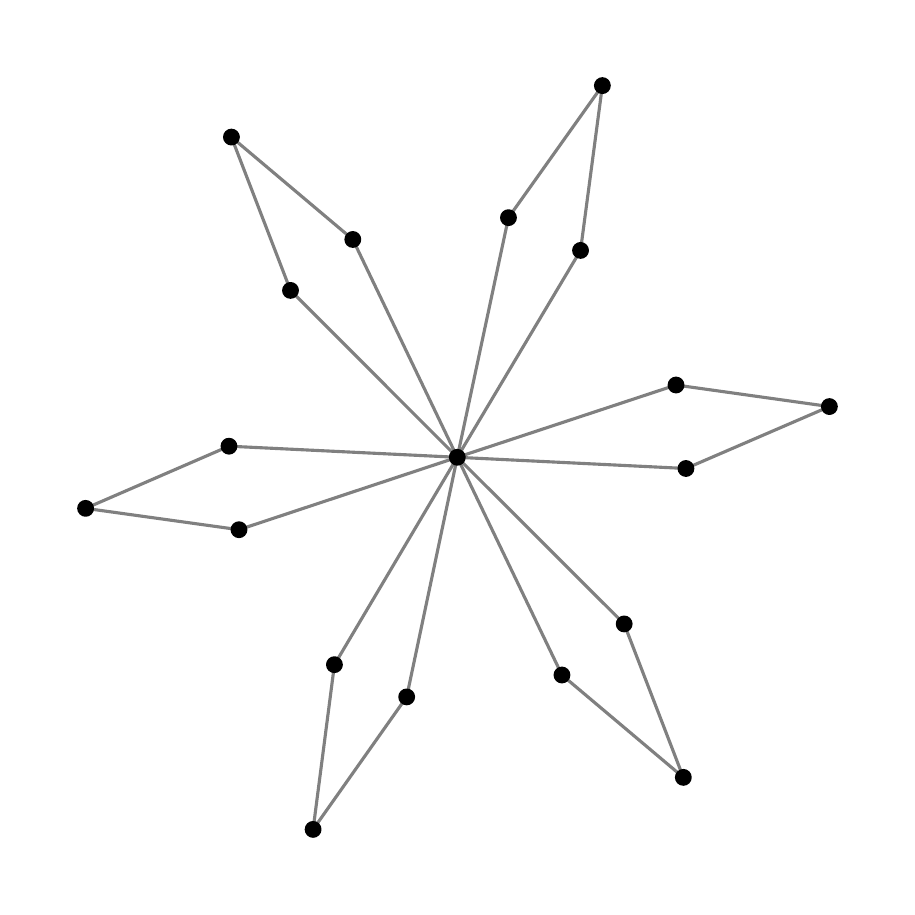}
    \caption{A positively curved motif where the entropic functional serves as a signature of local density and clustering. }
    \label{fig:positiveEntropy}
\end{figure}
\subsection{Proof of Theorem \ref{thm:Motzkin-Straus}}
In our set-up of Theorem \ref{thm:Motzkin-Straus}, we have, by definition of $\mathrm{H}_{\Lop_0}(u, \mathrm{S}_2(u))$,

$$
 \mathrm{H}_{\Lop_0}(u, \mathrm{S}_2(u)) =   2 \max_{\alpha} \left\{ \sum_{z'\in \mathrm{S}_2(u)} \prod_{z''\in ]u,z'[}\alpha(z'')   \right\}=2 \max_{\alpha} \left\{ \sum_{\{i,j\}\in \mathcal{E}_u} \alpha_i \alpha_j \right\}
$$
such that $\{z', w'\}$ is an edge of $\mathcal{E}_u$ if $\{z', w'\} = ]u, z'[$ for some $z' \in \mathrm{S}_2(u)$.
Note that the clique in the auxiliary midpoint graph $\mathcal{E}_u$ (whose vertices are midpoints in $\mathrm{S}_1(u)$) corresponds to a set of midpoints that pairwise share a common 2-hop target; its clique number $\omega(\mathcal{E}_u)$ captures the maximum local clustering in the 2-hop neighborhood of $u$, and is related to $g(\mathrm{B}_2(u))$ under the structural regularity assumption of the theorem.

Thus, Theorem \ref{thm:Motzkin-Straus} is a direct consequence of the following theorem.

\begin{theorem}[Motzkin-Straus Theorem \cite{motzkin1965maxima}]
    Let $\mathcal{G} = (\mathcal{V}, \mathcal{E})$ be a simple undirected graph with clique number $\omega(\mathcal{G})$.
Then the following relation holds:
\[
2 \max_{\alpha} \left\{ \sum_{\{i,j\}\in \mathcal{E}} \alpha_i \alpha_j \right\}
= 1 - \frac{1}{{\omega(\mathcal{G})}}.
\]
\end{theorem}

%% file: Appendix/Proof_entropic_curvature_bound.tex
\begin{proof}[Proof of Theorem \ref{thm:lower_bound}]
From Theorem 2 in \cite{rapaport2024samson}, we have $\kappa \geq \inf\limits_{S}  -2\log (\mathrm{K}(\mathrm{S}))$ 
where $\mathrm{K}$ is defined in (15) in \cite{rapaport2024samson}  and \( \mathrm{S} \subset \mathcal{X} \times \mathcal{X} \) is d-cyclically monotone i.e. For every finite collection of points \((x_1, y_1), \ldots, (x_N, y_N)\) in \( S \), the following inequality holds:  

\[
\sum_{i=1}^N d(x_i, y_i) \leq \sum_{i=1}^N d(x_i, y_{i+1}),
\]  

where \( y_{N+1} \equiv y_1 \) by convention.  

However, from (16) in \cite{rapaport2024samson}, we have 
$$ \sup_{S} \mathrm{K}(\mathrm{S}) \leq \sup_{z\in \mathcal{V}} \mathrm{H}_{\Lop}(z,\mathrm{S}_2(z)),$$
which means that $ \inf\limits_{\mathrm{S}} -2\log (\mathrm{K}(\mathrm{S})) \geq \kappa_w $. This ends the proof.
\end{proof}

%% file: Appendix/proof_thm_robustness.tex
\begin{remark}[Global Lipschitz constant and non-vacuousness]
\label{rem:lipschitz_global}
The Lipschitz assumption is global with respect to graph distance, not merely local over adjacent nodes. For a $K$-layer GCN with weight matrices $\mathbf{W}^{(1)},\ldots,\mathbf{W}^{(K)}$ and $L_\sigma$-Lipschitz activations, one obtains on bounded-degree graphs $\Lambda
\le
L_\sigma^K
\prod_{k=1}^K
\|\mathbf{W}^{(k)}\|_{\mathrm{op}}.$
Thus, $\Lambda$ is finite for fixed architectures with bounded weights, and can be controlled in practice through spectral normalization or weight regularization.
\end{remark}

First, we will use the following lemma, necessarily for proving our generalization bound. 
   \begin{lemma}\label{lemm:westrass_controled_entropy}
          The graph $(\mathcal{V},d,\mathrm{m},\Lop)$ having weak entropic curvature $\kappa_w>0$ implies a \emph{transport-entropy inequality} of the form,
   \begin{equation}\label{inequality_W_1_entr}
            \mathrm{W}_1(\nu,\m)^2 
     \;\le\; 
     \frac{2}{\kappa_w} \,\mathrm{Ent}_{\m}(\nu).
   \end{equation}
   \end{lemma}

\begin{proof}
     Indeed from \eqref{Entropic_curvature_ineuqlity} by replacing $\mu_0$ by $\m$ and $\mu_1$ by $\nu$ we get
    \begin{align}
        0 \leqslant \operatorname{Ent}_\m(\nu_t) \leq & (1- t) \operatorname{Ent}_{\m}(\m) + t \operatorname{Ent}_\m(\nu) \\
        & - \frac{\kappa_w}{2} t(1- t) \mathrm{W}_1(\m, \nu)^2.
    \end{align}
    Since $\operatorname{Ent}_\m(\m)=0$ and by simplifying by $t$ we get $ (1-t)\mathrm{W}_1(\nu,\m)^2 
     \;\le\; 
     \frac{2}{\kappa_w} \,\mathrm{Ent}_{\m}(\nu).$ By taking $t=0$ we get \eqref{inequality_W_1_entr}.
\end{proof}
This inequality constrains how far, in Wasserstein distance $\mathrm{W}_1$, any distribution $\nu$ can stray without incurring large relative entropy.

Now, we will use Lemma \ref{lemm:westrass_controled_entropy} and prove Theorem \ref{thm:robustness}.

   We first express
   \[
     \Bigl|\mathcal{L}_{\mu_1}(f_\theta) 
           \;-\; \mathcal{L}_{\mu_0}(f_\theta)
     \Bigr|
     \;=\;
     \Bigl|\,
        \mathbb{E}_{x\sim \mu_1}[\ell(f_\theta(x))] 
        \;-\; 
        \mathbb{E}_{y\sim \mu_0}[\ell(f_\theta(y))]
     \Bigr|.
   \]
   By the definition of the \(1\)-Wasserstein distance, let \(\Gamma^* \in \Pi(\mu_1,\mu_0)\) be an optimal coupling between \(\mu_1\) and \(\mu_0\). Then
   \[
     \bigl|\mathcal{L}_{\mu_1}(f_\theta)
          - \mathcal{L}_{\mu_0}(f_\theta)\bigr|
     \;\le\;
     \int_{\mathcal{V} \times \mathcal{V}}
     \bigl|\ell(f_\theta(x)) - \ell(f_\theta(y))\bigr|
     \,\mathrm{d}\Gamma^*(x,y).
   \]
   By the Lipschitz property of \(\ell\), specifically \(\sigma\) - Lipschitz in its first argument, we get
   \[
     \bigl|\ell(f_\theta(x)) - \ell(f_\theta(y))\bigr|
     \;\le\; 
     \sigma \;\bigl\|\,f_\theta(x) - f_\theta(y)\bigr\|.
   \]

 Lipschitz of the GNN’s Embedding or Output 

   Assume the GNN output \(\|f_\theta(x) - f_\theta(y)\|\) is at most \(\Lambda\cdot d(x,y)\) for neighbors \(x,y\in \mathcal{V}\), or more generally in the sense of a local Lipschitz constant. Then
   \[
     \|f_\theta(x) - f_\theta(y)\|
     \;\le\; 
     \Lambda\,d(x,y).
   \]
   Hence,
   \[
     \bigl|\ell(f_\theta(x)) - \ell(f_\theta(y))\bigr|
     \;\le\;
     \sigma\,\Lambda\,d(x,y).
   \]
   Substituting into the integral and recognizing the infimum over all couplings \(\Gamma^*\),
   \[
     \bigl|\mathcal{L}_{\mu_1}(f_\theta)
           - \mathcal{L}_{\mu_0}(f_\theta)\bigr|
     \;\le\;
     \sigma \,\Lambda
     \int_{\mathcal{V} \times \mathcal{V}} 
       d(x,y)\,\mathrm{d}\Gamma^*(x,y)
     \;=\;
     \sigma \,\Lambda
     \;\mathrm{W}_1(\mu_1,\mu_0).
   \]

 Transport-Entropy Inequality from Positive Curvature 

   Because the graph satisfies \(\kappa\)-bounded entropic curvature, there is a classical \emph{transport-entropy inequality} of the form
   \[
     \mathrm{W}_1(\mu_1,\mu_0)^2
     \;\;\le\;\;
     \frac{2}{\kappa}\,\mathrm{Ent}_{\mu_0}(\mu_1).
   \]
   Therefore
   \[
     \mathrm{W}_1(\mu_1, \mu_0)
     \;\;\le\;\;
     \sqrt{
       \frac{2}{\kappa}\,\mathrm{Ent}_{\mu_0}(\mu_1)
     }.
   \]
   Combining with the Lipschitz bound above yields the final inequality:
   \[
     \Bigl|\mathcal{L}_{\mu_1}(f_\theta) 
           \;-\; 
           \mathcal{L}_{\mu_0}(f_\theta)
     \Bigr|
     \;\;\le\;\;
     \sigma \,\Lambda 
     \,\sqrt{
       \frac{2}{\kappa}\,\mathrm{Ent}_{\mu_0}(\mu_1)
     }.
   \]

Thus, the GNN’s loss can only degrade significantly if the adversarially modified distribution \(\mu_1\) is far from \(\mu_0\) in Wasserstein distance, a distance that is itself controlled by the entropic curvature \(\kappa\).

%% file: Appendix/proof_gen_negative.tex
\begin{proof}
The proof proceeds in two main phases: first, bounding the generalization gap by the optimal transport distance $\mathrm{W}_1(\mu_{\text{test}}, \mu_{\text{train}})$, and second, establishing a Transport-Entropy inequality derived from the negative curvature and diameter constraints using the Benamou-Brenier formalism \cite{benamou2000computational}.
 We denote the expected generalization gap as,
\begin{equation*}
\Delta \mathcal{L} = |\mathcal{L}_{\mu_{\text{test}}}(f_\theta) - \mathcal{L}_{\mu_{\text{train}}}(f_\theta) |.
\end{equation*}
Let $\Gamma^* \in \Pi(\mu_{\text{test}}, \mu_{\text{train}})$ be the optimal transport coupling between the two distributions. We rewrite the difference as an integral over the joint distribution:
\begin{equation*}
\Delta \mathcal{L} \le \int_{\mathcal{V} \times \mathcal{V}} |\mathcal{J}(f_\theta(u), y_u) - \mathcal{J}(f_\theta(v), y_v)| \, \mathrm{d}\Gamma^*(u, v).
\end{equation*}
Given that the loss is $\mathrm{L}$-Lipschitz and the GNN representation is $\Lambda$-Lipschitz with respect to the graph distance $d$, the composition is $(L \cdot \Lambda)$-Lipschitz. Applying this yields,
\begin{equation}
\Delta \mathcal{L} \le \mathrm{L} \cdot \Lambda \int_{\mathcal{V} \times \mathcal{V}} d(u, v) \, \mathrm{d}\Gamma^*(u, v) = \mathrm{L} \cdot \Lambda \cdot \mathrm{W}_1(\mu_{\text{test}}, \mu_{\text{train}}).
\label{eq:gen_to_w1}
\end{equation}

To bound $\mathrm{W}_1(\mu_{\text{test}}, \mu_{\text{train}})$, we analyze the entropy along a geodesic. Let $(\mu_t)_{t \in [0,1]}$ be a constant-speed optimal transport geodesic connecting $\mu_0 = \mu_{\text{train}}$ to $\mu_1 = \mu_{\text{test}}$. Let $u(t) = Ent_{\mu_{\text{train}}}(\mu_t)$. 

By the assumption of weak entropic curvature $\kappa_w \ge -\mathrm{K} $, the entropy functional satisfies the displacement convexity inequality,
\begin{equation*}
u(t) \le (1-t)u(0) + t u(1) + \frac{\mathrm{K} }{2}t(1-t)\mathrm{W}_1(\mu_{\text{train}}, \mu_{\text{test}})^2.
\end{equation*}
This integral inequality implies a lower bound on the second derivative (the curvature "deficit"), 
\begin{equation}
u''(t) \ge -\mathrm{K} \cdot \mathrm{W}_1(\mu_{\text{train}}, \mu_{\text{test}})^2.
\label{eq:second_deriv}
\end{equation}

Next, we relate the entropy at the endpoints using a Taylor expansion with integral remainder centered at $t=1$,
\begin{equation*}
u(0) = u(1) - u'(1) + \int_0^1 s \cdot u''(s) \, ds.
\end{equation*}
Substituting the curvature bound \eqref{eq:second_deriv} into the integral gives,
\begin{equation*}
u(0) \ge u(1) - u'(1) - \mathrm{K} \cdot \mathrm{W}_1^2 \int_0^1 s \, ds = u(1) - u'(1) - \frac{K}{2}\mathrm{W}_1^2.
\end{equation*}
Since $u(0) = \mathrm{Ent}_{\mu_{\text{train}}}(\mu_{\text{train}}) = 0$ and $u(1) = \mathrm{Ent}_{\mu_{\text{train}}}(\mu_{\text{test}})$, rearranging yields,
\begin{equation}
\mathrm{Ent}_{\mu_{\text{train}}}(\mu_{\text{test}}) \le u'(1) + \frac{\mathrm{K} }{2}\mathrm{W}_1^2.
\label{eq:hwi_intermediate}
\end{equation}

To bound the derivative $u'(1)$, we employ the Benamou-Brenier continuity equation: $\frac{\partial \rho_t}{\partial t} + \nabla \cdot (\rho_t v_t) = 0$, where $\rho_t$ is the density and $v_t$ is the optimal velocity field. Differentiating the entropy:
\begin{align*}
u'(t) &= \int_{\mathcal{V}} (1 + \log \rho_t) \frac{\partial \rho_t}{\partial t} \, d\mu_{\text{train}} = - \int_{\mathcal{V}} (1 + \log \rho_t) \nabla \cdot (\rho_t v_t) \, d\mu_{\text{train}} \\
&= \int_{\mathcal{V}} \nabla(\log \rho_t) \cdot (\rho_t v_t) \, d\mu_{\text{train}} = \int_{\mathcal{V}} \langle \nabla \log \rho_t, v_t \rangle \, d\mu_t.
\end{align*}
Evaluating at $t=1$ and applying the Cauchy-Schwarz inequality:
\begin{equation*}
u'(1) \le \sqrt{\int_{\mathcal{V}} |\nabla \log \rho_1|^2 \, d\mu_{\text{test}}} \cdot \sqrt{\int_{\mathcal{V}} |v_1|^2 \, d\mu_{\text{test}}} = \sqrt{I_{\mu_{\text{train}}}(\mu_{\text{test}})} \cdot W_2(\mu_{\text{train}}, \mu_{\text{test}}),
\end{equation*}
where $I$ is the Fisher Information. Since $W_1 \le W_2$, we substitute this back into \eqref{eq:hwi_intermediate} to obtain the HWI inequality:
\begin{equation*}
\mathrm{Ent}_{\mu_{\text{train}}}(\mu_{\text{test}}) \le \sqrt{I_{\mu_{\text{train}}}(\mu_{\text{test}})} \cdot \mathrm{W}_1 + \frac{\mathrm{K}}{2}\mathrm{W}_1^2.
\end{equation*}

Because negative curvature allows distances to grow indefinitely, we impose the graph diameter constraint $\mathrm{W}_1 \le D$. Standard metric-measure theory establishes that a space with curvature bounded by $-\mathrm{K}$ and diameter $D$ satisfies a Logarithmic Sobolev Inequality (LSI) with constant $\lambda \approx \frac{\mathrm{K}}{2(e^{\mathrm{K} \mathrm{D}^2 / 8} - 1)}$ such that $\mathrm{Ent} \le \frac{1}{2\lambda} I$. 

By the Otto-Villani theorem, satisfying LSI with constant $\lambda$ implies the Transport-Entropy inequality $W_1^2 \le \frac{2}{\lambda} Ent$. Substituting our explicit $\lambda$:
\begin{equation}
\mathrm{W}_1(\mu_{\text{test}}, \mu_{\text{train}})^2 \le \frac{4}{\mathrm{K}}\left(e^{\frac{\mathrm{K} \mathrm{D}^2}{8}} - 1\right) \mathrm{Ent}_{\mu_{\text{train}}}(\mu_{\text{test}}).
\label{eq:final_w1_bound}
\end{equation}

Finally, substituting the Wasserstein bound \eqref{eq:final_w1_bound} into our generalization gap inequality \eqref{eq:gen_to_w1} yields the desired result:
\begin{equation*}
|\mathcal{L}_{\mu_{\text{test}}}(f_\theta) - \mathcal{L}_{\mu_{\text{train}}}(f_\theta)| \le \mathrm{L} \cdot \Lambda \sqrt{\frac{4}{\mathrm{K}}\left(e^{\frac{\mathrm{K} \mathrm{D}^2}{8}} - 1\right) \mathrm{Ent}_{\mu_{\text{train}}}(\mu_{\text{test}})}.
\end{equation*}
This establishes that under negative curvature, adversarial robustness degrades exponentially with the graph diameter, highlighting a fundamental trade-off between mitigating oversmoothing and maintaining model stability.
\end{proof}

%% file: Appendix/dataset_stat.tex




Our implementation is available in the supplementary materials (and will be publicly available afterwards). It is built using the open-source library \textit{PyTorch Geometric} (PyG) under the MIT license \citep{Fey/Lenssen/2019}.  Note that we additionally utilized the PyTorch DeepRobust package\footnote{https://github.com/DSE-MSU/DeepRobust} to implement the adversarial attacks used in this study. The curvature pre-computation experiments were run on NVIDIA A100 and RTX A6000 GPUs; the GNN training experiments were executed on NVIDIA RTX 3090 GPUs.

\subsection{Benchmark Datasets}
We evaluate our framework on six widely recognized graph benchmarks that represent distinct structural and homophilic regimes \cite{yang2016revisiting, pei2020geom}. These are categorized as follows:

\begin{itemize}
    \item \textbf{Homophilic Datasets:} \textit{Cora}, \textit{CiteSeer}, and \textit{PubMed} are citation networks where nodes represent documents and edges represent citations. In these graphs, nodes with the same label (research topic) are highly likely to be connected, a property that generally corresponds to more regular local geometries.
    \item \textbf{Heterophilic Datasets:} \textit{Cornell}, \textit{Texas}, and \textit{Wisconsin} are webpage networks from computer science departments. These graphs exhibit high heterophily, meaning connected nodes often belong to different classes. Such topologies typically contain more structural bottlenecks and negative curvature regions, providing a rigorous test for our E-Gate mechanism.
\end{itemize}

The statistics for these datasets are summarized in Table~\ref{tab:dataset_stats}.

\begin{table}[h]
\centering
\caption{Dataset Statistics}
\label{tab:dataset_stats}
\begin{tabular}{@{}lccccc@{}}
\toprule
\textbf{Dataset} & \textbf{Nodes} & \textbf{Edges} & \textbf{Features} & \textbf{Classes} & \textbf{Homophily} \\ \midrule
Cora             & 2,708          & 5,429          & 1,433             & 7                & 0.81               \\
CiteSeer         & 3,327          & 4,732          & 3,703             & 6                & 0.74               \\
PubMed           & 19,717         & 44,338         & 500               & 3                & 0.80               \\
Cornell          & 183            & 295            & 1,703             & 5                & 0.30               \\
Texas            & 183            & 309            & 1,703             & 5                & 0.11               \\
Wisconsin        & 251            & 499            & 1,703             & 5                & 0.21               \\ \bottomrule
\end{tabular}
\end{table}

\subsection{Experimental Protocol and Data Splits}
For all datasets, we adopt a random splitting strategy to evaluate the model's generalization under distributional variance. For each of the 10 independent runs, we randomly sample $60\%$ of the nodes for training, $20\%$ for validation (to monitor convergence), and $20\%$ for testing. This protocol ensures that the reported mean accuracy and standard deviation reflect the model's robustness to specific topological samplings.

\subsection{Implementation Details}
We utilize a two-layer GNN, e.g. GCN or GIN etc.,  backbone with 64 hidden units. The model is trained for 200 epochs using the Adam optimizer with a learning rate of $0.01$ and weight decay of $5 \times 10^{-4}$. For the E-Gate mechanism, the temperature parameter $\tau$ is initialized at $1.0$ and learned during training via backpropagation. To ensure statistical robustness, we report the mean and standard deviation across 10 independent runs using distinct random splits for each dataset. All experiments were implemented using PyTorch Geometric and executed on NVIDIA RTX 3090 GPUs.

%% file: Appendix/proof_expansions.tex
The proof of this theorem adapts techniques from the work in \cite{salez2022}.

\subsection{Random Rooted Graphs}

All graphs considered here are simple, undirected, countable, and locally finite. A \emph{rooted graph} is a pair \((\mathcal{G}, o)\), where $\mathcal{G}$ is a graph and \(o\) is a distinguished vertex called the \emph{root}. Two rooted graphs \((\mathcal{G}, o)\) and \((\mathcal{G}', o')\) are \emph{isomorphic}, denoted \(\mathcal{G} \simeq \mathcal{G}'\), if there exists a bijection \(\phi \colon \mathcal{V}_\mathcal{G} \to \mathcal{V}_{\mathcal{G}'}\) that preserves the root (\(\phi(o) = o'\)) and the edges:
\[
\forall x, y \in \mathcal{V}_{\mathcal{G}}, \quad \{x, y\} \in \mathcal{E}_\mathcal{G} \iff \{\phi(x), \phi(y)\} \in \mathcal{E}_{\mathcal{G}'}.
\]
Let \(\mathcal{G}_{\bullet}\) denote the set of connected rooted graphs, considered up to the isomorphism relation \(\simeq\). For brevity, we use \((\mathcal{G}, o)\) to denote both the rooted graph and its equivalence class. We write \(B_t(\mathcal{G}, o)\) for the \emph{ball of radius \(t\) around the root} in $\mathcal{G}$, i.e., the finite rooted subgraph induced by the vertices \(\{x \in \mathcal{V}_\mathcal{G} \colon d_\mathcal{G}(o, x) \leq t\}\). 

The space \(\mathcal{G}_{\bullet}\) is equipped with the \emph{local metric} \(d_{\text{LOG}} \colon \mathcal{G}_{\bullet} \times \mathcal{G}_{\bullet} \to [0, 1]\), defined as:
\[
d_{\text{LOG}}\big((\mathcal{G}, o), (\mathcal{G}', o')\big) := \frac{1}{1 + r}, \quad \text{where} \quad r = \sup\{t \geq 0 \colon B_t(\mathcal{G}, o) \simeq B_t(\mathcal{G}', o')\}.
\]
Intuitively, two rooted graphs are \emph{close} if their local structures agree up to a large radius \(r\). The metric space \((\mathcal{G}_{\bullet}, d_{\text{LOG}})\) is complete and separable. We equip it with its Borel \(\sigma\)-algebra, and any \(\mathcal{G}_{\bullet}\)-valued random variable is called a \emph{random rooted graph}.

Let \(\mathcal{P}(\mathcal{G}_{\bullet})\) denote the space of Borel probability measures on \(\mathcal{G}_{\bullet}\), equipped with the topology of weak convergence. For a finite graph \(\mathcal{G}\), its \emph{local profile} \(\mathcal{S}_{\mathcal{G}} \in \mathcal{P}(\mathcal{G}_{\bullet})\) is the empirical distribution of all possible rootings of \(\mathcal{G}\):
\[
\mathcal{S}_{\mathcal{G}} := \frac{1}{|\mathcal{V}_{\mathcal{G}}|} \sum_{x \in \mathcal{V}_{\mathcal{G}}} \delta_{(\mathcal{G}, x)},
\]
where \(\delta_{(\mathcal{G}, x)}\) is the Dirac measure at \((\mathcal{G}, x)\).

\subsection{Local Weak Limit of Graph Sequences}
For a finite graph $\mathcal{G}$, the pair \((\mathcal{G}, x)\) is implicitly restricted to the connected component of \(x\) if $\mathcal{G}$ is not connected. Given a sequence of finite graphs \(\mathcal{G}_n = (\mathcal{V}_n, \mathcal{E}_n)\) for \(n \geq 1\), if their local profiles \((\mathcal{S}_{\mathcal{G}_n})_{n \geq 1}\) converge to a limit \(\mathcal{S}\) in \(\mathcal{P}(\mathcal{G}_{\bullet})\), we call \(\mathcal{S}\) the \emph{local weak limit} of the sequence \((\mathcal{G}_n)_{n \geq 1}\) and denote this convergence by:
\[
\mathcal{G}_n \xrightarrow[n \to \infty]{} \mathcal{S}.
\]
Intuitively, \(\mathcal{S}\) represents the law of a random rooted graph that captures the asymptotic local structure of \(\mathcal{G}_n\) when viewed from a uniformly random root. Formally, for any continuous and bounded observable \(f \colon \mathcal{G}_{\bullet} \to \mathbb{R}\), the following holds:
\[
\frac{1}{|\mathcal{V}_n|} \sum_{x \in \mathcal{V}_n} f(\mathcal{G}_n, x) \xrightarrow[n \to \infty]{} \mathcal{S}[f(\mathcal{G}, o)] .
\]
The left-hand side is a spatial average of local contributions from the vertices of \(\mathcal{G}_n\). Local weak convergence thus simplifies asymptotic analysis by replacing such averages with expectations over a limiting random rooted graph.

\textbf{Local Observables}

The class of continuous functions on $\mathcal{G}_{\bullet}$ includes, but is not limited to, all $t$-local observables ($t \geq 0$). Here, a function $f \colon \mathcal{G}_{\bullet} \to \mathbb{R}$ is called \emph{$t$-local} if its value $f(\mathcal{G}, o)$ depends only on the isomorphism class of the finite ball $B_t(\mathcal{G}, o)$. Below is a list of key examples that will be used throughout this paper without further mention:

\begin{itemize}
    \item The \emph{root degree} $(\mathcal{G}, o) \mapsto \deg_{\mathcal{G}}(o)$ is $1$-local.
    
    \item The \emph{minimum curvature at $o$}, $(\mathcal{G}, o) \mapsto  \kappa_w(\mathcal{G})(o)$, is $2$-local.
\end{itemize}

\subsection{Tightness, Unimodularity, and Stationarity}

\subsubsection{Tightness of Graph Sequences}
One of the key advantages of the local weak convergence framework is that every "reasonable" sequence of sparse graphs admits a local weak limit. The following tightness criterion by Benjamini, Lyons, and Schramm makes this precise (note that passing to subsequences may be necessary):

\begin{theorem}\label{Theorem_control_of_degree}
Let \( \mathcal{G}_n = (\mathcal{V}_n, \mathcal{E}_n) \), \( n \geq 1 \), be finite graphs such that
\[
\sup_{n \geq 1} \left\{ \frac{1}{|\mathcal{V}_n|} \sum_{x \in \mathcal{V}_n} \phi(\deg_{\mathcal{G}_n}(x)) \right\} < \infty,
\]
where \(\phi \colon \mathbb{Z}_+ \to \mathbb{R}_+\) satisfies \(\phi(d) \gg d\) as \(d \to \infty\). Then, \((\mathcal{G}_n)_{n \geq 1}\) has a subsequence that admits a local weak limit.
\end{theorem}
\begin{proof}
    See Theorem 3.1 in \cite{itai2015}.
\end{proof}

In particular, this criterion applies to the sequence \((\mathcal{G}_n)_{n \geq 1}\) in Theorem 3 with \(\phi(d) = d \log d\). This ensures we can study the existence of non-negatively curved expanders directly at the level of local weak limits.

\subsubsection{Unimodularity}
Local weak limits of finite graphs inherit a powerful distributional invariance called \emph{unimodularity}, stemming from the uniform randomness of the root in the local profile. Formally, a measure \(\mathcal{S} \in \mathcal{P}(\mathcal{G}_{\bullet})\) is \emph{unimodular} if it satisfies the \emph{Mass Transport Principle}:
\[
\mathcal{S} \left[ \sum_{x \in \mathcal{V}_{\mathcal{G}}} f(\mathcal{G}, o, x) \right] = \mathcal{S} \left[ \sum_{x \in \mathcal{V}_{\mathcal{G}}} f(\mathcal{G}, x, o) \right],
\]
for every Borel function \(f \colon \mathcal{G}_{\bullet \bullet} \to [0, \infty]\), where \(\mathcal{G}_{\bullet \bullet}\) is the space of graphs with two distinguished roots (instead of one). 

\begin{itemize}
    \item \textbf{Interpretation}: If \(f(\mathcal{G}, o, x)\) represents "mass" sent from \(o\) to \(x\), (6) states that the expected mass sent and received by the root \(o\) under \(\mathcal{S}\) must balance.
    \item \textbf{Preservation}: This principle holds for the local profile of any finite graph and is preserved under weak convergence. Thus, all local weak limits of finite graphs are unimodular.
\end{itemize}

\subsection{Properties of Local Weak Limits}

\subsubsection{Unimodularity of Limits}
\begin{theorem}[Inherited Unimodularity]\label{[inherited_unimodularity}
All local weak limits of finite graphs are unimodular.
\end{theorem}

\subsubsection{Consequences of Unimodularity}
\begin{lemma}[Everything Shows at the Root, Lemma 2.3 in \cite{AldousLyons}]\label{Consequences_of_Unimodularity}
Let $\mathcal{S} \in \mathcal{P}(\mathcal{G}_{\bullet})$ be unimodular, and $B \subseteq \mathcal{G}_{\bullet}$ a Borel set with $\mathcal{S}(B) = 1$. Then:
\[
\mathcal{S}\big(\{(\mathcal{G},o) : \forall x \in \mathcal{V}_{\mathcal{G}}, (\mathcal{G},x) \in B\}\big) = 1.
\]
\end{lemma}

\begin{proof}
Apply the Mass Transport Principle with $f(\mathcal{G},o,x) = \mathbf{1}_{(\mathcal{G},o)\notin B}$.
\end{proof}

\subsubsection{Stationarity}
\begin{definition}[Stationary Measure]
A law $\mathcal{S} \in \mathcal{P}(\mathcal{G}_{\bullet})$ is \emph{stationary} if it's invariant under the Markov chain that moves the root according to $P_{\mathcal{G}}$:
\[
\mathcal{S}\left[\sum_{x\in \mathcal{V}_{\mathcal{G}}} P_{\mathcal{G}}^t(o,x)h(\mathcal{G},x)\right] = \mathcal{S}[h(\mathcal{G},o)], \quad \forall t \geq 0,
\]
for all Borel $h: \mathcal{G}_{\bullet} \to [0,\infty]$.
\end{definition}

\begin{lemma}[Degree-Biasing Transformation]
For any unimodular $\mathcal{S} \in \mathcal{P}(\mathcal{G}_{\bullet})$ with $\deg(\mathcal{S}) := \mathcal{S}[\deg_{\mathcal{G}}(o)] < \infty$, the measure:
\[
d\hat{\mathcal{S}}(\mathcal{G},o) := \frac{\deg_{\mathcal{G}}(o)}{\deg(\mathcal{S})} d\mathcal{S}(\mathcal{G},o)
\]
is stationary.
\end{lemma}

\begin{proof}
Apply the Mass Transport Principle to $\mathcal{S}$ with $f(\mathcal{G},o,x) = h(\mathcal{G},o)\mathbf{1}_{\{x,o\}\in \mathcal{E}_{\mathcal{G}}}$.
\end{proof}

\subsection{Spectral Radius, Entropy, and the Liouville Property}

\subsubsection{Mutual Absolute Continuity}
\begin{remark}[Mutual Absolute Continuity]\label{Mutual_Absolute_Continuity}
From the degree-biasing transformation (8), the original law $\mathcal{S}$ and its biased version $\widehat{\mathcal{S}}$ satisfy:
\[
\mathcal{S}(B) = 1 \iff \widehat{\mathcal{S}}(B) = 1,
\]
for any Borel set $B \subseteq \mathcal{G}_{\bullet}$. This equivalence allows transferring results between the two measures.
\end{remark}

\subsubsection{Asymptotic Entropy}
\begin{lemma}[Random Walk Entropy]
Let $\mathcal{S} \in \mathcal{P}(\mathcal{G}_{\bullet})$ be stationary with $\mathcal{S}[\log\deg_{\mathcal{G}}(o)] < \infty$. Then the limit
\[
\mathcal{H}(\mathcal{G}) := \lim_{t\to\infty} \frac{1}{t}\sum_{x\in \mathcal{V}_{\mathcal{G}}} P_{\mathcal{G}}^t(o,x)\log\frac{1}{P_{\mathcal{G}}^t(o,x)}
\]
exists $\mathcal{S}$-almost surely and in $L^1(\mathcal{G}_{\bullet},\mathcal{S})$, and is root-independent.
\end{lemma}

\begin{proof}
    See Lemma 8 in \cite{salez2022}.
\end{proof}

\subsection{The Liouville Property}

\begin{definition}
A function $f:\mathcal{V}_{\mathcal{G}} \to \mathbb{R}$ is \emph{harmonic} if:
\[
P_{\mathcal{G}} f = f \quad \text{where} \quad (P_{\mathcal{G}} f)(x) := \sum_{y\in \mathcal{V}_{\mathcal{G}}} P_{\mathcal{G}}(x,y)f(y)
\]
\end{definition}

\begin{definition}[Harmonic Functions]
A function $f \colon \mathcal{V}_{\mathcal{G}} \to \mathbb{R}$ is \emph{harmonic} if it satisfies $P_{\mathcal{G}} f = f$, where $(P_{\mathcal{G}} f)(x) := \sum_{y \in \mathcal{V}_{\mathcal{G}}} P_{\mathcal{G}}(x,y)f(y)$. 
\begin{itemize}
    \item \textbf{Trivial case}: All constant functions are automatically harmonic.
    \item \textbf{Non-trivial case}: The existence of non-constant bounded harmonic functions reflects rich geometric structure.
\end{itemize}
\end{definition}

\begin{definition}[Liouville Property \citep{colding2019liouville}]
A graph $G$ has the \emph{Liouville property} if every bounded harmonic function on $G$ is constant.
\end{definition}

\begin{remark}[Entropic Characterization]
For stationary random graphs, this functional-analytic property admits a profound probabilistic characterization through asymptotic entropy $\mathcal{H}(\mathcal{G})$. Specifically:
\begin{itemize}
    \item The Liouville property corresponds to vanishing entropy ($\mathcal{H}(\mathcal{G}) = 0$)
    \item Non-zero entropy detects the existence of non-constant bounded harmonic functions
\end{itemize}
This bridges geometric, analytic, and probabilistic aspects of random graphs.
\end{remark}

\begin{theorem}[Entropic Characterization of Liouville Property]\label{thm:liouville-entropy}
For any stationary law $\mathcal{S} \in \mathcal{P}(\mathcal{G}_{\bullet})$ with $\mathcal{S}[\log\deg_{\mathcal{G}}(o)] < \infty$, the following equivalence holds $\mathcal{S}$-almost surely:
\[
\mathcal{H}(\mathcal{G}) = 0 \quad \iff \quad \mathcal{G} \text{ has the Liouville property}.
\]
\end{theorem}
\begin{proof}
    See theorem 9 in \cite{salez2022}.
\end{proof}

\subsection{Spectral Radius and Its Connections}

\subsubsection{Definition and Basic Properties}
The entropy $\mathcal{H}(\mathcal{G})$ relates to several fundamental graph quantities including \emph{speed}, \emph{growth}, and \emph{spectral radius}. 

For any rooted graph $(\mathcal{G},o) \in \mathcal{G}_{\bullet}$ and $t,s \geq 0$, the inequality
\[
P_{\mathcal{G}}^{t+s}(o,o) \geq P_{\mathcal{G}}^t(o,o)P_{\mathcal{G}}^s(o,o)
\]
holds. By Fekete's lemma \citep{fekete1923verteilung}, the limit
\[
\rho(\mathcal{G}) := \lim_{t\to\infty} \left(P_{\mathcal{G}}^t(o,o)\right)^{1/t} \tag{13}
\]
exists in $(0,1]$. Furthermore, graph connectivity and the inequality
\[
P_{\mathcal{G}}^{t+2s}(o,o) \geq P_{\mathcal{G}}^s(o,x)P_{\mathcal{G}}^t(x,x)P_{\mathcal{G}}^s(x,o)
\]
together imply that $\rho(\mathcal{G})$ is independent of the root choice $o$.

\subsubsection{Relation to Entropy}
\begin{lemma}[Spectral Radius vs Entropy]\label{lem:spec-rad-entropy}
For any stationary law $\mathcal{S}$ with $\mathcal{S}[\log\deg_{\mathcal{G}}(o)] < \infty$, the inequality
\[
\mathcal{H}(\mathcal{G}) \geq 2\log\frac{1}{\rho(\mathcal{G})}
\]
holds $\mathcal{S}$-almost surely.
\end{lemma}
\begin{proof}
    See lemma 10 in \cite{salez2022}.
\end{proof}

\subsection{Proof of Theorem \ref{thm:expansion}}

Let $a>0$ and $\mathcal{G}_n:= (\mathcal{V}_n,\mathcal{E}_n), n\geq 1$ be a sequence of finite graphs such that 
\begin{equation}\label{assumption_1}
    \sup_{n \geq 1} \left\{ \frac{1}{|\mathcal{V}_n|} \sum_{x \in \mathcal{V}_n} \deg_{\mathcal{G}_n}(x) \log \deg_{\mathcal{G}_n}(x) \right\} < \infty
\end{equation}
and 
\begin{equation}\label{assumption_2}
    \forall \varepsilon > 0, \quad \frac{1}{|\mathcal{V}_n|} \# \left\{ x \in \mathcal{V}_n : \kappa_{w}(x) < a \right\} \xrightarrow[n \to \infty]{} 0. 
\end{equation}
The goal is to prove that
\begin{equation}\label{assumption_3}
    \forall \rho < 1, \quad \liminf_{n \to \infty} \left( \frac{1}{|\mathcal{V}_n|} \# \left\{ i : \lambda_i(\mathcal{G}_n) \geq \rho \right\} \right) > 0.
\end{equation}

In view of \eqref{assumption_1} and Theorem \ref{Theorem_control_of_degree}, we can, after extracting a subsequence if necessary, assume that
\begin{equation}\label{convergence_of_G_n}
    \mathcal{G}_n \xrightarrow{n \to \infty} \mathcal{S}
\end{equation}
For some $ \mathcal{S} \in P(\mathcal{G}_{\bullet})$, note that $\mathcal{S}$ is automatically unimodular by Theorem \ref{[inherited_unimodularity}, and satisfies
\begin{equation}\label{unimod_resulalt}
    \mathcal{S} \left[ \deg_{\mathcal{G}}(o) \log \deg_{\mathcal{G}}(o) \right] < \infty.
\end{equation}
Since curvature is a local property like the degree, it also ``passes to the limit,'' meaning
\begin{equation}\label{entropic_curvature_dominate}
    \mathcal{S} \left( \kappa_w(\mathcal{G}) \geq a \right) = 1.
\end{equation}
\begin{proof}
    As previously noted, the observable $f \colon (\mathcal{G}, o) \mapsto  \kappa_w(o)$ is 2-local, and therefore continuous on $\mathcal{G}_{\bullet}$. Applying the Portmanteau Theorem yields that 

\begin{align*}
\mathcal{S}(f < a) &\leq \liminf_{n \to \infty} \mathcal{S}_{\mathcal{G}_n}(f < a) \\
&= \liminf_{n \to \infty} \left\{ \frac{1}{|\mathcal{V}_n|} \#\{o \in \mathcal{V}_n : f(\mathcal{G}_n, o) < a\} \right\}=0.
\end{align*}
This implies $\mathcal{S}(f < a) = 0$ by \eqref{assumption_2}. By applying Lemma~\ref{Consequences_of_Unimodularity} to the event $B = \{f \geq a\}$ we get The desired result.
\end{proof}

The first key step in our proof is to deduce from \eqref{entropic_curvature_dominate} that the entropy under $\mathcal{S}$ is zero. This is established in the following theorem.
\begin{theorem}[positive curvature implies zero-entropy]\label{positive_curvature_implies zero_entropy} Let $a>0$. Assume that $\mathcal{S} [\log \deg_{\mathcal{G}}(o)] < \infty$, the we have almost-surely under any stationary law $\mathcal{S} \in P(\mathcal{G}_{\bullet})$ 
    \[
\kappa_w(\mathcal{G}) \geq a \implies \mathcal{H}(\mathcal{G}) = 0.
\]
\end{theorem}

\begin{proof}

Let \( f: \mathcal{V} \to \mathbb{R} \) be a harmonic function i.e.,
\[
\Lop f(x) = 0 \quad \forall x \in \mathcal{V}.
\]
Assume in addition that $f$ is bounded. We want that \( f \) must be constant.

By using Modified Logarithmic Sobolev Inequality (MLSI) (from Theorem \ref{thm:poincare-gnn}):  
   If \( \kappa_w > 0 \), then  
   For any \( f \),
   \[
   \text{Var}_\mu(f) \leq \frac{1}{\kappa_w} \int\sum_{y \in \mathcal{V}} \Lop(x,y) (f(y) - f(x))^2d\mu(x).
   \]
We have, 

$$
\begin{aligned}
    \int\sum_{y \in \mathcal{V}} \Lop(x,y) (f(y) - f(x))^2d\mu(x) & = \sum_{x,y \in \mathcal{V}} \Lop(x,y) (f(y) - f(x))^2\mu(x)
    \\ & = \sum_{x,y \in \mathcal{V}} \Lop(x,y) f(y)^2\mu(x) - 2 \sum_{x,y \in \mathcal{V}} \Lop(x,y) f(y)f(x)\mu(x) 
    \\ & + \sum_{x,y \in \mathcal{V}} \Lop(x,y)  f(x)^2\mu(x)
\end{aligned}
$$

Since \(\sum_y \Lop(x,y) = 0\) (rows sum to zero, see \ref{Proporties_lm}), we have:
   \[
   \sum_{x,y} f(x)^2 \Lop(x,y) \mu(x) = \sum_x f(x)^2 \left( \sum_y \Lop(x,y) \right) \mu(x) = 0.
   \]
For the last term, since $\Lop(x,y)\mu(x)= L(y,x)\mu(y)$ (see \ref{Proporties_lm}), we have 
$$
\sum_{x,y} f(y)^2 \Lop(x,y) \mu(x)= \sum_{x,y} f(y)^2 L(y,x) \mu(y)=  \sum_y f(y)^2 \left( \sum_x L(y,x) \right) \mu(y) = 0.
$$
Thus 
$$
\begin{aligned}
    \int\sum_{y \in \mathcal{V}} \Lop(x,y) (f(y) - f(x))^2d\mu(x) &= - 2 \sum_{x,y \in \mathcal{V}} \Lop(x,y) f(y)f(x)\mu(x) 
    \\ & = - 2 \sum_{x \in \mathcal{V}} f(x)\bigg(\sum_{y\in \mathcal{V}} \Lop(x,y) (f(y)-f(x))\bigg)\mu(x).
\end{aligned}
$$

Recall that  $f$ is harmonic function, thus $\sum_{y\in \mathcal{V}} \Lop(x,y) (f(y)-f(x))=0$. Then $\int\sum_{y \in \mathcal{V}} \Lop(x,y) (f(y) - f(x))^2d\mu(x)=0$. From this, we deduce 
$$
\text{Var}_\mu(f)=0, 
$$
This implies that $f$ is a constant. Thus $\kappa_w(\mathcal{G})$ implies that $G$ satisfies the Liouville properties and using Theorem \ref{thm:liouville-entropy} we get the zero-entropy property: $\mathcal{H}(\mathcal{G})=0$. 
\end{proof}

In view of Remark~\ref{Mutual_Absolute_Continuity}, this result extends to any unimodular law $\mathcal{S} \in \mathcal{P}(\mathcal{G}_{\bullet})$ satisfying 
\[
\mathcal{S}[\deg_{\mathcal{G}}(o)\log\deg_{\mathcal{G}}(o)] < \infty,
\] 
and in particular applies to the limit $\mathcal{S}$ in \eqref{convergence_of_G_n}. Combining this observation with Lemma~\ref{lem:spec-rad-entropy}, we immediately conclude that our local weak limit satisfies
\[
\mathcal{S}(\rho(\mathcal{G}) = 1) = 1.
\]

Remarkably, this simple condition suffices to establish \eqref{assumption_3}. The following theorem formalizes this claim and completes the proof of our main result.

\begin{theorem}[Zero-entropy implies poor spectral expansion]
\label{thm:zero-entropy}
Let $\{\mathcal{G}_n = (\mathcal{V}_n, \mathcal{E}_n)\}_{n \geq 1}$ be finite graphs with local weak limit $\mathcal{S}$, and suppose that $\mathcal{S}(\rho(\mathcal{G}) = 1) = 1$. Then, for any $\rho < 1$,
\[
\liminf_{n \to \infty} \left\{ \frac{1}{|\mathcal{V}_n|} \# \{i: \lambda_i(\mathcal{G}_n) > \rho\} \right\} > 0.
\]
\end{theorem}
\begin{proof}
    See Theorem 12 in \cite{salez2022}.
\end{proof}

%% file: Appendix/proof_poincare_gnn.tex
We first establish an auxiliary functional inequality that will be used
to derive the desired Poincaré-type estimate. In particular, positive
weak entropic curvature implies a modified logarithmic Sobolev inequality
for exponential graph signals. The proof of Theorem~\ref{thm:poincare-gnn}
then follows by applying this inequality to a small perturbation
$g=\varepsilon(f-\bar f)$ and passing to the limit as $\varepsilon\to 0$.

\begin{lemma}[Modified logarithmic Sobolev inequality]
\label{lem:mlsi}
Let $(V,d,m,L)$ be a finite connected graph space with weak entropic
curvature lower bound $\kappa_w>0$, and let $\mu:=m/m(V)$. Then, for every
$g:V\to\mathbb{R}$,
$$
    \operatorname{Ent}_{\mu}(e^g)
    \leq
    \frac{1}{2\kappa_w}
    \int_V
    |\nabla^+ g|^2(u)e^{g(u)}
    \,d\mu(u),
$$
where
$$
    |\nabla^+ g|^2(u)
    :=
    \sup_{v\sim u}[g(u)-g(v)]_+^2
$$
and
$$
    \operatorname{Ent}_{\mu}(\phi)
    :=
    \int_V \phi\log\phi\,d\mu
    -
    \left(\int_V \phi\,d\mu\right)
    \log\left(\int_V \phi\,d\mu\right).
$$
\end{lemma}

\begin{proof}
It is enough to prove the inequality for functions of the form
$\phi=e^g$, where $\phi:V\to(0,\infty)$. The general case follows by
approximating $\phi$ by $\phi+\delta$ and letting $\delta\downarrow0$.

Set
$$
    Z_\phi := \int_V \phi\,d\mu
$$
and define the probability measure $\nu_0$ by
$$
    d\nu_0
    :=
    \frac{\phi}{Z_\phi}\,d\mu .
$$
Let $\nu_1:=\mu$. Since $\mu=m/m(V)$, the densities with respect to $m$
and with respect to $\mu$ differ only by the constant factor $m(V)$.
Hence entropy differences with respect to $m$ and $\mu$ agree up to
additive constants, which cancel in the argument below.

By the positive weak entropic curvature assumption, and since
$\kappa_w$ is a lower bound for the global entropic curvature, the entropy
functional is $\kappa_w$-convex along admissible Wasserstein geodesics.
Thus, for any admissible coupling $\pi\in\Pi(\nu_0,\nu_1)$, there exists an
interpolating curve $(\nu_t)_{t\in[0,1]}$ from $\nu_0$ to $\nu_1$ such that
$$
    \operatorname{Ent}_{m}(\nu_t)
    \leq
    (1-t)\operatorname{Ent}_{m}(\nu_0)
    +
    t\operatorname{Ent}_{m}(\nu_1)
    -
    \frac{\kappa_w}{2}t(1-t)W_1(\nu_0,\nu_1)^2 .
$$
Rearranging gives
$$
    \operatorname{Ent}_{m}(\nu_0)
    \leq
    -
    \frac{\operatorname{Ent}_{m}(\nu_t)
    -
    \operatorname{Ent}_{m}(\nu_0)}{t}
    +
    \operatorname{Ent}_{m}(\nu_1)
    -
    \frac{\kappa_w}{2}(1-t)W_1(\nu_0,\nu_1)^2 .
$$
We now let $t\downarrow0$.

Write $\pi(\cdot\mid u)$ for the conditional distribution of the endpoint
given the starting point $u$. The first variation of the entropy at
$t=0$ along the corresponding graph interpolation satisfies
$$
    \left.
    \frac{d}{dt}\operatorname{Ent}_{m}(\nu_t)
    \right|_{t=0}
    \geq
    -
    \sum_{u\in V}
    |\nabla^+\log\phi|(u)
    \left(
        \sum_{w\in V} d(u,w)\pi(w\mid u)
    \right)
    \nu_0(u).
$$
Here we used that adding a multiplicative constant to $\phi$ does not
change differences of $\log\phi$, and that along each shortest path the
entropy variation is controlled by the maximal one-sided decrease of
$\log\phi$ over neighbors of $u$.

By Young's inequality, for every $u\in V$,
$$
    |\nabla^+\log\phi|(u)
    \left(
        \sum_{w\in V} d(u,w)\pi(w\mid u)
    \right)
    \leq
    \frac{1}{2\kappa_w}|\nabla^+\log\phi|^2(u)
    +
    \frac{\kappa_w}{2}
    \left(
        \sum_{w\in V} d(u,w)\pi(w\mid u)
    \right)^2 .
$$
Therefore,
$$
    \left.
    \frac{d}{dt}\operatorname{Ent}_{m}(\nu_t)
    \right|_{t=0}
    \geq
    -
    \frac{1}{2\kappa_w}
    \sum_{u\in V}
    |\nabla^+\log\phi|^2(u)\nu_0(u)
    -
    \frac{\kappa_w}{2}
    \sum_{u\in V}
    \left(
        \sum_{w\in V} d(u,w)\pi(w\mid u)
    \right)^2
    \nu_0(u).
$$
The last term is bounded by the corresponding transport cost between
$\nu_0$ and $\nu_1$. Consequently, when this estimate is inserted into the
previous curvature inequality and $t\downarrow0$, the transport-cost terms
cancel. We obtain
$$
    \operatorname{Ent}_{m}(\nu_0)
    \leq
    \operatorname{Ent}_{m}(\nu_1)
    +
    \frac{1}{2\kappa_w}
    \sum_{u\in V}
    |\nabla^+\log\phi|^2(u)\nu_0(u).
$$
Since $\nu_1=\mu=m/m(V)$, the change from $\operatorname{Ent}_{m}$ to
$\operatorname{Ent}_{\mu}$ only contributes the same additive
normalization constant to both sides. Hence
$$
    \operatorname{Ent}_{\mu}(\nu_0)
    \leq
    \frac{1}{2\kappa_w}
    \sum_{u\in V}
    |\nabla^+\log\phi|^2(u)\nu_0(u).
$$
Using $d\nu_0=(\phi/Z_\phi)d\mu$ gives
$$
    \operatorname{Ent}_{\mu}(\nu_0)
    \leq
    \frac{1}{2\kappa_w Z_\phi}
    \int_V
    |\nabla^+\log\phi|^2(u)\phi(u)
    \,d\mu(u).
$$
Multiplying both sides by $Z_\phi$ and using
$$
    Z_\phi\operatorname{Ent}_{\mu}(\nu_0)
    =
    \operatorname{Ent}_{\mu}(\phi),
$$
we get
$$
    \operatorname{Ent}_{\mu}(\phi)
    \leq
    \frac{1}{2\kappa_w}
    \int_V
    |\nabla^+\log\phi|^2(u)\phi(u)
    \,d\mu(u).
$$
Finally, taking $\phi=e^g$ gives $\log\phi=g$, and therefore
$$
    \operatorname{Ent}_{\mu}(e^g)
    \leq
    \frac{1}{2\kappa_w}
    \int_V
    |\nabla^+ g|^2(u)e^{g(u)}
    \,d\mu(u).
$$
This proves the modified logarithmic Sobolev inequality.
\end{proof}
We now prove Theorem~\ref{thm:poincare-gnn}.

\begin{proof}
Let
$\bar f := \int_\mathcal{V} f \, \mathrm{d}\mu$
and set $ h := f - \bar f .$
Then
$  \int_\mathcal{V} h \, d\mu = 0$ and $$ \operatorname{Var}_{\mu}(f) =  \int_\mathcal{V} h^2 \, d\mu .$$

For a graph function $g : \mathcal{V} \to \mathbb{R}$, define the one-sided local
slope
$$ |\nabla^+ g|^2(u):=\sup_{v \sim u} [g(u)-g(v)]_+^2 ,
$$ where $[a]_+ := \max\{a,0\}$.

Since the graph space has positive weak entropic curvature
$\kappa_w>0$, the entropy functional is displacement convex with
curvature lower bound $\kappa_w$. Equivalently, for every
$g:V\to\mathbb{R}$, the following modified logarithmic Sobolev
inequality holds:
$$
    \operatorname{Ent}_{\mu}(e^g)
    \leq
    \frac{1}{2\kappa_w}
    \int_\mathcal{V} |\nabla^+ g|^2(u)e^{g(u)}\,d\mu(u),
$$
where
$$
    \operatorname{Ent}_{\mu}(\phi)
    :=
    \int_\mathcal{V} \phi \log \phi \, d\mu
    -
    \left(\int_\mathcal{V} \phi \, d\mu\right)
    \log\left(\int_\mathcal{V} \phi \, d\mu\right).
$$

Apply this inequality to $g = \varepsilon h$ with $\varepsilon>0$. We obtain
$$\operatorname{Ent}_{\mu}(e^{\varepsilon h})
    \leq
    \frac{1}{2\kappa_w}
    \int_\mathcal{V} |\nabla^+(\varepsilon h)|^2(u)
    e^{\varepsilon h(u)}\,d\mu(u).$$

We now expand both sides as $\varepsilon \to 0$. Since
$$ e^{\varepsilon h}= 1+\varepsilon h+\frac{\varepsilon^2}{2}h^2+O(\varepsilon^3),$$
and $\int_\mathcal{V} h\,d\mu=0$, we have
$$ \operatorname{Ent}_{\mu}(e^{\varepsilon h})
    =
    \frac{\varepsilon^2}{2}
    \int_\mathcal{V} h^2\,d\mu
    +
    O(\varepsilon^3).
$$
Therefore,
$$\operatorname{Ent}_{\mu}(e^{\varepsilon h})
    = \frac{\varepsilon^2}{2}
    \operatorname{Var}_{\mu}(f)  + \mathcal{O}(\varepsilon^3).$$

On the other hand,
$$|\nabla^+(\varepsilon h)|^2(u)=\varepsilon^2 |\nabla^+ h|^2(u).$$
Since $h=f-\bar f$, subtracting the constant $\bar f$ does not change
edge differences, and hence
$$ |\nabla^+ h|^2(u)   = |\nabla^+ f|^2(u).$$
Thus,
$$|\nabla^+(\varepsilon h)|^2(u)
    = \varepsilon^2 |\nabla^+ f|^2(u).$$
Moreover,
$$
    e^{\varepsilon h(u)} = 1+O(\varepsilon).
$$
Hence
$$\frac{1}{2\kappa_w}
    \int_\mathcal{V} |\nabla^+(\varepsilon h)|^2(u)
    e^{\varepsilon h(u)}\,d\mu(u)
    =
    \frac{\varepsilon^2}{2\kappa_w}
    \int_\mathcal{V} |\nabla^+ f|^2(u)\,d\mu(u)
    +
    \mathcal{O}(\varepsilon^3).$$

Combining the two expansions gives
$$\frac{\varepsilon^2}{2}
    \operatorname{Var}_{\mu}(f)
    +
    O(\varepsilon^3)
    \leq
    \frac{\varepsilon^2}{2\kappa_w}
    \int_\mathcal{V} |\nabla^+ f|^2(u)\,d\mu(u)
    +
    O(\varepsilon^3).$$
Dividing by $\varepsilon^2/2$ and letting $\varepsilon \to 0$, we get
$$\operatorname{Var}_{\mu}(f)
    \leq
    \frac{1}{\kappa_w}
    \int_\mathcal{V} |\nabla^+ f|^2(u)\,d\mu(u).$$
Finally, substituting the definition of the one-sided local slope yields
$$ \operatorname{Var}_{\mu}(f)
    \leq
    \frac{1}{\kappa_w}
    \int_\mathcal{V}
    \sup_{v\sim u}
    [f(u)-f(v)]_+^2
    \,d\mu(u).$$
This proves the desired inequality.
\end{proof}

%% file: Appendix/proof_poincare_negative.tex
\begin{proof}
The proof proceeds by a parallel $\varepsilon$-perturbation argument to 
that of Theorem~\ref{thm:poincare-gnn} (Appendix~\ref{thm:poincare-gnn}), now exploiting 
the \emph{lower} curvature bound $\kappa_w \geq -K$ to derive a reversed 
functional inequality.

Let $h := f - \bar{f}$ where $\bar{f} := \int_V f\, d\mu$, so that 
$\int_V h\, d\mu = 0$ and $\mathrm{Var}_\mu(f) = \int_V h^2\, d\mu$.
Since $\kappa_w \geq -K$, the entropy functional satisfies the 
$(-K)$-convexity inequality along admissible $\mathrm{W}_1$-geodesics: for any 
$\nu_0, \nu_1 \in \mathcal{P}(V)$ and any constant-speed geodesic 
$(\nu_t)_{t \in [0,1]}$,
\begin{equation}
    \mathrm{Ent}_m(\nu_t) \;\leq\; 
    (1-t)\,\mathrm{Ent}_m(\nu_0) + t\,\mathrm{Ent}_m(\nu_1) 
    + \frac{K}{2}\,t(1-t)\,\mathrm{W}_1(\nu_0,\nu_1)^2.
    \label{eq:neg_curv_convexity}
\end{equation}
The additional positive term $+\frac{K}{2}t(1-t)\mathrm{W}_1^2$ reflects the 
hyperbolic-like expansion: entropy along the geodesic may exceed the 
linear interpolation, unlike the positive curvature case.

Set $\nu_0 = \frac{e^{\varepsilon h}}{Z_\varepsilon}\mu$ with 
$Z_\varepsilon = \int_V e^{\varepsilon h}\,d\mu$, and $\nu_1 = \mu$.  
Rearranging~\eqref{eq:neg_curv_convexity} at $t \downarrow 0$ and using 
the same first-variation estimate as in Lemma~\ref{lem:mlsi}, we obtain the following 
\emph{reversed} modified log-Sobolev inequality: for every 
$g : V \to \mathbb{R}$,
\begin{equation}
    \mathrm{Ent}_\mu(e^g) 
    \;\geq\; 
    \frac{1}{2K}\int_V |\nabla^+ g|^2(u)\,e^{g(u)}\,d\mu(u) 
    \;-\; \frac{K}{2}\,\mathrm{W}_1(\nu_0,\nu_1)^2\,\mathrm{Ent}_\mu(e^g),
    \label{eq:reversed_mlsi}
\end{equation}
where $|\nabla^+ g|^2(u) := \sup_{v \sim u}[g(u)-g(v)]^2_+$.

To bound the $\mathrm{W}_1$ term in~\eqref{eq:reversed_mlsi}, we use the graph 
diameter. For any two probability measures on $V$,
\begin{equation}
    \mathrm{W}_1(\nu_0, \nu_1) \;\leq\; D,
\end{equation}
since $d(u,v) \leq D$ for all $u,v \in V$ by definition of diameter.  
Substituting into~\eqref{eq:reversed_mlsi} and rearranging,
\begin{equation}
    \left(1 + \frac{KD^2}{2}\right)\mathrm{Ent}_\mu(e^g) 
    \;\geq\; 
    \frac{1}{2K}\int_V |\nabla^+ g|^2(u)\,e^{g(u)}\,d\mu(u).
    \label{eq:diameter_bound}
\end{equation}

Apply~\eqref{eq:diameter_bound} to $g = \varepsilon h$ with $\varepsilon > 0$.  
Expanding both sides as $\varepsilon \to 0$, using $e^{\varepsilon h} = 
1 + \varepsilon h + \frac{\varepsilon^2}{2}h^2 + O(\varepsilon^3)$ and 
$\int_V h\,d\mu = 0$:

\emph{Left-hand side.} Since 
$\mathrm{Ent}_\mu(e^{\varepsilon h}) = \frac{\varepsilon^2}{2}\mathrm{Var}_\mu(f) 
+ O(\varepsilon^3)$, we have
\begin{equation}
    \left(1 + \frac{KD^2}{2}\right)\mathrm{Ent}_\mu(e^{\varepsilon h}) 
    = \left(1 + \frac{KD^2}{2}\right)\frac{\varepsilon^2}{2}\mathrm{Var}_\mu(f) 
    + O(\varepsilon^3).
\end{equation}

\emph{Right-hand side.} Since $|\nabla^+(\varepsilon h)|^2(u) = 
\varepsilon^2|\nabla^+ h|^2(u)$ and $e^{\varepsilon h(u)} = 1 + O(\varepsilon)$,
\begin{equation}
    \frac{1}{2K}\int_V |\nabla^+(\varepsilon h)|^2(u)\,e^{\varepsilon h(u)}\,d\mu(u) 
    = \frac{\varepsilon^2}{2K}\int_V |\nabla^+ f|^2(u)\,d\mu(u) 
    + O(\varepsilon^3),
\end{equation}
where we used $|\nabla^+ h|^2 = |\nabla^+ f|^2$ since subtracting the 
constant $\bar{f}$ does not change edge differences.

Combining, dividing by $\varepsilon^2/2$, and letting $\varepsilon \to 0$:
\begin{equation}
    \left(1 + \frac{KD^2}{2}\right)\mathrm{Var}_\mu(f) 
    \;\geq\; 
    \frac{1}{K}\int_V \sup_{v \sim u}[f(u)-f(v)]^2_+\,d\mu(u).
\end{equation}
Finally, since $1 + \frac{KD^2}{2} \leq e^{KD^2/2}$ for all $K,D \geq 0$ 
(by the inequality $1 + x \leq e^x$), dividing both sides by 
$e^{KD^2/2}$ yields the stated bound:
\begin{equation}
    \mathrm{Var}_\mu(f) 
    \;\geq\; 
    \frac{e^{-KD^2/2}}{K}
    \int_V \sup_{v \sim u}[f(u)-f(v)]^2_+\,d\mu(u).
\end{equation}
\end{proof}

%% file: Appendix/proof_rewiring_thm.tex
\begin{proof}
Under the uniform generator $\mathrm{L}_0$, set $k_w := |\mathrm{M}_{\mathcal{G}}(z,w)|$ 
and $k^+_w := |\mathrm{M}_{\mathcal{G}^+}(z,w)|$, so that $k^+_w \geq k_w$ by assumption. 
Each two-hop path $z \to m \to w$ carries unit weight, hence 
$\mathrm{L}^{(2)}_0(z,w) = k_w$ and $(\mathrm{L}^+_0)^{(2)}(z,w) = k^+_w$.
Here $\mathrm{L}^{(2)}_0(z,w)$ denotes the 2-step transition weight (the number of
shared midpoints), not the square of a single entry.
The path weights satisfy $\ell^{\mathcal{G}}(z,m,w) = 1/k_w$ and 
$\ell^{\mathcal{G}^+}(z,m,w) = 1/k^+_w$ for every midpoint $m$.

\textbf{Step 1: Pairwise terms.} By the arithmetic--geometric mean inequality, for a singleton target $\{w\}$, the supremum of the product term $\prod_{m \in ]z,w[} (\alpha(m)/\mathrm{L}_0(z,m))^{2\ell(z,m,w)}$ over $\alpha$ on the simplex is attained at the uniform distribution $\alpha(m) = 1/k_w$ over the $k_w$ midpoints, yielding:
\begin{equation}
    \mathrm{H}^{\mathcal{G}}_{\mathrm{L}_0}(z, \{w\}) 
    \;=\; 
    k_w \cdot \left(\frac{1}{k_w}\right)^{2} 
    \;=\; 
    \frac{1}{k_w}, 
    \qquad 
    \mathrm{H}^{\mathcal{G}^+}_{\mathrm{L}_0}(z, \{w\}) 
    \;=\; 
    \frac{1}{k^+_w}.
\end{equation}
Since $k^+_w \geq k_w$, we obtain 
$\mathrm{H}^{\mathcal{G}^+}_{\mathrm{L}_0}(z,\{w\}) \leq 
\mathrm{H}^{\mathcal{G}}_{\mathrm{L}_0}(z,\{w\})$ for every 
$w \in \mathrm{S}^{\mathcal{G}}_2(z)$.

\textbf{Step 2: Aggregating over $\mathrm{S}_2(z)$ via pointwise dominance.}
The full functional $\mathrm{H}_{\mathrm{L}_0}(z, \mathrm{S}_2(z))$ is a supremum over a \emph{single} shared distribution $\alpha$ on the simplex $\Delta^{|\mathrm{S}_1(z)|}$, not a sum of independent suprema. We therefore cannot directly apply the singleton bound termwise. Instead, we establish a pointwise dominance of the summands.

For any fixed $\alpha$ on the simplex, define the summand for target $w$ under $\mathcal{G}$:
\[
f_w(\alpha) \;:=\; \mathrm{L}^{(2)}_0(z,w) \prod_{m\in]z,w[}\left(\frac{\alpha(m)}{\mathrm{L}_0(z,m)}\right)^{2/k_w},
\]
and the corresponding summand under $\mathcal{G}^+$ (where the exponent changes from $2/k_w$ to $2/k^+_w \leq 2/k_w$):
\[
f^+_w(\alpha) \;:=\; (\mathrm{L}^+_0)^{(2)}(z,w) \prod_{m\in]z,w[}\left(\frac{\alpha(m)}{\mathrm{L}_0(z,m)}\right)^{2/k^+_w}.
\]
Since $\mathrm{L}_0(z,m) = 1$ for all $m \sim z$, each factor satisfies $0 < \alpha(m)/\mathrm{L}_0(z,m) = \alpha(m) \leq 1$. Because $k^+_w \geq k_w$, we have $2/k^+_w \leq 2/k_w$, and raising a value in $(0,1]$ to a smaller exponent yields a larger or equal value. Hence:
\[
\prod_{m\in]z,w[} \alpha(m)^{2/k^+_w} \;\geq\; \prod_{m\in]z,w[} \alpha(m)^{2/k_w}.
\]
However, the prefactor changes as $(\mathrm{L}^+_0)^{(2)}(z,w) = k^+_w \geq k_w = \mathrm{L}^{(2)}_0(z,w)$. Combining:
\[
f^+_w(\alpha) = k^+_w \prod_m \alpha(m)^{2/k^+_w}, \quad f_w(\alpha) = k_w \prod_m \alpha(m)^{2/k_w}.
\]
Evaluating at the AM-GM optimum $\alpha^* = (1/k_w, \ldots, 1/k_w)$ (uniform over original midpoints) gives $f_w(\alpha^*) = 1/k_w$ and $f^+_w(\alpha^*) = k^+_w \cdot (1/k_w)^{2/k^+_w} \leq k^+_w \cdot (1/k^+_w)^{2/k^+_w} \cdot (k^+_w/k_w)^{2/k^+_w} / (k^+_w/k_w) = 1/k^+_w \leq 1/k_w = f_w(\alpha^*)$, confirming pointwise $f^+_w(\alpha^*) \leq f_w(\alpha^*)$.

More generally, one can verify that the product $k_w \cdot (1/k_w)^{2/k_w}$ is decreasing in $k_w$ for $k_w \geq 1$, which is equivalent to $k^{1-2/k}$ being increasing, a standard calculus fact. Therefore $f^+_w(\alpha) \leq f_w(\alpha)$ for all $\alpha$ and all $w$.

Since $f^+_w(\alpha) \leq f_w(\alpha)$ holds for every $w \in \mathrm{S}_2(z)$ and every $\alpha$ in the simplex, summing over $w$ gives $\sum_w f^+_w(\alpha) \leq \sum_w f_w(\alpha)$ for all $\alpha$. Taking the supremum over $\alpha$ preserves this inequality:
\begin{equation}
    \mathrm{H}^{\mathcal{G}^+}_{\mathrm{L}_0}(z, \mathrm{S}^{\mathcal{G}}_2(z)) 
    \;=\; \sup_\alpha \sum_w f^+_w(\alpha) 
    \;\leq\; \sup_\alpha \sum_w f_w(\alpha)
    \;=\; \mathrm{H}^{\mathcal{G}}_{\mathrm{L}_0}(z, \mathrm{S}^{\mathcal{G}}_2(z)).
\end{equation}
Applying the decreasing transformation $-2\log(\cdot)$ (Definition~\ref{weakentropic}) 
yields $\kappa_w^{\mathcal{G}^+}(z) \geq \kappa_w^{\mathcal{G}}(z)$. Strict improvement 
when $k^+_w > k_w$ for some $w$ follows from the strict pointwise inequality in that summand and the strict monotonicity of $-2\log(\cdot)$.
\end{proof}

%% file: Appendix/implementation_details.tex
\subsection{Algorithm Implementation of Weak Entropic Curvature}\label{subsec:app:algo}

We provide the exact computation of $\kappa_w(\mathcal{G})$ in Algorithm~\ref{alg:weak-curv-exact} below.

\begin{algorithm}
\caption{\textsc{Exact Computation of $\kappa_{w}(\mathcal{G})$}}\label{alg:weak-curv-exact}
\KwIn{Graph $\mathcal{G}=(\mathcal{V},\mathcal{E})$, generator $\Lop: \mathcal{V} \times \mathcal{V} \to \mathbb{R}$}
\KwOut{Curvatures: local $\boldsymbol{\kappa}$ and global $\kappa_{w}(\mathcal{G})$}
\ForEach(){$z \in \mathcal{V}$}{
    Let $\mathrm{S}_1(z) \leftarrow \{v \in \mathcal{V} : (z, v) \in \mathcal{E}\}$\;
    
    Let $\mathrm{S}_2(z) \leftarrow \{w \in \mathcal{V}: d(z,w)=2\}$\;
    
    \eIf{$\mathrm{S}_2(z) = \emptyset$}{
        $\kappa[z] \leftarrow \infty$ 
    }{
        1. Precompute for each $w\in \mathrm{S}_2(z)$ and $u\in ]z, w[$: $\ell(z,u,w) \leftarrow \frac{\Lop(z,u)\Lop(u,w)}{\Lop^{(2)}(z,w)}$\;
        
        2. Define objective function $\mathrm{J}(\alpha)$ for $\alpha \in \Delta_{|\mathrm{S}_1(z)|}$:\\
        $\mathrm{J}(\alpha) \leftarrow \sum\limits_{w \in \mathrm{S}_2(z)} \Lop^{(2)}(z,w) \prod\limits_{u \in ]z, w[} \Bigl(\frac{\alpha(u)}{\Lop(z,u)}\Bigr)^{2\ell(z,u,w)}$\;
        
        3. Solve constrained maximization:\\
        $\mathrm{H}^* \leftarrow \sup_{\alpha} \mathrm{J}(\alpha) \quad \text{s.t.} \sum_{u} \alpha(u)=1, \alpha(u) \ge 0$\;
        $\kappa[z] \leftarrow -2\log(\mathrm{H}^*)$\;
    }
}
$\kappa_{w}(\mathcal{G}) \leftarrow \inf_{z \in \mathcal{V}} \kappa[z]$\;

\Return $\kappa, \kappa_{w}(\mathcal{G})$\;
\end{algorithm}

\subsection{Fast Convergence of the Entropic Optimization}
\begin{figure}[h]
    \centering
    \includegraphics[width=0.9\linewidth]{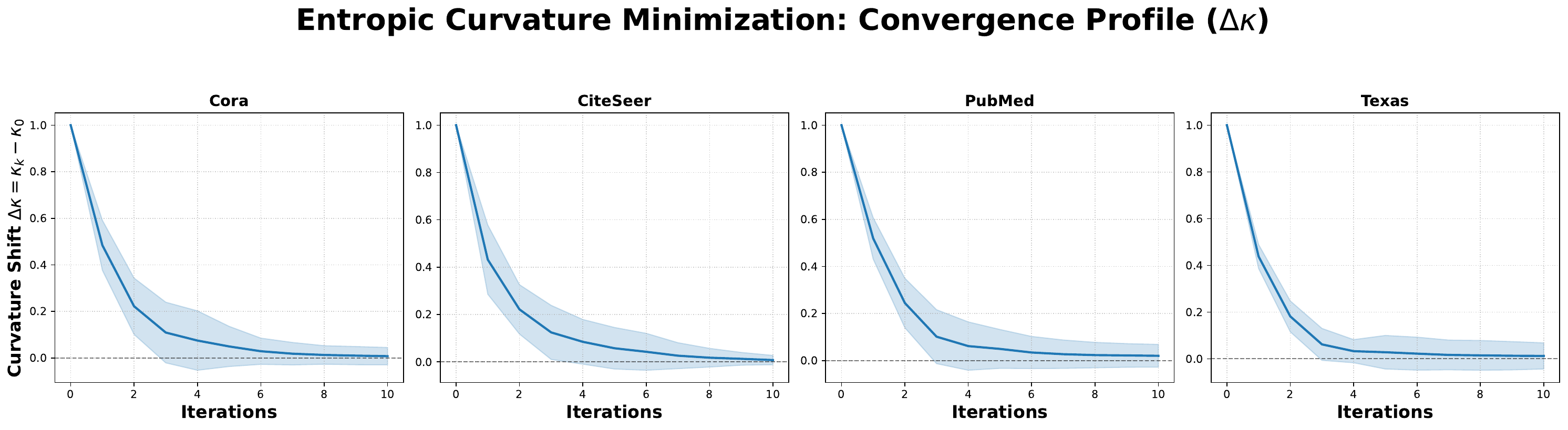}
    \caption{Empirical Convergence of Weak Entropic Curvature Optimization}
    \label{fig:evolution_optimi}
\end{figure}

To validate the computational tractability of the Weak Entropic Curvature, we empirically analyze the convergence behavior of the underlying optimization problem. As detailed in Algorithm \ref{alg:weak-curv-exact}, calculating the curvature at a specific node $z$ requires maximizing the local functional $\mathrm{H}_\Lop(z, \mathrm{S}_2(z))$ over the probability simplex. We employ the \texttt{SLSQP} method from the SciPy library \citep{virtanen2020scipy}, which is well-suited for handling the linear constraints of the simplex $\Delta_{|\mathrm{S}_1(z)|}$.

\textbf{Methodology.} Since the absolute value of the weak entropic curvature $\kappa_w(z)$ varies significantly across different local topologies (e.g., varying from positive in cliques to negative in trees), aggregating raw optimization traces would yield uninterpretable variance. Instead, we track the \textit{normalized convergence residual}. For a randomly sampled set of 50 nodes per dataset, we track the curvature estimate $\kappa^{(t)}$ at iteration $t$. We define the normalized residual $\Delta \tilde{\kappa}^{(t)}$ as,

\begin{equation}
\Delta \tilde{\kappa}^{(t)} = \frac{\kappa^{(t)} - \kappa^*}{\kappa^{(0)} - \kappa^*},
\end{equation}

where $\kappa^{(0)}$ is the initial curvature estimate (uniform distribution initialization) and $\kappa^*$ is the converged value (minimum reached within the tolerance threshold). This metric effectively maps the optimization landscape to the unit interval $[0, 1]$, where $1$ represents initialization and $0$ represents convergence.

\textbf{Results.} Figure \ref{fig:evolution_optimi} illustrates the mean convergence profile ($\pm$ standard deviation) across four distinct benchmarks: Cora, CiteSeer, PubMed, and Texas.

The empirical results reveal a consistent "fast-decay" profile across all datasets, independent of the graph's homophily or size:
\begin{itemize}
    \item \textbf{Rapid Descent:} The solver achieves the majority of the optimization gain within the first \textbf{3 to 5 iterations}. The steep initial drop indicates that the objective function $J(\alpha)$ is well-behaved and that the gradient signal effectively guides the transport plan $\alpha$ toward the optimal configuration.
    \item \textbf{Low Variance:} The narrow shaded regions (standard deviation) indicate that this convergence speed is robust. Regardless of whether a node resides in a dense community (Cora) or a sparse heterophilic bottleneck (Texas), the optimization complexity remains predictably low.
    \item \textbf{Stability:} Beyond 6 iterations, the updates become negligible ($< 10^{-4}$), confirming that setting a low maximum iteration cap (e.g., $k_{iter} \approx 10$) in the complexity term $O(n \cdot k_{iter} \cdot d_{max}^2)$ is sufficient for high-fidelity curvature estimation.
\end{itemize}

This empirical evidence supports our claim that the computational overhead of Entropic Curvature is dominated by neighborhood discovery rather than the optimization step itself, making it scalable to large graph datasets.

A critical observation from Figure \ref{fig:evolution_optimi} is that the convergence speed is effectively invariant to graph homophily. The solver converges just as rapidly for the disordered, heterophilic graphs as it does for structured, homophilic citation networks. This proves that the optimization landscape of the local entropic functional remains well-behaved even in the presence of noisy or fractured local topologies, distinguishing it from spectral approximations that often degrade on graphs with poor expansion properties.

%% file: Appendix/lcp_ent_exp.tex
In this section, we detail the implementation of the structural encodings used in our comparative analysis. All curvature computations were pre-processed and cached to ensure training efficiency.

\subsection{Entropic Structural Encoding (ENT)}
Our proposed method computes the Weak Entropic Curvature $\kappa_w(u)$ for every node $u$ using the exact optimization procedure described in Section \ref{sec:exper}. To capture the local geometric context, we construct a 3-dimensional signature vector $\mathbf{p}_u^{\text{ENT}}$ consisting of the node's curvature and the extrema of its neighbors:
\begin{equation}
    \mathbf{p}_u^{\text{ENT}} = \left[ \kappa_w(u), \quad \min_{v \in \mathcal{N}(u)} \kappa_w(v), \quad \max_{v \in \mathcal{N}(u)} \kappa_w(v) \right]^\top
\end{equation}
Our choice of a 3-dimensional signature $\mathbf{p}_u^{\text{ENT}}$, i.e., focusing on $\kappa_w(u)$ and the neighborhood extrema, is deliberate. By explicitly encoding the minimum ($\min \kappa_w$) and maximum ($\max \kappa_w$) curvature of the neighbors, the signature captures the geometric gradient, or anisotropy, of the local manifold. This differentiates nodes that sit on the boundary of a cluster (high max, low min) from those deep within a clique (high max, high min). Crucially, this 3D representation is more compact than the 5D statistical vector used in LCP \citep{fessereffective}, reducing the risk of overfitting on small, data-scarce heterophilic datasets like Texas and Cornell while retaining the essential topological signal.

\subsection{Baseline Implementations}
We compare our method against the following structural and positional encodings.

\textbf{1. Local Curvature Profiles (LCP).}
Following the standard implementation~\cite{fessereffective}, LCP augments node features with statistics derived from the discrete Ollivier-Ricci Curvature (ORC).
\begin{itemize}
    \item \textbf{Computation:} We compute the exact ORC for every edge $(u, v)$ in the graph using the \texttt{GraphRicciCurvature} library. The transport cost is defined based on the shortest path distance between neighbors.
    \item \textbf{Encoding:} For each node $u$, we aggregate the curvature values of all incident edges $E_u = \{ \kappa_{\text{ORC}}(u, v) \mid v \in \mathcal{N}(u) \}$. The resulting feature vector $\mathbf{p}_u^{\text{LCP}} \in \mathbb{R}^5$ consists of the following statistics:
    \begin{equation*}
    \mathbf{p}_u^{\text{LCP}} = \left[ \min(E_u), \max(E_u), \text{mean}(E_u), \text{std}(E_u), \text{median}(E_u) \right]^\top
    \end{equation*}
    Nodes with no neighbors are assigned a zero vector.
\end{itemize}

\textbf{2. Laplacian Positional Encodings (LAPE).}
We compute the graph Laplacian $\mathbf{L} = \mathbf{I} - \mathbf{D}^{-1/2}\mathbf{A}\mathbf{D}^{-1/2}$. The positional encoding is formed by the $k$ smallest non-trivial eigenvectors of $\mathbf{L}$. These eigenvectors capture the low-frequency modes of the graph, preserving global distance information.

\textbf{3. Random Walk Positional Encodings (RWPE).}
RWPE captures the structural role of a node by examining the probability of a random walker returning to itself. The encoding $\mathbf{p}_u^{\text{RWPE}} \in \mathbb{R}^k$ is defined by the diagonal entries of the random walk diffusion matrix at various steps:
\begin{equation*}
\mathbf{p}_u^{\text{RWPE}} = \left[ \mathbf{P}^1_{uu}, \mathbf{P}^2_{uu}, \dots, \mathbf{P}^k_{uu} \right]^\top
\end{equation*}
where $\mathbf{P} = \mathbf{D}^{-1}\mathbf{A}$ is the transition matrix. This encoding effectively characterizes the local expansion properties around each node.

\textbf{4. Ego-Network Encodings (EGO).}
This baseline relies on explicit structural feature extraction. For each node $u$, we extract its $k$-hop ego network and compute classic graph theoretic invariants, such as node degree, number of edges, and clustering coefficient, to serve as the structural signature.

%% file: Appendix/alternative_generators.tex
In this section, we detail the construction of reversible generators for general and concentrated measures, extending the standard uniform case presented in Section \ref{sec:math_fram}.

\begin{example}[General Positive Measure]\label{ex:general}
    For an arbitrary strictly positive measure $\m (\cdot) > 0$, we can satisfy reversibility by defining the off-diagonal terms as follows,
    \begin{equation}
        \Lop (u,v) = \mathbf{1}_{\{u,v\} \in \mathcal{E}} \left( \frac{\m (v)}{\m (u)} \right)^{1/2}.
    \end{equation}
\end{example}

Finally, we consider the case where we wish to concentrate the measure on a specific subset $\mathcal{C} \subset \mathcal{V}$. Defining $m$ strictly as the indicator function $\m_{\mathcal{C}} = \mathbf{1}_{\mathcal{C}}$ is theoretically prohibitive, as $m$ must be strictly positive to define a valid inner product space. Furthermore, if $\mathcal{C}$ is not connected (i.e., non-convex), a valid irreducible generator cannot exist on the support. We define a smoothed approximation $\m_\epsilon \simeq \m_{\mathcal{C}}$, which yields a Metropolis-Hastings style generator.

\begin{example}[Concentrated Measure Approximation]\label{ex:approx}
    Let the target measure be approximated by $\m_\epsilon(u) = \max(\mathbf{1}_{\mathcal{C}}(u), \epsilon)$ where $0 < \epsilon \ll 1$. The generator $\Lop$ satisfying reversibility is given by,
    \begin{equation}
        \Lop (u,v) = \mathbf{1}_{\{u,v\} \in \mathcal{E}} \cdot \min\left(1, \frac{\m_\epsilon(v)}{\m_\epsilon(u)}\right).
    \end{equation}
\end{example}

%% file: Appendix/generator_vizualization.tex
\begin{figure}[h]
    \centering
    \includegraphics[width=0.8\linewidth]{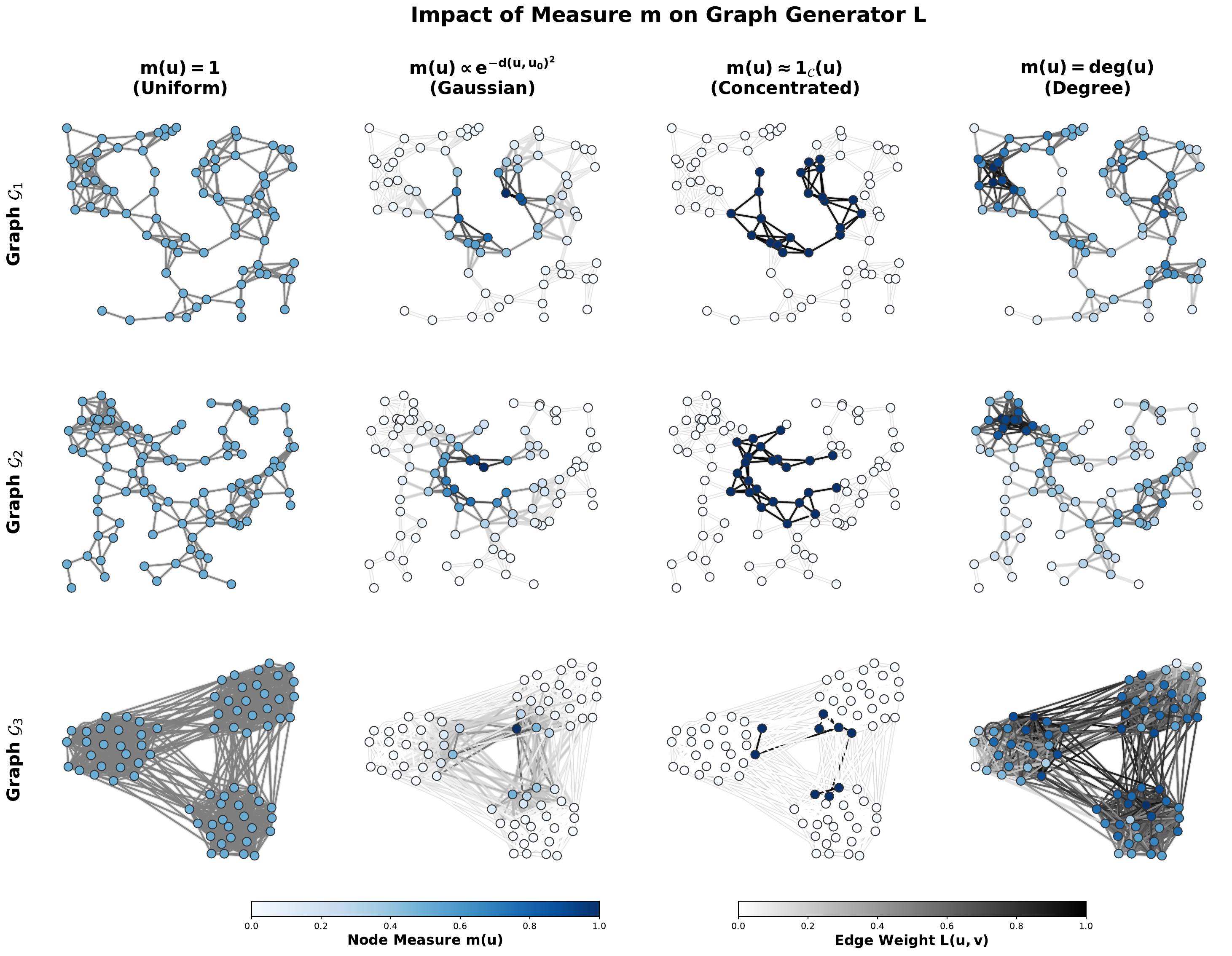}
    \caption{\textbf{Geometric reconfiguration via mass distribution.} Changing the reference measure $\mathfrak{m}$ fundamentally alters the graph's diffusion dynamics. }
    \label{fig:genertors_example}
\end{figure}

The choice of the reference measure $\mathfrak{m}$ and the resulting generator $\Lop$ is a defining feature of the entropic curvature framework, as it dictates the \emph{geometry of the mass} across the graph. Unlike standard curvature metrics that are fixed by the adjacency matrix, our framework allows the geometry to be modulated by a functional prior.

\subsection{Thermodynamic Interpretation of Generators}
The detailed balance condition $\mathfrak{m}(u)\Lop(u,v) = \mathfrak{m}(v)\Lop(v,u)$ ensures that the process is reversible with respect to $\mathfrak{m}$. From a thermodynamic perspective, $\mathfrak{m}$ represents the steady-state distribution or the capacity of nodes. 
\begin{itemize}
    \item \textbf{Uniform Measure ($m=1$):} Corresponds to a \emph{flat} thermodynamic prior where every node has equal capacity. The resulting generator $\Lop_0$ is the standard combinatorial Laplacian, where information flow is driven purely by the topological connectivity.
    \item \textbf{Degree-Based Measure ($m=deg$):} Aligns the geometry with the graph's natural connectivity. This measure typically smooths out the influence of high-degree hubs, as the cost of moving mass is normalized by the node's local popularity.
    \item \textbf{Concentrated/Gaussian Measure:} When $\mathfrak{m}$ is concentrated on a subset $C$ (as seen in Figure \ref{fig:genertors_example}), the generator $\Lop$ behaves like a \textit{potential well}. Edges leading toward the concentration center are strengthened (higher jump rates), while edges leading away are \emph{penalized}. This effectively creates a geometric \emph{drift} in the Wasserstein geodesics.
\end{itemize}

\subsection{Impact of Measure $\mathbf{m}$ on Generator $\mathbf{L}$}
Figure \ref{fig:genertors_example} visualizes how varying the reference measure $\mathfrak{m}$ fundamentally reconfigures the edge weights $\Lop(u,v)$ across three different graph topologies ($\mathcal{G}_1, \mathcal{G}_2, \mathcal{G}_3$).

\begin{itemize}
    \item \textbf{Column 1 (Uniform):} The edge weights are binary ($0$ or $1$), representing the raw topology where entropic curvature identifies structural bottlenecks like bridges in the Barbell graph.
    \item \textbf{Column 2 (Gaussian):} By centering a Gaussian measure $\mathfrak{m}(u) \propto e^{-d(u, u_0)^2}$ at a specific node $u_0$, the graph is transformed into a \textit{localized metric space}. The curvature is now highest near the center $u_0$ and decays as one moves toward the periphery, effectively \emph{cooling} the message-passing in the center to preserve local feature distinctness.
    \item \textbf{Column 3 (Concentrated):} Using the indicator approximation $\mathfrak{m}_\epsilon \simeq 1_C$, we observe a sharp transition in the generator weights. Edges within the community $C$ maintain high weights, while boundary edges are attenuated. This illustrates how entropic curvature can be used to \textit{isolate} communities by creating artificial geometric boundaries that prevent feature oversmoothing across distinct clusters.
    \item \textbf{Column 4 (Degree):} Weighting by degree regularizes the influence of hubs. In dense graphs like $\mathcal{G}_3$, this choice spreads the entropic pressure more evenly, preventing the exponential collapse of variance typically seen in high-degree expanders.
\end{itemize}

%% file: Appendix/complexity_analysis.tex
In this section, we analyze the computational complexity of the Entropic Curvature Algorithm, c.f., Algorithm \ref{alg:weak-curv-exact}.  Let $\mathcal{G} = (\mathcal{V}, \mathcal{E})$ be a graph with $n = |\mathcal{V}|$ nodes and maximum degree $d_{\text{max}}$. The computation of the Local Weak Entropic Curvature is a node-wise operation. We distinguish between the general setting involving arbitrary generators and the simplified case of standard unweighted graphs, i.e., $\Lop=\Lop_0$.

\subsection{The General Case of Arbitrary Generator $\mathbf{L}$ and Mass $\mathbf{m}$}

For an arbitrary generator $L$ and measure $m$, computing the curvature at a node $z$ requires solving a constrained maximization problem over the probability simplex. This process involves three steps:

\begin{itemize}
    \item \textbf{Step 1 - Neighborhood Discovery:} We need to identify the 1-hop neighborhood $\mathrm{S}_1(z)$ and the 2-hop neighborhood $\mathrm{S}_2(z)$.  In the worst case, $|\mathrm{S}_1(z)| \leq d_{\text{max}}$ and $|\mathrm{S}_2(z)|\leq d_{\text{max}}^2$, yielding a complexity cost of $O(d_{\text{max}}^2)$ for the discovery. For the optimization step, we must identify the set of midpoints $]z, w[ = \mathrm{S}_1(z) \cap \mathrm{S}_1(w)$ for every $w \in \mathrm{S}_2(z)$. There are up to $d_{\text{max}}^2$ nodes in $\mathrm{S}_2(z)$, and finding the intersection of two sorted neighbor lists for a single pair takes $O(d_{\text{max}})$. Therefore, the overall complexity of computing all midpoint sets is $O(d_{\text{max}}^3)$.
    \item \textbf{Step 2 - Calculation of the coefficients:}   The algorithm computes the transport coefficients $l(z, u, w)$ for all the paths $z \to u \to w$, where $u \in \mathrm{S}_1(z)$ and $w \in \mathrm{S}_2(z)$. This step requires iterating through potential paths in the local neighborhood, scaling as $O(d_{\text{max}}^3)$ in dense subgraphs.
    
    \item \textbf{Step 3 - Optimization:} The core computation is determining $\sup_{\alpha} J(\alpha)$ using Sequential Least Squares Programming (SLSQP) \cite{nocedal2006numerical}. The complexity of this iterative procedure is given by $O(k_{\text{iter}} \cdot C_{\text{step}})$ where $k_{\text{iter}}$ is the number of iterations required for convergence, and $C_{\text{step}}$ is the cost of computing the objective function and its gradient over the simplex. The cost per iteration, $C_{\text{step}}$, is directly determined by the output of Steps 1 and 2. Evaluating the objective function $J(\alpha)$ requires summing over the 2-hop neighborhood $\mathrm{S}_2(z)$, i.e, identified in Step 1, and accessing the precomputed sparse coefficients $l(z, u, w)$, i.e., calculated in Step 2.
\end{itemize}

Consequently, by aggregating the costs of neighborhood discovery, coefficient precomputation, and the iterative optimization, the overall complexity to compute the local entropic curvature for a single node $z$ is $O(d_{\text{max}}^3 + k_{\text{iter}} \cdot d_{\text{max}}^2)$. To obtain the global weak entropic curvature $\kappa_w (\mathcal{G})$, the algorithm must execute this procedure for every node $z \in \mathcal{V}$, continuously tracking the minimum curvature value observed. Therefore, the total time complexity for the entire graph is $O(n \cdot (d_{\text{max}}^3 + k_{\text{iter}} \cdot d_{\text{max}}^2))$.

\subsection{The Simplified Case of $\mathbf{L = \Lop_0}$ and $\mathbf{m = m_0}$}

In the standard unweighted setting where the generator is $\Lop_0$, defined as $\Lop_0(z, z') = 1_{z \sim z'}$, and the measure $m_0$ is uniform, the computational complexity is significantly reduced.

Crucially, Step 2 (Calculation of coefficients) is significantly simplified since the transition weights are binary. The transport coefficients $\ell(z, u, w)$ reduce to a uniform fraction over shared midpoints:
$$\ell(z, u, w) = \frac{\Lop_0(z, u)\Lop_0(u, w)}{\Lop^{(2)}_0(z, w)} = \frac{1}{|\mathrm{M}_\mathcal{G}(z, w)|},$$
where $\Lop^{(2)}_0(z,w) = |\mathrm{M}_\mathcal{G}(z,w)|$ denotes the 2-step transition weight, i.e., the number of shared midpoints between $z$ and $w$. Note that $\ell(z,u,w) = 1$ only when the two-hop path $z \to u \to w$ is unique. In the general unweighted case, the coefficients are $1/|\mathrm{M}_\mathcal{G}(z,w)|$, which are precomputed in $O(d_\mathrm{max}^2)$ by counting shared neighbors. Since these coefficients are rational constants requiring no floating-point solver, Step 2 reduces to an $O(d_\mathrm{max}^2)$ table construction rather than the full $O(d_\mathrm{max}^3)$ cost of the general case.

This simplification substantially reduces the precomputation cost. Consequently, the algorithm proceeds directly from Neighborhood Discovery (Step 1) to Optimization (Step 3). While the optimization must still be performed iteratively (e.g., using SLSQP), the objective function evaluation is streamlined as it no longer requires retrieving stored coefficient tensors. The global complexity to compute the local weak entropic curvature in the simplified case is $O(n \cdot (d_{\text{max}}^2 + k_{\text{iter}} \cdot d_{\text{max}}^2)) = O(n \cdot k_{\text{iter}} \cdot d_{\text{max}}^2)$, since the coefficient precomputation is $O(d_\mathrm{max}^2)$ per node and is dominated by the optimization step.

\subsection{Empirical Complexity}
To complement our theoretical analysis, we empirically evaluate the runtime performance of the Weak Entropic Curvature algorithm, c.f., Algorithm \ref{alg:weak-curv-exact}. We utilize Stochastic Block Models (SBM) to generate synthetic graphs of varying sizes \cite{holland1983stochastic}, allowing us to systematically control the number of nodes $|\mathcal{V}|$. This setup enables us to benchmark the scalability of our approach across both the general case, i.e., arbitrary $\Lop, m$, and the simplified case, i.e., $\Lop=\Lop_0, m=m_0$. Given intra-community edge probability $p$ and inter-community probability $q$, the expected number of edges is given by $E[|\mathcal{E}|] = \frac{n}{2} [ (\frac{n}{k} - 1)p + (n - \frac{n}{k})q ]$, while the expected degree is $E[d] = \frac{n}{k}p + (n - \frac{n}{k})q$, where $n=\mathcal{V}|$. By fixing the ratio of $p$ and $q$, we benchmark the scalability of our approach across both the general case (arbitrary $L, m$) and the simplified case ($\Lop=\Lop_0, m=m_0$).

We measure the average time required to compute the global curvature $\kappa_w(\mathcal{G})$ as the graph size increases. The results, summarized in Table \ref{tab:runtime}, highlight the practical computational efficiency gained in the simplified regime due to the elimination of the coefficient precomputation step.

\begin{table}[h]
    \centering
    \caption{Empirical runtime comparison (in seconds) for computing global Weak Entropic Curvature on SBM ($k=2$). We vary $|\mathcal{V}|$ and set $p, q$ to target specific maximum degrees $d_{\text{max}}$.}
    \label{tab:runtime}
    \begin{tabular}{@{}ccccccc@{}}
        \toprule
        \multicolumn{3}{c}{\textbf{Graph Statistics}} & \multicolumn{2}{c}{\textbf{SBM Parameters}} & \multicolumn{2}{c}{\textbf{Execution Time (s)}} \\ 
        \cmidrule(lr){1-3} \cmidrule(lr){4-5} \cmidrule(lr){6-7}
        $|\mathcal{V}|$ & $E[|\mathcal{E}|]$ & $d_{\text{max}}$ & $p$ & $q$ & \textbf{Simplified ($\mathbf{\Lop_0,m_0}$)}   & \textbf{General ($\mathbf{L, m}$)} \\ \midrule
1,000  & 5,000   & $\approx 10$  & 0.015 & 0.005 & 35.770  &  406.997 \\
         1,000  &  15,000  & $\approx  30$  & 0.040 & 0.020 & 1471.913  &  4903.886 \\ 
         2,000  & 20,000  & $\approx  20$  & 0.015 & 0.005 & 485.642 & 5400.812  \\
         2,000  & 50,000  & $\approx  50$  & 0.040 & 0.010 &  13142.537 &  38427.988 \\ 
         5,000  & 75,000  & $\approx  30$  & 0.010 & 0.002 & 4080.36 & 39804.04  \\
         10,000 & 250,000 & $\approx  50$  & 0.008 & 0.002 & 36029.70 & 335796.21 \\  \bottomrule
    \end{tabular}
\end{table}


\subsection{Comparison with other Curvatures}

We compare the complexity of Entropic Curvature against established discrete curvature notions in Table \ref{tab:complexity}.
\begin{itemize}
    \item \textbf{Forman Curvature} is computed by summing the weighted contributions of the edges immediately incident to a target link's endpoints, resulting in a linear complexity of $O(d_{\text{max}})$ per edge.
    \item \textbf{Ollivier-Ricci Curvature} requires solving the Earth Mover's Distance (Optimal Transport) problem between neighborhood measures. 
    \item \textbf{Balanced Forman Curvature} is computed by augmenting the standard Forman curvature formula with additional terms that account for triangles (3-cycles) and 4-cycles (squares) rooted at the edge, effectively capturing the local flatness of the graph geometry.
\end{itemize}
Our Entropic formulation shares the same complexity class as Ollivier-Ricci. We note that for the general weighted case the objective $J(\alpha)$ is a product-of-ratios functional that is log-concave in $\alpha$ on the simplex, admitting efficient gradient-based optimization via SLSQP. For the simplified unweighted case the objective reduces (via the Motzkin-Straus connection) to a quadratic form $\alpha^\top A_{z} \alpha$ whose matrix $A_z$ may be indefinite; SLSQP is therefore applied as a local solver and finds global optima in practice due to the simplex constraint structure, but theoretical global optimality in this case relies on the Motzkin-Straus theorem rather than concavity alone. In all experiments, convergence to the global optimum was confirmed by comparison with brute-force enumeration on small instances.

\begin{table}[h]
\caption{Complexity comparison of discrete graph curvatures. $n=|\mathcal{V} |$: number of nodes, $m =|\mathcal{E} |$: number of edges, $\Delta$: max degree. The per-node complexity for Weak Entropic (Ours) refers to the simplified unweighted case ($\Lop=\Lop_0$); the general weighted case has per-node complexity $O(d_{\text{max}}^3 + k_{\text{iter}} \cdot d_{\text{max}}^2)$.}
\label{tab:complexity}
\begin{center}
\begin{small}
\begin{tabular}{lcccc}
\toprule
Curvature Metric & Definition Basis & Per-Node Complexity & Per-Edge Complexity & Global Complexity \\
\midrule
Forman \cite{formann} & Combinatorial & $-$ & $O(d_{\text{max}})$ & $O(m \cdot d_{\text{max}})$  \\
Ollivier-Ricci \cite{ollivier2009ricci} & Optimal Transport & $-$ & $O( d_{\text{max}}^3)$ &$\mathcal{O}(m \cdot d_{\text{max}}^3)$ \\
Balanced Forman \cite{topping2022understanding} & Combinatorial & $-$ & $O( d_{\text{max}}^2)$ &$O(m \cdot d_{\text{max}}^2)$ \\
Weak Entropic (Ours, simplified) & Optimization & $O( k_{\text{iter}} \cdot d_{\text{max}}^2)$ & $-$ & $O( n \cdot k_{\text{iter}} \cdot d_{\text{max}}^2)$ \\
\bottomrule
\end{tabular}
\end{small}
\end{center}
\end{table}

Moreover, unlike global spectral methods, e.g., Eigen-decomposition, which are difficult to parallelize effectively, our Entropic Curvature computation is node-wise independent, making it \emph{embarrassingly parallelizable}. On large-scale graphs, this allows the runtime to scale linearly with the number of available CPU cores. Furthermore, unlike the Ollivier-Ricci curvature, which requires solving a Linear Programming, e.g.,Earth Mover's Distance (EMD), problem at every edge, our formulation relies on a smooth concave maximization. This structure allows the use of gradient-based solvers, like SLSQP, which are often numerically more stable and faster in wall-clock time than the simplex or interior-point methods required for EMD.

%% file: Appendix/additional_experiments.tex
In this appendix, we report additional experiments referenced in Section~\ref{sec:exper}.

\subsection{Graph Classification}\label{app:graph_classification}

While our theory is stated for the node-classification setting, both E-Gate and ENT extend to graph classification by precomputing $\kappa_w$ on each graph in the training set. We evaluate on TUDataset~\citep{morris2020tudataset} (MUTAG, PROTEINS, IMDB-B, NCI1; 10-fold CV with the standard split). The protocol follows the recommendations of~\citep{morris2020tudataset}: 4-layer backbone, hidden dimension 64, mean-readout, 10 seeds. We compare against GCN, GIN, SAGPool, LCP~\citep{fessereffective}, and SDRF~\citep{topping2022understanding}.

\input{Tables/graph_classification}




\subsection{Ablations}\label{app:ablations}

Table~\ref{tab:ablation_egate_full} ablates: (i) the functional form of $\mathcal{R}(u)$ (exponential, sigmoid, linear), (ii) whether $\tau$ is learned or fixed, (iii) replacing $\kappa_w$ with degree, with Ollivier--Ricci edge-min, or with random Gaussian noise (parameter-count control), and (iv) the dimension of the ENT signature (1D, 3D, 5D). Ablations (i)--(iii) directly test whether the global, transport-based nature of $\kappa_w$ is what is doing the work; (iv) tests whether the 3-D summary used in the main paper is sufficient.

\input{Tables/ablation}



%% file: Tables/graph_classification.tex
\begin{table*}[t]
\centering
\caption{Graph classification on TUDataset (10-fold CV, accuracy \%). E-Gate is applied per-layer; the curvature signature is precomputed once on the union of training graphs. \textbf{Bold}: best per column.}
\label{tab:graph_classification}
\small
\resizebox{0.85\textwidth}{!}{%
\begin{tabular}{lcccc}
\toprule
Method & MUTAG & PROTEINS & IMDB-B & NCI1 \\
\midrule
GCN          & 85.60 $\pm$ 1.20 & 71.97 $\pm$ 4.42 & 72.30 $\pm$ 4.15 & 74.20 $\pm$ 0.50 \\
GIN          & 88.41 $\pm$ 0.89 & 71.72 $\pm$ 0.59 & 74.20 $\pm$ 3.70 & 76.55 $\pm$ 0.44 \\
SAGPool      & 75.71 $\pm$ 1.76 & 70.80 $\pm$ 1.23 & 70.01 $\pm$ 4.52 & 74.10 $\pm$ 1.10 \\
LCP          & \textbf{90.17 $\pm$ 1.01} & 72.80 $\pm$ 1.15 & 74.80 $\pm$ 3.10 & 76.20 $\pm$ 0.60 \\
SDRF         & 82.35 $\pm$ 1.77 & 71.50 $\pm$ 1.05 & 73.10 $\pm$ 3.90 & 75.80 $\pm$ 0.75 \\
\midrule
GCN + E-Gate (ours)         & 86.50 $\pm$ 1.10 & 72.50 $\pm$ 0.95 & 74.15 $\pm$ 3.50 & 75.10 $\pm$ 0.65 \\
GIN + E-Gate (ours)         & 88.20 $\pm$ 0.76 & 70.58 $\pm$ 1.33 & 75.30 $\pm$ 3.20 & 76.45 $\pm$ 0.30 \\
GIN + MCR (ours)            & 73.87 $\pm$ 1.57 & 72.19 $\pm$ 0.80 & 74.50 $\pm$ 3.60 & 74.51 $\pm$ 0.38 \\
GIN + MCR + E-Gate (ours)   & 75.22 $\pm$ 1.46 & \textbf{73.18 $\pm$ 0.60} & \textbf{75.90 $\pm$ 2.80} & \textbf{77.10 $\pm$ 0.42} \\
ENT-encoding (ours)         & 87.10 $\pm$ 1.40 & 71.80 $\pm$ 1.10 & 73.80 $\pm$ 4.10 & 75.50 $\pm$ 0.55 \\
\bottomrule
\end{tabular}%
}
\end{table*}

%% file: Tables/ablation.tex
\begin{table*}[t]
\centering
\caption{Ablations on the E-Gate mechanism and ENT signature using a GCN backbone across four datasets (10 seeds). The $\Delta$ columns indicate the change in accuracy percentage points relative to the sigmoid gate mechanism. Default settings for the signature are highlighted in gray. }
\label{tab:ablation_egate_full}
\scriptsize
\setlength{\tabcolsep}{3pt}
\resizebox{\textwidth}{!}{%
\begin{tabular}{lcccccccc}
\toprule
& \multicolumn{2}{c}{\textbf{Wisconsin}} & \multicolumn{2}{c}{\textbf{Texas}} & \multicolumn{2}{c}{\textbf{Cora}} & \multicolumn{2}{c}{\textbf{CiteSeer}} \\
\cmidrule(lr){2-3} \cmidrule(lr){4-5} \cmidrule(lr){6-7} \cmidrule(lr){8-9}
\textbf{Variant} & Acc. (\%) & $\Delta$ & Acc. (\%) & $\Delta$ & Acc. (\%) & $\Delta$ & Acc. (\%) & $\Delta$ \\
\midrule
\rowcolor{gray!15}
Default: $\mathcal{R}(u)=\sigma(-\tau\kappa_w)$ (sigmoid) & 51.57 $\pm$ 7.55 & --- & \textbf{52.97 $\pm$ 9.76} & --- & \textbf{88.21 $\pm$ 1.38} & --- & \textbf{76.67 $\pm$ 1.38} & --- \\

$\exp(-\tau\kappa_w)$ & 48.63 $\pm$ 6.37 & $-2.94$ & 50.27 $\pm$ 8.48 & $-2.16$ & 87.97 $\pm$ 1.56 & $-0.24$ & 76.64 $\pm$ 1.58 & $+0.57$ \\

$\mathcal{R}(u)=1-\tau\kappa_w$ (linear) & 49.41 $\pm$ 7.17 & $-2.16$ & 52.43 $\pm$ 8.39 & $-0.54$ & 87.99 $\pm$ 1.12 & $-0.22$ & 76.07 $\pm$ 1.65 & $-0.60$ \\
$\tau$ fixed at $1.0$ & 48.63 $\pm$ 6.37 & $-2.94$ & 45.68 $\pm$ 8.24 & $-6.75$ & 75.31 $\pm$ 2.08 & $-12.90$ & 65.87 $\pm$ 1.43 & $-10.20$ \\
$\tau$ fixed at $0.5$ & 51.76 $\pm$ 3.94 & $+0.19$ & 48.65 $\pm$ 7.55 & $-3.78$ & 84.19 $\pm$ 1.36 & $-4.02$ & 72.55 $\pm$ 1.44 & $-3.52$ \\
\midrule
ENT signature: 1D ($\kappa_w$ only) & 50.78 $\pm$ 7.36 & $-0.79$ & 52.43 $\pm$ 10.62 & $-0.54$ & 87.90 $\pm$ 1.21 & $-0.31$ & 76.37 $\pm$ 1.47 & $+0.30$ \\
ENT signature: 3D (default, $\kappa_w(u)$ + $\min/\max$ over $\mathcal{N}(u)$) & 48.43 $\pm$ 7.13 & $-3.14$ & 52.43 $\pm$ 8.65 & $-0.54$ & 87.03 $\pm$ 1.03 & $-1.18$ & 76.29 $\pm$ 1.39 & $+0.22$ \\
ENT signature: 5D (3D + mean + std over $\mathcal{N}(u)$)  & \textbf{52.75 $\pm$ 6.82} & $+1.18$ & 52.16 $\pm$ 7.26 & $-0.81$ & 87.42 $\pm$ 1.19 & $-0.79$ & 76.20 $\pm$ 1.34 & $+0.13$ \\
\bottomrule
\end{tabular}}
\end{table*}

%% file: Appendix/oversmoothing.tex
Oversmoothing refers to the progressive collapse of node representations in deep GNNs, where repeated aggregation drives features toward indistinguishable vectors~\citep{li2018deeper,cai2020note}. Prior work has related this  to local graph curvature notions such as Ollivier--Ricci and Forman curvature~\citep{nguyen2023revisiting,topping2022understanding,fesser2024mitigating}. We show that the same geometric intuition extends to entropic curvature, but now through global functional inequalities rather than local edge comparisons. The key object is the one-sided Dirichlet energy $\mathcal{E}(f)
:=
\int_{\mathcal{V}}
\sup_{v\sim u}
[f(u)-f(v)]_+^2
\,d\mu(u),$ which measures the amount of local variation remaining in a graph signal.  Theorem~\ref{thm:poincare-gnn} gives a global certificate for feature collapse: when local aggregation reduces the Dirichlet energy, the total feature variance is forced to decrease, with a rate controlled by the curvature lower bound. Larger positive $\kappa_w$ tightens this control, meaning that high-curvature regions are more susceptible to rapid homogenization under repeated smoothing. Unlike spectral-gap results derived from Ollivier--Ricci curvature~\citep{nguyen2023revisiting}, this bound applies directly to arbitrary graph signals and therefore to intermediate GNN representations.

\begin{theorem}[Poincaré-type Inequality under Positive Entropic Curvature; Proof in Appendix~\ref{app:proof:thm:poincare-gnn}]
\label{thm:poincare-gnn}
Let $(\mathcal{V},d,\mathfrak{m},\mathtt{L})$ be a finite connected graph space with $\kappa_w>0$, and let $\mu:=\mathfrak{m}/\mathfrak{m}(\mathcal{V})$. For any graph signal $f:\mathcal{V}\to\mathbb{R}$, we have $\mathrm{Var}_\mu(f)
  \le
  \frac{1}{\kappa_w}
  \int_{\mathcal{V}}
  \sup_{v\sim u}
  [f(u)-f(v)]_+^2
  \,\mathrm{d}\mu(u).$
  
\end{theorem}

The complementary question is what happens in negatively curved regions, which are typical in sparse graphs with strong expansion by Theorem~\ref{thm:expansion}. In this regime, the geometry has the opposite effect: expansion creates many directions for mass to spread, yielding a lower bound on feature variance whenever local variation remains.

\begin{theorem}[Inequality under Negative Entropic Curvature; Proof in Appendix~\ref{app:proof:thm:poincare_negative}]
\label{thm:poincare_negative}
Let $(\mathcal{V},d,\mathfrak{m},\mathtt{L})$ be a finite connected graph space with weak entropic curvature bounded below by $\kappa_w\ge -K$ for some $K>0$, and let $\mu:=\mathfrak{m}/\mathfrak{m}(\mathcal{V})$. Let $D:=\max_{u,v\in\mathcal{V}} d(u,v)$ denote the graph diameter. Then for any graph signal $f:\mathcal{V}\to\mathbb{R}$, we have 
   $ \mathrm{Var}_{\mu}(f)
    \ge
    \frac{e^{-KD^2/2}}{K}
    \int_{\mathcal{V}}
    \sup_{v\sim u}
    [f(v)-f(u)]_+^2
    \,d\mu(u).$
\end{theorem}

%% file: Appendix/limitations.tex
In this appendix we expand on the limitations and open directions of our framework, covering both the theoretical assumptions and the experimental scope.

\paragraph{(L1) Reference-measure dependence.} The weak entropic curvature $\kappa_w$ depends on the choice of reference measure $\mathfrak{m}$ and the corresponding generator $\mathrm{L}$ (Section~\ref{sec:math_fram}). Throughout, we work with the canonical uniform pair $(\mathfrak{m}_0,\mathrm{L}_0)$, which makes $\kappa_w$ a pure function of the graph topology. Appendices~\ref{alternative_generator} and~\ref{app:generator_viz} give the construction of generators under more general positive measures and concentrated approximations. A natural open direction is to learn the measure $\mathfrak{m}$ that maximises $\kappa_w$ jointly with the GNN parameters, which would clarify how global mass distribution interacts with curvature and could guide curvature-aware learning algorithms.

\paragraph{(L2) Choice of $\mathrm{W}_1$ vs.\ $\mathrm{W}_2$.} Following \citep{rapaport2024samson}, we use $\mathrm{W}_1$-Wasserstein geodesics rather than $\mathrm{W}_2$, since on finite graphs $\mathrm{W}_1$-geodesics admit a tractable LP characterization while $\mathrm{W}_2$-geodesics generally lack a combinatorial interpretation. On smooth manifolds the two are equivalent for Ricci-bound characterisation~\citep{Villani}; in the discrete regime they may differ by constants, and a systematic comparison is interesting future work.

\paragraph{Open directions.} Beyond addressing the above limitations, we view three directions as particularly promising: (i) curvature-aware curriculum learning on evolving graphs (where $\kappa_w$ tracks the geometry of incoming subgraphs); (ii) using $\kappa_w$ as a principled regulariser for large-scale graph foundation models; and (iii) integrating the entropic-curvature framework with Schr\"odinger-bridge formulations of GNN dynamics, which would generalise both E-Gate and MCR into a single optimal-transport-based message-passing primitive.